\documentclass[twoside, 10pt]{article} 

\usepackage[T1]{fontenc}
\usepackage[utf8]{inputenc}
\usepackage[french,english]{babel}

\usepackage[a4paper, margin=2.5cm]{geometry} 
\usepackage{parskip}
\usepackage{float}
\usepackage{tcolorbox}
\usepackage{appendix} 
\usepackage{titlesec}
\titlespacing*{\paragraph}{0pt}{2ex plus 0.2ex minus 0.1ex}{1em}

\usepackage{amsmath}
\usepackage{amssymb}
\usepackage{amsthm}
\usepackage{mathtools}
\usepackage{stmaryrd}

\usepackage{graphicx}
\usepackage{caption}
\usepackage{subcaption}
\usepackage{booktabs}
\usepackage{array}

\usepackage{lmodern}
\usepackage{titling}

\usepackage{listings}
\usepackage{xcolor}

\usepackage[style=authoryear, backend=biber, natbib=true, maxcitenames=2, maxbibnames=99, uniquelist=false]{biblatex}
\usepackage[colorinlistoftodos]{todonotes}

\definecolor{todored}{RGB}{189, 30, 30}

\usepackage{hyperref}
\hypersetup{
    colorlinks=true,
    linkcolor=blue,
    filecolor=magenta,      
    urlcolor=cyan,
    citecolor=blue,
}
\usepackage[capitalise]{cleveref}
\usepackage{aliascnt}

\newtheorem{theorem}{Theorem}[section]

\newaliascnt{lemma}{theorem}
\newtheorem{lemma}[lemma]{Lemma}
\aliascntresetthe{lemma}
\newaliascnt{proposition}{theorem}

\aliascntresetthe{proposition}
\newaliascnt{corollary}{theorem}
\newtheorem{corollary}[corollary]{Corollary}
\aliascntresetthe{corollary}
\newaliascnt{assumption}{theorem}
\newtheorem{assumption}[assumption]{Assumption}
\aliascntresetthe{assumption}
\newaliascnt{definition}{theorem}

\aliascntresetthe{definition}

\newcommand{\R}{\mathbb{R}}

\newcommand{\var}{\mathrm{Var}}
\newcommand{\diag}{\mathrm{Diag}}
\newcommand{\E}{\mathbb{E}}
\newcommand{\tr}{\mathrm{Tr}}

\newcommand{\vect}{\theta}
\newcommand{\featvec}{x}
\newcommand{\labelscl}{y}
\newcommand{\bs}{b}

\renewcommand{\a}{A^{(1)}}
\renewcommand{\b}{A^{(2)}}

\renewcommand{\c}{A^{(3)}}
\newcommand{\w}{w}
\newcommand{\s}{S}

\newcommand{\q}{A^{(J)}}
\newcommand{\m}{\rho}
\newcommand{\frm}{\rho^{(0)}}

\newcommand{\gammax}{\gamma_{\max}}
\DeclareMathOperator*{\argmin}{argmin}

\newcommand{\risk}{\mathcal R}
\newcommand{\eqrisk}{\widetilde{\mathcal R}}
\newcommand{\frrisk}{\widetilde{\mathcal R}^{(0)}}
\newcommand{\gammastar}{\widetilde{\gamma}^*}

\newcommand{\bcrit}{\widetilde \bs_{\mathrm{crit}}}
\newcommand{\bmax}{\widetilde \bs_{\mathrm{max}}}
\newcommand{\gammastarb}[1][\bs]{\widetilde\gamma^*_{T,#1}}

\newcommand{\gaussT}[1]{{#1}_T^{\mathrm{gauss}}}

\crefname{assumption}{Assumption}{Assumptions}
\Crefname{assumption}{Assumption}{Assumptions}
\crefname{lemma}{Lemma}{Lemmas}
\Crefname{lemma}{Lemma}{Lemmas}

\title{\Large \textbf{Asymptotic Risk of Learning-Rate Schedules for Least-Squares Regression with Discretely Distributed Features}
\vspace{0.5em}}

\author{
  {\Large Noé Colomban\thanks{Corresponding author: \texttt{noe.colomban.2023@polytechnique.org}}, 
  Francis Bach,
  Fabian Schaipp} \\[.2cm]
   Inria - École Normale Sup\'erieure \\
PSL Research University
}
\date{}

\begin{document}

\maketitle

\begin{abstract}
 We study the last iterate of stochastic gradient descent (SGD) on streaming least squares with discretely distributed features (Slock's model), which decouples the dynamics across eigendirections. Under power-law source and capacity conditions, we derive the exact finite-time risk and asymptotic equivalents of constant and Warmup-Stable-Decay (WSD) schedules. With a tuned base learning-rate, the resulting scaling law decays as $\Theta(T^{-q})$, where  $T$ is the horizon, and $q$ is independent of the cooldown length $c \in (0,1]$ but is strictly smaller for the constant schedule: any cooldown of fixed relative length $c > 0$, however short, improves the risk by a polynomial factor, while its length only affects the constant, which is minimized by linear decay ($c = 1$). Finally, we show that the optimal risk depends on the batch size $b$ and horizon $T$ only through $bT$, up to a critical batch size.
\end{abstract}

\section{Introduction}

While classical stochastic approximation theory relies heavily on iterate averaging to achieve optimal convergence rates, modern practice almost exclusively uses the final iterate. In
particular, the Warmup-Stable-Decay (WSD) schedule \citep{hu2024minicpm, haegele2024scaling, zhai2022scaling} has emerged as
a paradigm for large-scale model training, departing from the cosine and polynomial decays
that were previously standard. Recent theoretical efforts have begun to bridge the gap between
these practical schedules and the exact dynamics of the last iterate. Strong convergence
guarantees for the final iterate have been established in the interpolation regime
\citep{varre2021last}, and the necessity of geometrically tailored decay phases was formally
captured by \textcite{ge2019step}. Building on this intuition, \textcite{schaipp2025} recently
demonstrated that worst-case convex optimization bounds already capture the shape of the loss curves for
schedules used in practice, and even predict how the optimal base learning-rate transfers
between them. 

Such bounds are, however, agnostic to the quantities that ultimately govern the 
convergence rate, most notably the label noise and the spectrum of the data covariance. They
therefore constrain the shape of a good schedule without expressing the resulting risk in terms
of the difficulty of the problem, which is needed to quantify how much a decay phase
actually buys.

By restricting our focus to the streaming least-squares problem, we trade generality for
exactness. Since this setting explicitly accounts for the data spectrum and the noise, it admits
an exact bias-variance decomposition of the final iterate's risk. Exactness comes at a price:
computing the risk in closed form requires a covariance model that decouples the coordinates \citep{slock1993}, and we assume power-law decay of both the spectrum and the initial distance. Under these assumptions, we obtain the exact finite-time risk of every schedule
considered, and asymptotic equivalents that allow a precise comparison between them.

\paragraph{Summary and contributions.} In this paper, we study the WSD family, parameterized by
a relative cooldown length $c \in (0,1]$, together with the constant schedule,
which we treat separately. Under Slock's model and power-law source and capacity
conditions, we obtain the exact finite-time risk of the last iterate
(\cref{th:risk}) and its asymptotic equivalent for each schedule
(\cref{th:constant,th:wsd}), from which we derive the
asymptotically optimal base learning-rate $\gammastar_T$ and the resulting risk.
For a horizon $T$, the optimal risk is $\Theta(T^{-q})$ with an exponent $q$ that does not depend
on~$c$ on the whole interval $(0, 1]$, yet is strictly smaller for the constant
schedule. These exponents agree with those of \textcite{li2025functional}, derived up to constants: any cooldown of fixed relative length improves the risk polynomially in the horizon. Our exact equivalents also resolve the constant, minimized at $c^*=1$ (pure linear decay), and make learning-rate transfer across cooldown lengths explicit. Moreover, when the capacity and source exponents of \cref{as:alphabeta} satisfy $\alpha<2$ and $\beta<\alpha+1$, the optimal rate of WSD coincides with that of averaged SGD with a constant step size: a cooldown thus recovers, for the last iterate, the rate of averaging (which is unimprovable), up to a constant (\cref{sec:expression_gammastar}). Finally, we extend the analysis to a batch size
$b$ and show that the optimal risk depends on $b$ and $T$ only through $bT$, up
to a critical batch size that we can characterize.

\paragraph{Setup.}\label{sec:setup}
Let $\featvec $ and $\labelscl$ be random variables taking values in $\R^d$ and $\R$, respectively.
We consider the least-squares problem:
\begin{equation}\label{eq:least_squares}
\min_{\vect\in \R^d} f(\vect), \quad f(\vect) := \frac{1}{2} \E\left[(\featvec ^\top \vect - \labelscl)^2\right],
\end{equation}
where the expectation is over the joint distribution of $(\featvec , \labelscl)$ on $\R^d \times \R$.

Let $\vect_*$ be a solution to \cref{eq:least_squares}, and $f_* := f(\vect_*)$. We denote $\varepsilon:=\labelscl-\featvec ^\top \vect_*$ the label noise, and $\mathcal D$ the joint distribution of $(\featvec , \varepsilon)$.

Assume $H:=\E[\featvec  \featvec ^\top]$ exists and $\E[\labelscl^2]<\infty$. Then $f$ is differentiable with $\nabla f(\vect) =  H(\vect-\vect_*) - \E[\varepsilon \featvec]$ (see \cref{app:first_proofs} for the proof). Since $\vect_*$ is a minimizer of $f$ over $\R^d$, $\nabla f(\vect_*) = 0$, so $\E[\varepsilon \featvec ] = 0$, and
\begin{equation}\label{eq:gradient_f}
\nabla f(\vect)=H(\vect-\vect_*).
\end{equation}
$H$ is a symmetric positive semi-definite matrix, so it admits an orthogonal eigendecomposition $H = Q \Lambda Q^\top$. Here, $Q \in \R^{d \times d}$ is an orthogonal matrix and $\Lambda$ is a diagonal matrix containing the eigenvalues $\lambda = (\lambda_1, \dots,\lambda_d) \in \R^d$ such that $\lambda_1 \ge \lambda_2 \ge \dots \ge \lambda_d \ge 0$. For any $i \in \{1, \dots, d\}$, we denote the $i$-th eigenvector by $q_i = Q e_i$, where $e_i$ is the $i$-th vector of the canonical basis of $\R^d$.

\paragraph{Stochastic gradient descent (SGD).} Let $\vect_0 \in \R^d$. We study SGD with a learning-rate schedule,  where the update rule is given by:
\begin{equation}\label{eq:sgd}
    \vect_{t+1} =\vect_t - \gamma\eta_tg_t, \quad \E[g_t \mid\mathcal F_t] = \nabla f(\vect_t),
\end{equation}
denoting by $(\mathcal F_t)_{t \ge 0}$ the filtration generated by
$(\vect_0, \vect_1, \dots, \vect_t)$, so that $\mathcal F_t$ contains the
information available up to time $t$. In least-mean-squares, the unbiased
gradient $g_t$ is obtained as follows.

We draw i.i.d.~$(\featvec _t,\varepsilon_t)_{t\geq 0}$ from $\mathcal D$. Then $\labelscl_t := \featvec _t^\top \vect_* + \varepsilon_t \stackrel{d}{=} \labelscl$. Set $g_t = \featvec _t(\featvec _t^\top \vect_t- \labelscl_t)$.
Here, $\E[g_t \mid \mathcal F_t]= H(\vect_t-\vect_*)=\nabla f(\vect_t)$ by \cref{eq:gradient_f}. Hence $g_t$ is an unbiased estimator of the gradient $\nabla f(\vect_t)$, and SGD with a starting point $\vect_0$ is given by
$\vect_{t+1} = \vect_t - \gamma \eta_t  \featvec _t(\featvec _t^\top \vect_t-\labelscl_t) = (I - \gamma \eta_t \featvec _t \featvec_t^\top )\vect_t + \gamma \eta_t \featvec _t \labelscl_t$, which gives:
\begin{equation}\label{eq:sgd_batch1}
\vect_{t+1}-\vect_* = (I-\gamma\eta_t \featvec _t \featvec _t^\top)(\vect_t-\vect_*) + \gamma \eta_t \varepsilon_t \featvec _t.
\end{equation}

\paragraph{Learning-rate schedule.} \label{par:LR} In our analysis, for a fixed horizon $T$, we define a learning-rate schedule by two components: a \emph{base learning-rate} $\gamma \in \mathbb{R}_{>0}$ scaling the overall step size, and a \emph{schedule} sequence $(\eta_t)_{t=0}^{T-1} \in \mathbb{R}_{\geq0}^T$ dictating the update rule over $T$ steps. The effective step size is then $\gamma \eta_t$. We focus on the family of schedules known as Warmup-Stable-Decay (WSD), without warmup, defined as follows.

Let $c \in (0,1]$. We consider the WSD schedule (without warmup):
\[
\eta_j =  \min\Big\{ 1, \frac{T-j}{cT} \Big\} = \begin{cases} 
1 & \text{if } j \leq (1-c)T, \\
\frac{T-j}{cT} & \text{if } j > (1-c)T.
\end{cases}
\]
Here, $c$ represents the relative cooldown length, following \textcite{zhai2022scaling}. With this definition, the schedule begins its linear decay at step $\lfloor(1-c)T\rfloor+1$. Note the following boundary behaviors:
\begin{itemize}

\item When $c=1$, the model recovers the pure linear decay schedule: $\eta_j =  \frac{T-j}{T}$.

\item As $c \to 0$, the schedule approaches a constant rate: $\eta_j = 1$. To ensure the decay update remains well-defined, the strict constant schedule ($c=0$) is excluded from the WSD family and treated independently.

\end{itemize}

\paragraph{Notations.} \label{par:notations} We denote the trace of a matrix or a linear operator by $\tr(\cdot)$. For a vector $u \in \R^d$ (or a sequence $(u_i)_{i \in \mathbb N} \in \R^{\mathbb{N}}$ in the infinite-dimensional setting) 
$\diag(u)$ represents the diagonal matrix (resp. operator) whose main diagonal is given by $u$. The norm $\|\cdot \|$ will represent the Euclidean norm. Throughout the paper, $I$ will denote the identity operator.

We define the $\Gamma$ (Gamma) function for any real number $x>0$ as $\Gamma(x):=\int_0^\infty t^{x-1}e^{-t}dt$ and the $\zeta$ (zeta) function for $s>1$ as $\zeta(s) := \sum_{n=1}^\infty n^{-s}$. 

Throughout, a tilde marks a quantity attached to the asymptotic risk $\eqrisk_T$ where $T\to \infty$ rather than the exact one $\risk_T$.

\subsection{Slock's Model}
\label{subseq:slock}

To study least-squares problems, a common framework is to suppose Gaussian inputs $\featvec  \sim \mathcal N(0,H)$, in addition to $\sigma^2 = \var(\varepsilon)$ (see \cite{lin2024scaling}, Assumption 1). However, the computation of the recurrence relation in \cref{eq:sgd_batch1} involves the fourth-order moment $\E[\featvec _t \featvec _t^\top \vect_t \vect_t^\top \featvec _t \featvec _t^\top]$. Under the standard Gaussian framework, this term couples the dynamics across all dimensions, making an exact coordinate-wise analysis intractable. We therefore seek an alternative distribution for $\featvec $ that decouples these dimensions. One such approach relies on a discrete model \citep{slock1993}.

\begin{assumption}[Slock's Model]\label{as:slock}
Set $\sigma>0$. Let $(\varepsilon_t)_{t\geq0}$ be a sequence of i.i.d.~random variables with zero mean and $\var(\varepsilon_t)=\sigma^2$. Let $(\mu_t)_{t\geq0}$ be i.i.d. random vectors in $\R^d$, with distribution:
\begin{equation}
    \mu_t = r\, q_i, \quad i\sim \pi,
\end{equation}
where $\pi_i = \frac{\lambda_i}{\tr(\Lambda)}$ for $i\in \{1,\dots,d\}$ and $r=\sqrt{\tr(\Lambda)}$.

Let $(s_t)_{t\geq 0}$ be a sequence of i.i.d.~random variables satisfying $\E[s_t]=0$ and $s_t^2=1$ a.s.

We assume that $(\mu_t), (s_t)$ and $(\varepsilon_t)$ are mutually independent. At each step $t$, we draw $\featvec _t := s_t \mu_t$ and $\varepsilon_t$.
\end{assumption}

By construction, \cref{as:slock} implies $\E[x_t x_t^\top] = \sum_{i=1}^d
\frac{\lambda_i}{\tr(\Lambda)} r^2 q_i q_i^\top = Q\Lambda Q^\top = H$; while
$\E[x_t] = \E[s_t]\E[\mu_t] = 0$, since $s_t$ and $\mu_t$ are independent and
$s_t$ is zero mean.

Each sample activates a single eigendirection: in the eigenbasis, $Q^\top x_t =
s_t\, r\, e_i$ with $i \sim \pi$, so that $\|x_t\|^2 = \tr(\Lambda)$ almost
surely and $x_t x_t^\top = \tr(\Lambda)\, q_i q_i^\top$. Two samples are
therefore collinear when they share the same index and orthogonal otherwise, and
all the matrices $x_t x_t^\top$ are diagonalized by the same $Q$; in
particular they commute. The iteration in \cref{eq:sgd} thus preserves the
eigenbasis of $H$ and updates one coordinate at a time, drawn with probability
proportional to $\lambda_i$: this is what keeps the fourth-order moment diagonal
and reduces the risk to the scalar recursions of \cref{eq:recurrence_moment}.

One-hot features are an instance of the model: if $x = e_i$ with probability $p_i$, then
$H = \diag(p)$, $\tr(\Lambda) = 1$ and $\pi = p$, so \cref{as:slock}
recovers exactly this distribution up to the sign $s_t$, which plays no role in
\cref{eq:recurrence_moment}. This one-hot setting is closely related to the linear bigram model of
\citet{kunstner2025scaling}, whose inputs are one-hot encodings of tokens with
power-law frequencies.
Although Slock's model is chosen for tractability, it is not restrictive in the regime of interest: along the vanishing base learning-rates selected by our analysis, its bias and variance match those of Gaussian features (see \cref{app:gauss}).

\section{Related Work}

\paragraph{Learning-rate schedules.} The Warmup-Stable-Decay (WSD) schedule \citep{zhai2022scaling,hu2024minicpm,haegele2024scaling} has been proposed as an alternative to the cosine schedule. \textcite{haegele2024scaling} show that WSD matches the performance of cosine on LLM pretraining, while reducing the compute needed for scaling-law experiments, as the constant part of the schedule can be reused. In the worst-case convex setting, \textcite{defazio2023optimal} show that linear decay is the optimal schedule for the last iterate of SGD, and that it performs well in deep learning practice.

\paragraph{Least squares and optimal schedules.}
\citet{li2025functional} derive, up to constants, scaling laws for constant, exponential-decay and WSD-like schedules in power-law kernel regression; their $(s,\beta)$ are our $(\omega,\alpha)$. Our exact equivalents also resolve the constants, which determine the optimal cooldown $c^*$ and learning-rate transfer. \textcite{bordelon2026theory} derive the optimal schedule of a power-law random feature model through optimal control; with $b=\alpha$ and $a=\alpha+\beta$, $(a-1)/b=\omega$, placing our setting in their \emph{easy} phase ($a>b$). \textcite{meterez2026anytime} share our least-squares setting but study horizon-free schedules, where averaging recovers minimax rates without knowledge of $T$; we compare last iterates at fixed $T$.

\paragraph{Cooldown length.}
\textcite{schaipp2025} concentrate on last-iterate scheduling, and suggest that, with a tuned base learning-rate $\gamma$, the best cooldown length for WSD is $c^*=1$, which corresponds to a pure linear decay schedule. We obtain the same conclusion in our framework (\cref{th:c_star}), minimizing the asymptotic risk. \textcite{haegele2024scaling} find optimal cooldowns strictly inside $(0,1)$, but at a base learning-rate that is not retuned as a function of the horizon and of the cooldown length.

\paragraph{Batch size.} Increasing the batch size shortens a run
by trading steps for parallelism, but only up to a critical batch size. \textcite{zhang2025how} show, via
an infinite-dimensional least-squares analysis backed by language model experiments, that this threshold is governed by the data budget rather than the model size.

\section{Computation of the Risk}

For a fixed horizon $T$, we define the risk at step $t$ and the last-iterate risk as:
\begin{equation}
\risk_{t\mid T} := \E[f(\vect_t)] - f_*, \qquad \risk_T := \risk_{T\mid T}.
\end{equation}
In this section, our goal is to compute a closed form for the last-iterate excess risk
$\mathcal R_T$. Since Slock's model decouples the dynamics
across coordinates, an exact coordinate-wise analysis is tractable: we first
derive the finite-time risk $\mathcal R_{t\mid T}$ for every $1 \le t \le T$
(Theorem~\ref{th:risk}; proofs in \cref{app:base}). Writing $H = Q \Lambda Q^\top$
and the second-order moment matrix of the parameter error in the eigenbasis $M_t = Q^\top \E[(\vect_t - \vect_*)(\vect_t - \vect_*)^\top] Q$, the risk depends only on the
diagonal $m_t$ of $M_t$, through \cref{eq:risk_trace}.
\begin{equation}\label{eq:risk_trace}
\mathcal R_{t\mid T} = \frac12 \tr(\Lambda M_t) = \frac12 \lambda^\top m_t.
\end{equation}
We have $M_{t+1}= M_t - \gamma \eta_t \Lambda M_t - \gamma \eta_t M_t \Lambda + \gamma^2\eta_t^2  \E[\featvec_t \featvec_t^\top \vect_t \vect_t^\top \featvec_t \featvec_t^\top] + \gamma^2\eta_t^2\sigma^2\Lambda$. Under \cref{as:slock}, the fourth-order moment is diagonal and equal to $\tr(\Lambda)\Lambda\diag(m_t)$, and the recurrence for $m_t$ becomes for $t \in \{0,\dots,T-1\}$:
\begin{equation}\label{eq:recurrence_moment}
m_{t+1} = \left[I + (-2\gamma\eta_t + \gamma^2\eta_t^2\,\tr(\Lambda))\Lambda\right]m_t
+ \gamma^2\eta_t^2\sigma^2\lambda.
\end{equation}

As is classical, we observe that the risk decomposes into two terms \citep{lin2024scaling}: the first corresponds to the noiseless dynamics ($\sigma^2=0$), and the second represents the variance accumulated over $t$ steps.
We define the \emph{bias} $B_{t\mid T}$ and \emph{variance} $V_{t\mid T}$ below. We denote: 
\begin{equation*}\label{eq:B_V}
\begin{aligned}
    &\Pi(k,t,(\gamma\eta_j),\Lambda) := \prod_{j=k}^{t-1}[I - 2\gamma\eta_j \Lambda + \gamma^2\eta_j^2\tr(\Lambda)\Lambda],
    \\
    &B_{t\mid T} :=  \frac{1}{2}\lambda^\top  \Pi(0,t,(\gamma\eta_j),\Lambda)m_0, 
    \\
    &V_{t\mid T} :=  \frac{1}{2} \lambda^\top \sum_{k=0}^{t-1}\gamma^2\eta_k^2\sigma^2 \Pi(k+1,t,(\gamma\eta_j),\Lambda)\lambda.
\end{aligned}
\end{equation*}
In the same way, we can define the last-iterate bias and variance by $B_T = B_{T\mid T}$ and $V_T = V_{T\mid T}$.

\begin{theorem}\label{th:risk}
Let $m_0\in \R^d$ be the diagonal of $Q^\top(\vect_0-\vect_*) (\vect_0-\vect_*)^\top Q$.
Then, after $t$ steps of SGD, $t \in \{1,\dots,T\}$, the expected excess risk is given by:
\begin{equation}
    \mathcal R_{t\mid T} = B_{t\mid T} + V_{t\mid T}.
\end{equation}
\end{theorem}
We assume a power-law spectrum for the data covariance (the capacity condition) and a power-law decay of the initial distance $\vect_0-\vect_*$ in the eigenbasis (the source condition) in \cref{as:alphabeta}.
These source and capacity conditions are standard in the least-squares literature \citep{caponnetto2007optimal, lin2024scaling}. Modeling task difficulty this way is precisely what allows one to theoretically recover the neural scaling laws empirically observed in large-scale training \citep{lin2024scaling,kunstner2025scaling}.

\begin{assumption}[Power-Law Spectrum and Initial Distance]\label{as:alphabeta}
There exist $\alpha>1, \beta>0$ and $L,\Delta>0$ such that for all $i \geq 1$, 
\begin{equation}
    m_{0,i} = \Delta i^{-\beta} \quad \text{and} \quad \lambda_i = L i^{-\alpha}.
\end{equation}
\end{assumption}

In particular, $\tr(\Lambda)$ is well-defined and its value is~$L\sum_{i=1}^d i^{-\alpha}$. Under \cref{as:alphabeta}, we define the constant $\omega:=\frac{\beta-1}{\alpha}+1.$

Thanks to the decoupling property of Slock's model, the matrices $I - 2\gamma\eta_k \Lambda + \gamma^2\eta_k^2\tr(\Lambda)\Lambda$ are diagonal and polynomial in $\Lambda$. Therefore, \cref{as:alphabeta} allows us to express the scalar products $B_{t \mid T}$ and $V_{t \mid T}$ as explicit sums where the terms are rational functions of the index $i$. This yields the formulas detailed in \cref{th:risk_alpha_beta}.

\paragraph{Empirical observation of \cref{th:risk}.}
To illustrate the correctness of \cref{th:risk} under \cref{as:alphabeta}, we simulate the empirical SGD dynamics by drawing random samples according to Slock's model. We evaluate the excess risk under two standard learning-rate schedules: a constant schedule $\eta_t = 1$, and a linear decay schedule $\eta_t =  \frac{T-t}{T}$ which depends on a predefined horizon $T$. \cref{fig:sgd_comparison} shows the agreement between the empirical loss and our theoretical formulas.

\begin{figure}
    \centering
    \captionsetup[subfigure]{skip=2pt, belowskip=-4pt}
 
    \begin{subfigure}[b]{0.48\textwidth}
        \centering
        \includegraphics[width=\textwidth]{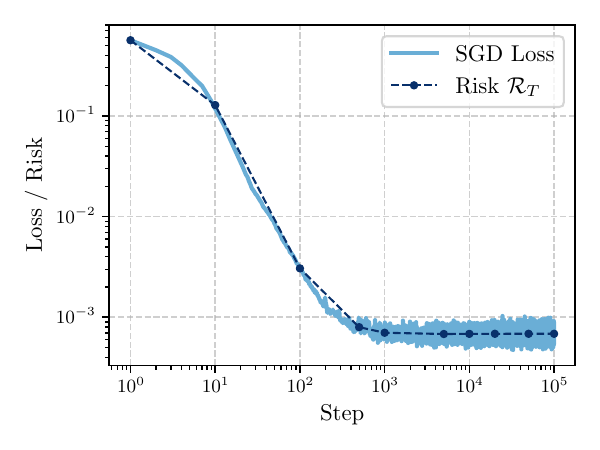}
        \caption{Constant schedule}
        \label{fig:sgd_constant}
    \end{subfigure}\hfill
    \begin{subfigure}[b]{0.48\textwidth}
        \centering
        \includegraphics[width=\textwidth]{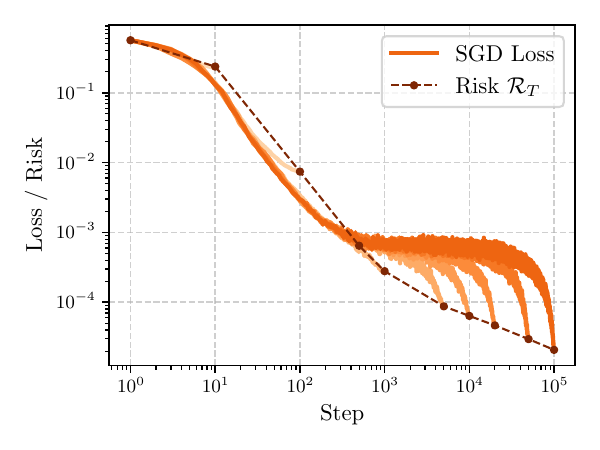}
        \caption{Linear schedule}
        \label{fig:sgd_linear}
    \end{subfigure}

    \caption{Comparison of the empirical SGD excess risk against the theoretical final risk from \cref{th:risk_alpha_beta}. Both experiments are averaged over $n = 50$ independent runs with parameters $\gamma = 0.1, L = \Delta = 1$, $\alpha = 1.5$, $\beta = 2$, $\sigma = 0.1$ and $d=100$. For the linear schedule (b), the plot displays the entire risk trajectories over time, where each trajectory is computed with a schedule decaying to zero at its specific horizon $T$.}
    \label{fig:sgd_comparison}
\end{figure}

Note that a constant schedule is horizon-free, so a single run suffices and the theoretical risk can be evaluated at every $t$ along it. The linear schedule, by
contrast, depends on $T$ through $\eta_t$: each horizon defines a different schedule, so a separate run is needed for each~$T$, and the theoretical curve matches the empirical loss only at its own final step $t=T$. This is why \cref{fig:sgd_comparison} displays one trajectory per horizon.

\paragraph{Comparison with Gaussian features.} For Gaussian inputs $x \sim \mathcal{N}(0,H)$, the exact risk remains computable, but the fourth-order moment couples all coordinates, which makes asymptotic equivalents out of reach. This coupling term is of order $\gamma^2$, and we observe numerically (\cref{app:gauss}) that along a vanishing base learning-rate $\gamma_T \to 0$, which is the regime selected by the optimal $\tilde\gamma^*_T$ (\cref{tab:gamma_star_combined}), the ratios of the biases and of the variances of the two models tend to 1.

\section{Asymptotic Analysis}\label{sec:equivalents}

In the following theorems, the equivalence between two sequences $(u_T)$ and $(v_T)$, denoted $u_T \sim v_T$, is defined by $u_T - v_T = o(v_T)$. There are no issues regarding the randomness of the variables, as all quantities are evaluated in expectation per \cref{th:risk}.

\paragraph{Dimension.} Since our primary goal is to derive tractable, exact formulas for the last-iterate risk, we consider the infinite-dimensional setting $d\to \infty$, where the parameter space becomes the set of sequences $\R^{\mathbb N}$. The infinite-dimensional limit is a standard setting in least-squares literature to derive exact analytical solutions and bridge the gap with kernel methods \citep{caponnetto2007optimal}. The trace $\tr(\Lambda) = L\zeta(\alpha)$ is still well-defined as $\alpha>1$. 

\paragraph{Admissible values of $\gamma$.} 
We must ensure that the terms $[1 - 2\gamma\eta_j L i^{-\alpha} + \gamma^2\eta_j^2 \tr(\Lambda)L i^{-\alpha}]$ do not cause the bias or the variance to diverge. Therefore, they must lie strictly between $-1$ and $1$. These terms are automatically positive, since their minimum over $\eta$ is $1 - Li^{-\alpha}/\mathrm{Tr}(\Lambda) > 0$ (\cref{lemma:stability}), so the constraint reduces to $1 - 2\gamma\eta L i^{-\alpha} + \gamma^2\eta^2\mathrm{Tr}(\Lambda) L i^{-\alpha} < 1$ for all $i \geq 1$ and $\eta \in (0,1]$.

\begin{assumption}[Stability Condition] \label{as:eta}
We assume the base learning-rate $\gamma$ satisfies:
\begin{equation}\label{eq:eta}
   0 < \gamma < {2}\big/{\tr(\Lambda)}.
\end{equation}
\end{assumption}
We recover Equation 2.4 of \textcite{defossez2015averaged}. We denote $\gammax := \frac{2}{\tr(\Lambda)}$.
The asymptotic equivalents are derived using the $z$-transform techniques described by \textcite{bach2025ztransform}.
For the decaying schedules (WSD including linear), we will identify two distinct regimes for the asymptotic variance $V_T$, depending on the spectral decay: $1 < \alpha < 2$ and $\alpha > 2$. The critical case $\alpha = 2$ typically introduces logarithmic corrections and is omitted from our analysis for simplicity.

\subsection{Constant Schedule}

There is no need to choose a horizon in this case so~$T$ may represent any integer. We find that the variance part $V_T$ tends to a limit $V_\infty$ proportional to~$\sigma^2$, while the bias part decreases as an order $T^{-\omega}$.

\begin{theorem}\label{th:constant}
Let $\eta_j = 1$ for all $j\geq 0$.
Under \cref{as:alphabeta} and \cref{as:eta}, as $T \to \infty$, we have
$B_T \sim \widetilde B_T$ and $V_T \to V_\infty$, where

\begin{equation}\label{eq:constant}
\begin{aligned}
\widetilde B_T &:= \frac{L\Delta}{2\alpha}\, \Gamma(\omega)
  \big((2-\gamma\tr(\Lambda))L\gamma T\big)^{-\omega}, \\[-.05cm]
V_\infty &:= \frac{\zeta(\alpha)}{2}\,
  \frac{\sigma^2 L\gamma}{2-\gamma\tr(\Lambda)}.
\end{aligned}
\end{equation}
\end{theorem}
In a common Gaussian features model that assumes $\featvec \sim \mathcal N(0,H)$, the bias part
is, to leading order in $\gamma$, $\frac{L\Delta}{2\alpha} \Gamma\left(\omega\right)(2\gamma L)^{-\omega}T^{-\omega}$
if $\beta < \alpha$ \citep{bach2025ztransform}. Therefore only the constant is different, and the
power of $T$ is the same. We recover exactly this behavior in Slock's model as the trace penalty
$\gamma \tr(\Lambda) \to 0$. \cref{sec:eta_star} will show that optimal base learning-rates tend
to 0 as $T\to \infty$ so the results are asymptotically equivalent (for $\beta > \alpha$, we
observe this numerically, see \cref{app:gauss}).

\paragraph{Experimental observation.} We compare the exact risk terms $B_T$ and $V_T$ with the equivalent $\widetilde B_T$ and $\widetilde V_T$ defined in \cref{eq:constant}, by plotting the ratios $\widetilde B_T/B_T$ and $V_\infty/V_T$, which should therefore tend to 1. We choose a fixed base learning-rate $\gamma:= 10^{-3}  \in (0, \gammax)$, and a set $\mathcal S_\alpha$ of values of $\alpha$: $\mathcal S_\alpha := \{1.4,1.8,2.2,2.6\}$.

As illustrated in \cref{fig:asymptotic_constant} in \cref{app:figures}, the asymptotic equivalent of the bias converges to its exact value faster than that of the variance. For values of $\alpha$ close to 1, the variance requires a larger horizon $T$ to approach its asymptote (e.g., at $T=10^5$, the ratio is still above $1.09$). This discrepancy occurs because the leading-order equivalent for the bias retains its dependence on $T$, whereas for the variance, all $T$-dependent terms are discarded in the constant limit. Consequently, the variance ratio converges only as fast as these discarded transient terms decay, which is slow when~$\alpha \to 1$.

\subsection{WSD Schedule}
We now turn to the general WSD schedule, parameterized by a relative cooldown length $c \in (0,1]$. Unlike the constant learning-rate, the WSD schedule explicitly depends on the total number of steps $T$. Therefore, the bias~$B_T$ and variance $V_T$ are evaluated exactly at the final step $T$. Depending on the spectral decay, two distinct asymptotic regimes arise: a regime where small dimensions dominate ($\alpha > 2$), and one where large dimensions dominate ($1 < \alpha < 2$). The total variance explicitly decomposes into two components: the variance accumulated during the constant phase ($V^A_T$) and the variance accumulated during the decay phase ($V^B_T$). Both components remain significant in the leading-order asymptotics when large dimensions dominate ($1 < \alpha < 2$).

Let $c \in (0,1]$. We consider the WSD schedule of relative cooldown length $c$.

\begin{theorem}\label{th:wsd}
Let $(\eta_j)$ be the WSD schedule with a cooldown length of $c$, and set
$\kappa_1 := 2-c-\gamma\tr(\Lambda)\big(1-\frac{2c}{3}\big)$ and
$\kappa_2 := c\big(1-\frac{\gamma\tr(\Lambda)}{3}\big)$.
Under \cref{as:alphabeta} and \cref{as:eta}, as $T \to \infty$, we have
$B_T \sim \widetilde B_T$ and $V_T \sim \widetilde V_T$, where
\begin{align}
\widetilde B_T &:= \frac{L\Delta}{2\alpha}\,\Gamma(\omega) \,
  (\kappa_1 L\gamma T)^{-\omega},
  \label{eq:bias_equiv} \\
\widetilde V_T &:= \mathcal{W}_\gamma(\alpha,c)\,\sigma^2 (L\gamma)^{\frac{1}{\alpha}}\,
  T^{\frac{1}{\alpha}-1}
  \quad \text{if } 1<\alpha<2,
  \label{eq:var_equiv_small} \\
\widetilde V_T &:= \frac{\sqrt{\pi}}{8}\,\zeta\big(\frac{\alpha}{2}\big)\,
  \sigma^2 (L\gamma/c)^{\frac12}\, T^{-\frac12}
  \quad \text{if } \alpha>2,
  \label{eq:var_equiv_large}
\end{align}
with \\
$\xi := 1/\alpha-1$, $\mathcal{I}_{\alpha,\gamma} := \int_0^1 x^{2\xi}(1-\gamma\tr(\Lambda)\,x/3)^{\xi-1}\,dx$
and $\mathcal{W}_\gamma(\alpha,c) := \frac{1}{2\alpha}\big[\frac{\Gamma(-\xi)}{2-\gamma\tr(\Lambda)}(\kappa_2^{\xi}-\kappa_1^{\xi}) + \Gamma(1-\xi)\,\mathcal{I}_{\alpha,\gamma}\,c^{\xi}\big]$.
\end{theorem}

Note that the order $q$ of $\Theta(T^{-q})$ is independent of the relative cooldown length $c > 0$, in both the bias and the variance: $c$ only affects the constant
in front of~$T$. The contrast is with the constant schedule, whose variance does not vanish: any fixed relative cooldown length $c > 0$ suffices to change the rate.

\paragraph{Experimental observation.} We plot in \cref{fig:wsd_equivalent} in \cref{app:figures}, like in the previous section, the ratios $\widetilde B_T/B_T$ and $\widetilde V_T/V_T$, for a cooldown length $c=0.2$.

The exact variance balances two contributions: one from the leading dimensions (the first few dimensions with large eigenvalues) which decays as $O(T^{-1/2})$, and one from the tail dimensions (the infinite accumulation of dimensions with vanishing eigenvalues) which decays as $O(T^{(1-\alpha)/\alpha})$. Consequently, the relative error between the exact variance and its asymptotic equivalent scales as $O(T^{-\frac{|2-\alpha|}{2\alpha}})$. As $\alpha \to 2$, this convergence exponent approaches zero. The remainder thus decays almost at the exact same rate as the main term, requiring larger values of $T$ for the ratio to reach $1$.

\subsection{Summary of All Equivalents}
\label{sec:all_equivalent}
For clarity, the correction factors $1-\frac{\gamma\tr(\Lambda)}{2}$ and
$1-\frac{\gamma\tr(\Lambda)}{3}$ appearing in \cref{th:constant,th:wsd} are frozen at
their $\gamma\to0$ limit, so that \cref{tab:equiv_combined} displays leading order in
$\gamma\tr(\Lambda)$ rather than the exact constants. This is not an additional assumption:
\cref{lemma:eta_star_to_0} shows, independently of any such simplification, that every
minimizer of the asymptotic risk vanishes as $T\to\infty$, so $\gamma\tr(\Lambda)\to0$ along
the optimal sequence and the frozen constants are the relevant ones as soon as the
learning-rate is tuned. We define:
\begin{equation}\label{eq:mathcalW}
\mathcal W(\alpha,c) := \frac{\Gamma(1-1/\alpha)}{4\alpha}\left[\frac{\alpha}{2-\alpha}\,c^{1/\alpha-1} - (2-c)^{1/\alpha-1}\right],
\end{equation}
the limit of $\mathcal W_\gamma(\alpha,c)$ in \cref{th:wsd} as $\gamma \to 0$.

Two features stand out. First, unsurprisingly, the
bias decreases with $\gamma$ while the variance increases, and both dependencies are pure
powers of $\gamma$; this is what will make the trade-off solvable in closed form in
\cref{eq:frrisk}. Second, at fixed $\gamma$ the variance of the constant schedule does
not vanish as $T \to \infty$, whereas it does for every decaying schedule.

This last point already suggests an advantage for the decaying schedules, but it also shows
why \cref{tab:equiv_combined} cannot settle the comparison: the constant schedule is only
meaningful along a vanishing sequence $\gamma_T \to 0$, so the schedules can be compared
only once each is run at its own optimal base learning-rate (see \cref{sec:eta_star}).

\begin{table} 
    \centering
    \setlength{\tabcolsep}{3pt}
    \renewcommand{\arraystretch}{1.3}
    \begin{tabular}{@{}lcc@{}}
        \toprule
        \textbf{Sched.} & \textbf{Bias} $B_T$ & \textbf{Variance} $V_T$ \\
        \midrule
        Cst.
            & $L\Delta\,(2L\gamma T)^{-\omega}$
            & $\sigma^2 L\gamma$ \\
        WSD$_{\alpha<2}$
            & $L\Delta\,\big((2-c)L\gamma T\big)^{-\omega}$
            & $\mathcal{W}(\alpha,c)\sigma^2 (L\gamma)^{1/\alpha}T^{1/\alpha-1}$ \\
        WSD$_{\alpha>2}$
            & $L\Delta\,\big((2-c)L\gamma T\big)^{-\omega}$
            & $\sigma^2 (L\gamma/c)^{1/2}\,T^{-1/2}$ \\
        \bottomrule
    \end{tabular}

    \caption{Equivalents of the bias and the variance for the studied schedules (constant: both regimes $1<\alpha<2$ and $\alpha>2$),
    with $\omega = (\alpha+\beta-1)/\alpha$. The cooldown length is $c \in (0,1]$. Constants depending only on $(\alpha,\beta)$ are omitted. $\mathcal W(\alpha, c)$ is defined in \cref{eq:mathcalW}.}  
    \label{tab:equiv_combined}
\end{table}

\section{Optimal Learning-Rate}
\label{sec:eta_star}

\newcommand{\widetildeR}{\widetilde{\mathcal R}}

Let $\risk_T(\gamma)$ be the excess risk after $T$ steps of SGD run with base learning-rate $\gamma$ and a given schedule, and let $\gamma^*_T \in \argmin_\gamma \mathcal R_T(\gamma)$. No closed form is available for $\gamma^*_T$, and it is not even known to be unique. $\widetilde B_T$ and $\widetilde V_T$ denote the asymptotic equivalents derived in \cref{th:constant,th:wsd}, and $\widetildeR_T := \widetilde B_T + \widetilde V_T$. We therefore try to minimize instead the asymptotic risk $\widetildeR_T(\gamma)$. 
By \cref{lemma:eta_star_to_0}, any minimizer of $\widetildeR_T(\gamma)$ vanishes as $T \to \infty$, so that the correction factors $1 - \gamma\tr(\Lambda)/2$ and $1 - \gamma\tr(\Lambda)/3$ can be frozen at their limit. Let us denote $\frrisk_T(\gamma)$ the asymptotic risk obtained by freezing them and $\gammastar_T$ its minimizer. Then we can write, with $v$ and~$r$ the exponents of the variance term in \cref{tab:equiv_combined}:

\begin{equation}\label{eq:frrisk}
\frrisk_T(\gamma) =  \underbrace{C_B \gamma^{-
\omega}T^{-\omega}}_{\text{bias}} + \underbrace{C_V \gamma^{v}T^{-r}}_{\text{variance}}.
\end{equation}
The minimizer $\gammastar_T$ of $\frrisk_T(\gamma)$ is optimal for the exact risk $\risk_T$:
\begin{theorem} \label{th:gammastar}
With $\widetilde\gamma^*_T$ defined as the minimizer of $\frrisk_T(\gamma)$, we obtain that the exact risk for the base learning-rate $\widetilde\gamma^*_T$ is asymptotically optimal (see more details in \cref{app:gammastar}):
\begin{equation}\label{eq:gammastar}
\inf_{\gamma \in (0, \gammax)} \risk_T(\gamma) \sim \risk_T(\gammastar_T),
  \qquad T \to \infty .
\end{equation}
\end{theorem}

\subsection{Optimal Risk}
\label{sec:expression_gammastar}

The optimal learning-rate
$\widetilde\gamma_T^*$ can be written as in \cref{tab:gamma_star_combined}, where the
constants $C_{\mathrm{cst}}, C_{\mathrm{wsd<2}}, C_{\mathrm{wsd>2}}$ depend only
on $(\alpha, \beta, c)$ and are given in \cref{app:constants}.

\begin{table}
    \centering
    \setlength{\tabcolsep}{4pt}
    \renewcommand{\arraystretch}{1.5}
    \begin{tabular}{@{}lll@{}}
        \toprule
        \textbf{Schedule} & \textbf{Regime} & $\boldsymbol{\widetilde\gamma^*_T}$ \\
        \midrule
        Constant & both
            & $\big(\frac{C_{\mathrm{cst}}\,\Delta}{\sigma^2 L^{\omega}}\big)^{\frac{1}{1+\omega}} T^{-\frac{\omega}{1+\omega}}$ \\
        WSD & $1<\alpha<2$
            & $\big(\frac{C_{\mathrm{wsd}<2}\,\Delta}{\sigma^2 L^{\beta/\alpha}}\big)^{\frac{\alpha}{1+\alpha\omega}} T^{-\frac{\beta}{1+\alpha\omega}}$ \\
        WSD & $\alpha>2$
            & $\big(\frac{C_{\mathrm{wsd}>2}\,\Delta}{\sigma^2 L^{\omega-1/2}}\big)^{\frac{2}{1+2\omega}} T^{-\frac{2\omega-1}{1+2\omega}}$ \\
        \bottomrule
    \end{tabular}

    \caption{Optimal learning-rate $\widetilde\gamma^*_T$ for each schedule and regime of $\alpha$. All exponents of $T$ are negative; in particular, for $\alpha > 2$, $\omega > 1 - 1/\alpha > 1/2$.}
    \label{tab:gamma_star_combined}
\end{table}

By \cref{lemma:eta_star_to_0},
the exponent of $\gammastar_T$ is $\vartheta = \frac{\omega-r}{v+\omega}$. Since
$v + r = 1$ in all cases, $\vartheta = 1 - \frac{1}{v+\omega}$, so that
$\risk_T(\gammastar_T) = \Theta(T^{-\frac{\omega}{v+\omega}})$. As $\omega$ depends on the problem alone, the schedules are compared through $v$ only: a smaller $v$ yields both a slower-vanishing $\widetilde{\gamma}^*_T$ (smaller $\vartheta$) and a faster rate (larger $\frac{\omega}{v+\omega}$), with $v = 1$ for the constant schedule against $v = 1/\alpha$ ($\alpha < 2$) or $v = 1/2$ ($\alpha > 2$) for WSD.

\paragraph{Experimental observation.} We empirically verify in \cref{fig:eta_star} the optimality of the base learning-rate $\gammastar_T$, which is derived for the asymptotic regime $T \to \infty$. By plotting the exact risk $\risk_T(\gamma)$ over a range of $\gamma$ at a fixed horizon $T$, we evaluate the discrepancy between the theoretical $\gammastar_T$ and the actual finite-time minimizer $\argmin_\gamma \risk_T(\gamma)$.

\begin{figure}
    \centering
    \captionsetup[subfigure]{skip=2pt, belowskip=-4pt}

    \begin{subfigure}[b]{0.48\textwidth}
        \centering
        \includegraphics[width=\textwidth]{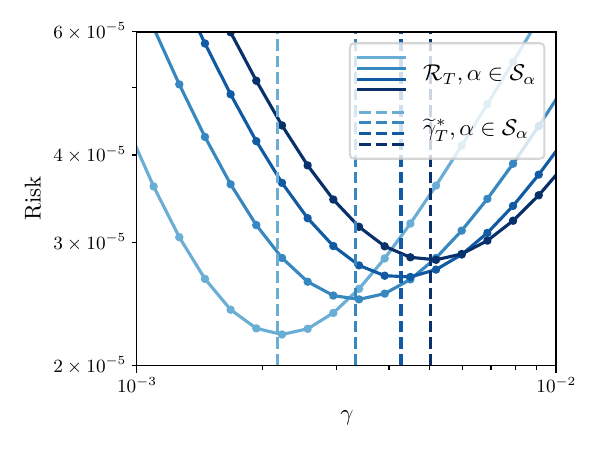}
        \caption{Constant schedule}
        \label{fig:sgd_constant}
    \end{subfigure}\hfill
    \begin{subfigure}[b]{0.48\textwidth}
        \centering
        \includegraphics[width=\textwidth]{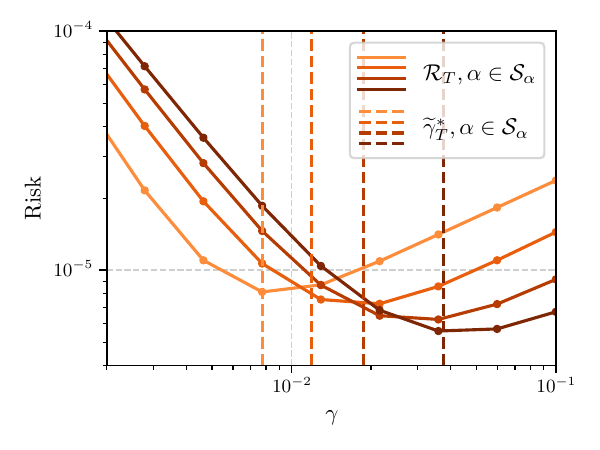}
        \caption{Linear decay ($c=1$)}
        \label{fig:sgd_linear}
    \end{subfigure}

    \caption{Excess risk as a function of the learning-rate $\gamma$ for different values of $\alpha \in \mathcal S_\alpha$, with $T=10^5$, $\beta=2$, $d=100$, $\sigma=0.1$, and $L=\Delta=1$. The vertical dashed lines correspond to the theoretical asymptotic optima $\gamma = \widetilde\gamma^*_T$. While these approximations appear to coincide with the minima of the risk curves under a constant schedule, a slight discrepancy is visible for the linear decay. This error arises because the asymptotic derivation requires $T \to \infty$, an effect that becomes pronounced near the regime change as~$\alpha \to 2$.}
    \label{fig:eta_star}
    
\end{figure}

\paragraph{Optimal equivalents, using $\widetilde\gamma^*$.}
We can substitute $\widetilde\gamma^*$ into \cref{tab:equiv_combined}, which leads to \cref{tab:optimal_scaling}.

\begin{table}[t]
    \centering
    \setlength{\tabcolsep}{4pt}
    \renewcommand{\arraystretch}{1.3}
    \begin{tabular}{@{}ll@{}}
        \toprule
        \textbf{Sched.} & \textbf{Optimal risk} \\
        \midrule
        Cst.
            & $(L\Delta)^{\frac{1}{1+\omega}}\, (\sigma^2/T)^{\frac{\omega}{1+\omega}}$ \\
        WSD$_{\alpha<2}$
            & $(L\Delta)^{\frac{1}{1+\alpha\omega}} \big(\mathcal{W}(\alpha,c)\,(2-c)^{-1/\alpha}\, \sigma^2/T\big)^{\frac{\alpha\omega}{1+\alpha\omega}}$ \\
        WSD$_{\alpha>2}$
            & $(L\Delta)^{\frac{1}{1+2\omega}} \big(\big(c(2-c)\big)^{-1/2}\, \sigma^2/T\big)^{\frac{2\omega}{1+2\omega}}$ \\
        \bottomrule
    \end{tabular}

    \caption{Equivalents obtained using the rate $\widetilde\gamma^*$ optimized for each schedule. The exponent of $\sigma^2/T$ is greater for WSD schedules than for a constant schedule, in both regimes of $\alpha$.
    Constants depending only on $(\alpha,\beta)$ are omitted except where they depend on $c$.}
    \label{tab:optimal_scaling}
\end{table}
Note that the exponent $T^{-\frac{\alpha+\beta-1}{2\alpha+\beta-1}}$ of the
constant schedule corresponds to that of \textcite{bordelon2026theory}
(consequence of their Result 1). As noted for $\gammastar_T$, the rate
$\Theta(T^{-\frac{\omega}{v+\omega}})$ depends on the schedule through $v$ alone, so the benefit of cooldown comes from $v_{\mathrm{wsd}} < v_{\mathrm{cst}} = 1$
and grows polynomially with the horizon. The relative cooldown length $c$, by
contrast, leaves the exponent unchanged and only enters the constant: within the
WSD family the choice of $c$ is a second-order effect. Moreover, their optimal risk $T^{-(\alpha+\beta-1)/(\alpha+\beta)}$ matches our WSD rate for $\alpha\in(1,2)$ but is better for $\alpha>2$.

We can compare \cref{tab:optimal_scaling} with an average $\bar \theta_T = \frac1T \sum_{t=1}^T \theta_t$. We compute in \cref{sec:average}, for a constant schedule, the corresponding bias $\bar B_T$ and variance $\bar V_T$. With $\omega':=\min\{\omega,2\}$ and $\omega\neq 2$, and assuming \cref{th:gammastar} still applies, this gives an optimal risk of $
\inf_\gamma\bar{\mathcal R}_T(\gamma)
 = \Theta \big( (L\Delta)^{\frac{1}{1+\alpha\omega'}}\big(\sigma^2/T\big)^{\frac{\alpha\omega'}{1+\alpha\omega'}} \big).$
This matches WSD for $\alpha<2$ as long as $\omega < 2$, since averaging saturates at $\omega=2$ \citep{dieuleveut2016nonparametric}. On the other hand, horizon-free step sizes $t^{-\mathfrak \chi}$ without averaging are known to be slower than their averaged counterparts \citep[Table~1]{bach2011nonasymptotic}.

\cref{fig:comparison_cst_wsd} in \cref{app:figures} shows the effect of reusing for WSD the base learning-rate tuned for the constant schedule.

\subsection{Cooldown Length and Learning-Rate Transfer}

For a relative cooldown $c$ and a horizon $T$, denote $\gammastar_T(c)$ the base learning-rate presented in \cref{tab:gamma_star_combined}. \textcite{schaipp2025} minimize a non-smooth convex optimization bound over the base learning
rate, and denote by $\gamma^*(c)$ the resulting optimum for a schedule with relative cooldown
length $c$. They observe that the ratio $\gamma^*(1)/\gamma^*(c)$ is independent of the
horizon $T$ (\cref{fig:schaipp_wsd_gammastar_c}), and report (their Figure~11) that this invariance is consistent with the learning
rates tuned for Llama-style architectures: for a $20\%$ cooldown they measure
$\gamma^*(1) \approx e^{0.7}\,\gamma^*(0.2)$.

The same invariance holds for our own optimal base learning-rate (see \cref{fig:wsd_gammastar_c})
$\widetilde{\gamma}^{\,*}_T(c)$, and for a reason that is transparent here: by Table~2, the
horizon enters $\widetilde{\gamma}^{\,*}_T(c)$ only through a factor $T^{-\vartheta}$ whose
exponent $\vartheta$ does not depend on $c$, the relative cooldown length affecting the
multiplicative constant alone. We can compute a closed form of $\log \left( {\gammastar_T(1)}/{\gammastar_T(c)} \right)$ from \cref{tab:gamma_star_combined},
\begin{equation}\label{eq:log_gammastar}
\log\Big(\frac{\gammastar_T(1)}{\gammastar_T(c)}\Big) =
\begin{cases}
    \frac{\alpha}{1+\alpha\omega}\,\log\left(\frac{(2-c)^{\omega}\,\mathcal{W}_c}{\mathcal{W}_1}\right) & \text{if } 1<\alpha<2,\\[6pt]
    \frac{2}{1+2\omega}\,\log\left(\frac{(2-c)^{\omega}}{\sqrt{c}}\right) & \text{if } \alpha>2,
\end{cases}
\end{equation}
where $\mathcal W_c := \mathcal W(\alpha,c).$

\begin{figure}
    \centering
    \captionsetup[subfigure]{skip=2pt, belowskip=-4pt}
    \begin{subfigure}{0.48\linewidth}
        \centering
        \includegraphics[width=\linewidth]{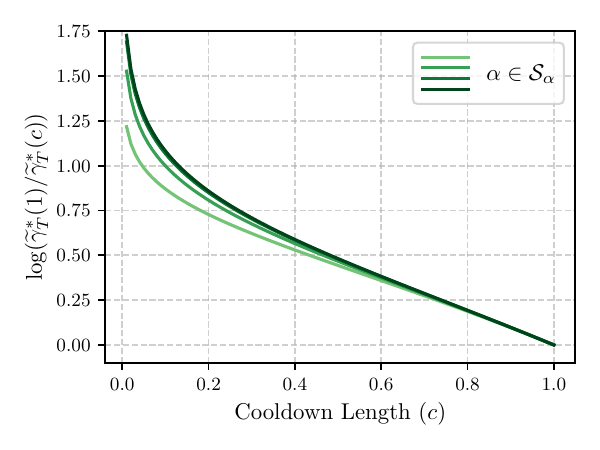}
        \caption{Our model.}
        \label{fig:wsd_gammastar_c}
    \end{subfigure}\hfill
    \begin{subfigure}{0.48\linewidth}
        \centering
        \includegraphics[width=\linewidth]{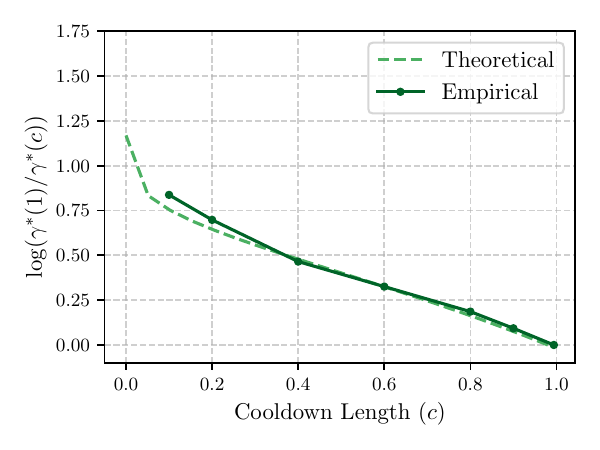}
        \caption{\textcite{schaipp2025}.}
        \label{fig:schaipp_wsd_gammastar_c}
    \end{subfigure}
    \caption{Log-ratio $\log(\gamma^*(1)) - \log(\gamma^*(c))$ of the optimal base
    learning-rates as a function of the cooldown length $c$.
    (a)~Asymptotically optimal $\gammastar_T$ from \cref{tab:gamma_star_combined}; the
    value of $T$ is inconsequential since the ratio is independent of $T$ by
    \cref{eq:log_gammastar}.
    (b)~Empirical $\gamma^*$ and $\gamma^*$ obtained from the convex bound of
    \textcite{schaipp2025} (their $\sqrt{\mathcal T_1/\mathcal T_2}$, with constants $G=D=1$).}
    \label{fig:gammastar_c}
\end{figure}

\paragraph{Optimizing the prefactor.} Using the asymptotically optimal learning-rate, we can deduce a cooldown length $c\in(0,1]$ which would optimize the WSD schedule. For $\alpha<2$, we must minimize $\frac{\mathcal{W}(\alpha, c)}{(2-c)^{1/\alpha}}$ while for $\alpha>2$, we need to maximize $c(2-c)$. The second case is directly optimized with $c^*=1$.

For the first case ($1 < \alpha < 2$), \cref{lemma:c_star} implies that the objective function $c\mapsto \frac{\mathcal{W}(\alpha, c)}{(2-c)^{1/\alpha}}$ is strictly decreasing on the interval $(0, 1].$ Therefore, the minimum is also reached at the upper boundary $c^* = 1$. This ultimately demonstrates that the pure linear decay schedule remains theoretically optimal within the WSD family across both asymptotic regimes for $\gammastar_T.$

\begin{corollary}[Pure Linear Decay Optimality]\label{th:c_star}
Under the optimal base learning-rate $\widetilde{\gamma}^*_T$, the asymptotic excess risk $\widetildeR_T$ of a WSD schedule is minimized for a relative cooldown length $c^*=1$, which recovers the pure linear decay schedule.
\end{corollary}

\section{Analysis with a Minibatch}

This section, detailed in \cref{app:batch_section}, explains the changes that occur when we change the batch size. At each step $t$, instead of drawing a single data sample $(\featvec _t, \varepsilon_t)$, we draw i.i.d. samples $(\featvec _{t,i}, \varepsilon_{t,i}) \sim \mathcal D$ for $i \in \{1,\dots,\bs\}$ and define $g_t$ as the average of the $\bs$ gradient estimators: $g_{t} := \frac{1}{\bs}\sum_{i=1}^\bs(\featvec _{t,i}\featvec _{t,i}^\top (\vect_t -\vect_*)- \varepsilon_{t,i} \featvec _{t,i})$.

Deriving the new asymptotic equivalents follows the same methodology as in the $\bs=1$ case, as the $O(i^{-2\alpha})$ terms consistently vanish for the WSD schedule across both spectral regimes ($\alpha<2$ and $\alpha>2$). Ultimately, the introduction of a mini-batch size $\bs$ manifests through two main changes: a rescaled effective noise variance of $\sigma^2/\bs$, and a modified trace term $\frac{\tr(\Lambda)}{\bs}$ replacing $\tr(\Lambda)$.

The stability assumption becomes \cref{as:gamma_batch} with $0 < \gamma < \frac{2 \bs}{\tr(\Lambda) + L(\bs-1)}$, which defines a maximal batch size. Indeed, we can compute a new optimal base learning-rate $\gammastarb = \bs^{p}\, \gammastar_T$ with $p+\vartheta=1$, but it must stay under the bound $\frac{2 \bs}{\tr(\Lambda) + L(\bs-1)}$, which tends to $2/L$ as $b$ grows. Hence $b$ must be bounded by some $\bmax(T)$. 

Thus, we can define a critical batch size regarding the budget $\bs T$ similarly to \textcite{zhang2025how}. Below this threshold, the optimal risk depends on $\bs$ and $T$ only through the sample budget $\bs T$ (\cref{th:eta_star_batch}): at a fixed budget, increasing the batch size reduces the number of steps at no statistical cost, provided the base learning-rate is re-tuned as $\gammastarb$. \cref{th:eta_star_batch} is stated for a fixed $b$; the critical batch size is a heuristic extrapolation, since near $\widetilde b_{\mathrm{crit}}$ the optimal base learning-rate is no longer small and the equivalents need not apply.

\section{Limitations}

The first limitation of this work is the setting. Least squares is a smooth convex problem, whereas deep learning objectives are highly non-convex, and we study plain SGD rather than the adaptive methods used in practice. Two assumptions restrict the setting further. Slock's model (\cref{as:slock}) is chosen to decouple the coordinates rather than to describe real data. The source and capacity conditions (\cref{as:alphabeta}) impose exact power laws on the spectrum and initial distance; problems departing from this form fall outside our results.

A further limitation lies in the practical applicability of our asymptotic results, which require very large horizons $T$. Indeed, convergence can be significantly delayed for small $T$ due to competing finite-time terms. As illustrated in \cref{fig:wsd_equivalent}, horizons exceeding $10^5$ steps may be required depending on the power-law decay exponent of the spectrum.

This paper only studies linear cooldowns. A polynomial cooldown with $\delta > \alpha - 1$ would restore the optimal exponent for $\alpha > 2$ \parencite{bordelon2026theory}, and \textcite{li2026optimal} identify $\delta = 2\alpha - 1$ as optimal; we have not analysed this family.

\section{Conclusion}

Under Slock's model, we obtained the exact finite-time risk for the constant, linear and WSD schedules, and derived asymptotic equivalents of it when the horizon tends to $\infty$. We found that WSD is better than constant by a polynomial factor of the horizon $T$. For $\alpha >2$, linear cooldown is not optimal across all schedules; an exact analysis of polynomial cooldowns $\eta_t = (1-t/T)^\delta$ is a natural next step.

From the risk, we computed the best base learning-rate $\gammastar$ and the best linear cooldown length $c^*$, which is 1 if $\gamma$ is tuned, as \textcite{schaipp2025} found numerically. The ratio $\gammastar_T(1)/\gammastar_T(c)$ is independent of the horizon $T$, which makes learning-rate transfer across cooldown lengths explicit. 

Finally, we adapted the theorems for a batch size greater than 1 to see that batching is free up to some threshold: a critical batch size can be derived, which is the batch size beyond which the stability assumption no longer holds.

\subsection*{Acknowledgements}
This work has received support from the French government, managed by the National Research Agency,
under the France 2030 program with the reference “PR[AI]RIE-PSAI” (ANR-23-IACL-0008).

\subsection*{AI Use Statement}
 In this work, we used generative AI tools to assist in the writing of proofs,
 namely to review the mathematical proofs and help detect errors in them, and to
 provide feedback on the experimental methodology. We have not used generative AI
 tools to develop the theoretical model, formulate mathematical claims, provide
 critical ingredients for the proofs, propose hypotheses, implement methods,
 generate the synthetic data, or interpret results; cleaning or reformatting
 datasets and qualitative data analysis are not applicable to this work.
 Additionally, we used generative AI tools to improve the English phrasing
 throughout the paper, to proofread the manuscript for typographical and
 notational inconsistencies, and to draft short passages of text. We also used
 generative AI tools to reorganize the plotting code of the supplementary
 material and to numerically cross-check the constants of the asymptotic
 equivalents. We have reviewed all AI-assisted work: all proofs were conceived
 and written by the authors, every correction suggested by these tools was
 checked by the authors, and every suggested revision was assessed by the
 authors before being adopted. We take responsibility for the final content of
 this work, including text, claims or artifacts produced with the aid of
 generative AI.

\printbibliography

\appendix
\crefalias{section}{appendix}
\crefalias{subsection}{subappendix}
\crefalias{subsubsection}{subappendix}

\tableofcontents


\section*{Appendix}

The code and the results needed to reproduce every figure and
table of this paper are available at \url{https://github.com/noecolomban/Asymptotic_risk}.
Each figure is produced by a dedicated script \texttt{experiments/figX\_*.py};
all experiments were run on a single CPU.

\section{Notations and Constants}
\label{app:notations}

\subsection{Notations of the appendix}
Throughout the appendix, we will use the following notations:
\begin{itemize}
    \item $\bar\gamma = \gamma - \frac12 \gamma^2 \tr(\Lambda)$,
    \item $\bar\nu_k = \gamma \eta_k - \frac12 \gamma^2 \eta_k^2 \tr(\Lambda)$,
    \item $\Pi_{T,i} = \prod_{k=0}^{T-1} \left[ 1 - 2 \bar \nu_k Li^{-\alpha}\right]$,
    \item $\a = 2\gamma L$ and $\b = \gamma^2 L\tr(\Lambda)$,
    \item $\a_i = i^{-\alpha}\a$ and $\b_i = i^{-\alpha}\b$.
\end{itemize}

The constants $\a$ and $\b$ will be redefined for a batch size greater than 1.

\subsection{Constants of the optimal base learning-rate}
\label{app:constants}
The constants appearing in \cref{tab:gamma_star_combined} are
\[
\begin{gathered}
C_{\mathrm{cst}} := \frac{\omega\,\Gamma(\omega)}{\alpha\,\zeta(\alpha)\,2^{\omega-1}},
\quad
C_{\mathrm{wsd<2}} := \frac{\omega\,\Gamma(\omega)}{2\,\mathcal{W}(\alpha,c)\,(2-c)^{\omega}},
\quad
C_{\mathrm{wsd>2}} := \frac{8\,\omega\,\Gamma(\omega)\,c^{1/2}}{\alpha\sqrt{\pi}\,\zeta(\frac{\alpha}{2})\,(2-c)^{\omega}}.
\end{gathered}
\]

\section{Analysis with a Minibatch}
\label{app:batch_section}

We generalize our study to SGD with a batch size $\bs \ge 1$. Averaging $\bs$ gradients reduces the variance of each step, allowing a larger base learning-rate and hence fewer steps. We show that this costs nothing statistically up to a critical batch size: at fixed sample budget $\bs T$, the optimal risk depends on $\bs$ and $T$ only through their product.

\paragraph{New Setup.}
At each step $t$, instead of drawing a single data sample $(\featvec _t, \varepsilon_t)$, we draw i.i.d. samples $(\featvec _{t,i}, \varepsilon_{t,i}) \sim \mathcal D$ for $i \in \{1,\dots,\bs\}$ and define $g_t$ as the average of the $\bs$ gradient estimators:\\ $g_{t} := \frac{1}{\bs}\sum_{i=1}^\bs(\featvec _{t,i}\featvec _{t,i}^\top (\vect_t -\vect_*)- \varepsilon_{t,i} \featvec _{t,i})$. We still obtain an unbiased gradient estimator: $\mathbb{E}[g_t\mid \mathcal F_t] = \frac{1}{\bs}\sum_{i=1}^\bs \nabla f(\vect_t) = \nabla f(\vect_t)$.

\begin{theorem}
    Under \cref{as:alphabeta}, for a batch size $\bs$,
    \begin{equation}
    \left\{
    \begin{aligned}
        B_{t\mid T} &= \frac{1}{2} \sum_{i=1}^\infty \Delta Li^{-\alpha-\beta}  \prod_{j=0}^{t-1} \left[ 1-2\gamma\eta_j Li^{-\alpha}+\gamma^2\eta_j^2 \frac{\tr(\Lambda)}{\bs}L i^{-\alpha} + \gamma^2\eta_j^2\frac{\bs-1}{\bs}L^2i^{-2\alpha}\right]
        \\
        V_{t\mid T} &= \frac{\sigma^2}{2 \bs} \sum_{i=1}^\infty  \sum_{k=0}^{t-1} L^2i^{-2\alpha}\gamma^2\eta_k^2  \prod_{j=k+1}^{t-1} \left[ 1-2\gamma\eta_j Li^{-\alpha}+\gamma^2\eta_j^2 \frac{\tr(\Lambda)}{\bs}L i^{-\alpha} +\gamma^2 \eta_j^2\frac{\bs-1}{\bs}L^2i^{-2\alpha}\right]
    \end{aligned}
    \right.
    \end{equation}
\end{theorem}

If $\bs=1$, we get the same formulas as in \cref{th:risk}.
To ensure the positiveness and monotonicity of the different terms in the batch setting, we must replace \cref{as:eta}. Specifically, we require $ 0\leq 1-2\gamma\eta_j Li^{-\alpha}+\gamma^2\eta_j^2 \frac{\tr(\Lambda)}{\bs}L i^{-\alpha} +\gamma^2 \eta_j^2\frac{\bs-1}{\bs}L^2i^{-2\alpha} <1$. By \cref{lemma:stability} with $S_i := \frac{\tr(\Lambda)}{\bs} + \frac{\bs-1}{\bs}Li^{-\alpha}$ ($S_i>Li^{-\alpha}$ since $\tr(\Lambda)>L$), this is equivalent to $\gamma < 2/S_i$ for every $i \geq 1$. As $(S_i)_{i\geq1}$ is non-increasing,
the binding constraint is $i=1$, and the assumption simplifies to:
\begin{assumption}[Stability and Monotonicity Condition] \label{as:gamma_batch}
We assume that the base learning-rate $\gamma$ satisfies:
\begin{equation}\label{eq:gamma_batch}
    0 < \gamma < \frac{2 \bs}{\tr(\Lambda) + L(\bs-1)}.
\end{equation}
\end{assumption}

Deriving the new asymptotic equivalents follows the same methodology as in the $\bs=1$ case, as the $O(i^{-2\alpha})$ terms consistently vanish for the WSD schedule across both spectral regimes ($\alpha<2$ and $\alpha>2$). Ultimately, the introduction of a mini-batch size $\bs$ manifests through two main changes: a rescaled effective noise variance of $\sigma^2/\bs$, and a modified trace term $\frac{\tr(\Lambda)}{\bs}$ replacing $\tr(\Lambda)$.

\begin{theorem}\label{th:batch}
Under \cref{as:alphabeta} and \cref{as:gamma_batch}, as $T \to \infty$, with a batch size $\bs$:

\subparagraph{Constant schedule:}
\begin{equation}
\begin{aligned}
    B_T &   \sim \widetilde B_T := \frac{L\Delta}{2\alpha} \frac{\Gamma\left(\omega\right)}{(2\gamma(1-\frac{\gamma \tr(\Lambda)}{2 \bs})L T)^{\omega}},
    \\
    V_T &\sim \widetilde V_T:= \frac{\gamma L\sigma^2}{4b} \sum_{i=1}^\infty \frac{1}{\left(1-\frac{\gamma \tr(\Lambda)}{2 \bs} \right)i^{\alpha}-\frac{L\gamma}{2}\frac{\bs-1}{\bs}}.
\end{aligned}
\end{equation}

\subparagraph{WSD schedule:}

\begin{equation}
\begin{aligned}
    B_T &\sim \widetilde B_T := \frac{L\Delta}{2\alpha}
    \frac{\Gamma\!\left(\omega\right)}
    {\left[\gamma L T\!\left(2 - c - \frac{\gamma\tr(\Lambda)}{\bs}\!\left(1 - \frac{2c}{3}\right)\right)
    \right]^{\omega}},
    \\[6pt]
    V_T &\sim \widetilde V_T:=
    \begin{cases}
    \displaystyle
    \begin{aligned}
    &\Bigg(
      \frac{(L\gamma)^{\frac{1}{\alpha}}\sigma^2}{2\alpha \bs \left( \frac{\gamma \tr(\Lambda)}{\bs} - 2 \right)}
      \Gamma\!\left(\frac{\alpha-1}{\alpha}\right)
      \!\left[\kappa_1^{\frac{1-\alpha}{\alpha}} - \kappa_2^{\frac{1-\alpha}{\alpha}}\right]
      \\
      &+ \frac{L^2\gamma^2\sigma^2}{2\alpha \bs}
      \Gamma\!\left(2-\frac{1}{\alpha}\right)
      (\gamma L)^{\frac{1-2\alpha}{\alpha}}
      \mathcal{I}^{(\bs)}_{\alpha,\gamma}\,c^{\frac{1-\alpha}{\alpha}}
    \Bigg)T^{\frac{1-\alpha}{\alpha}}
    \end{aligned}
    & \text{if }\alpha<2,\\[14pt]
    \displaystyle
    \frac{\sqrt{\pi}}{8}\,L^{1/2}\gamma^{1/2}\frac{\sigma^2}{\bs}\,
    \zeta\!\left(\frac{\alpha}{2}\right)(cT)^{-1/2},
    & \text{if }\alpha>2.
    \end{cases}
\end{aligned}
\end{equation}

where we define
\[
\begin{array}{lll}
    \mathcal{I}^{(\bs)}_{\alpha,\gamma} := \int_0^1 x^{\frac{2-2\alpha}{\alpha}} \left(1-\frac{\gamma\tr(\Lambda)}{3b}x\right)^{\frac{1-2\alpha}{\alpha}}dx, &
    \kappa_1 :=   2 - c - \frac{\gamma \tr(\Lambda)}{\bs} \left( 1 - \frac{2c}{3} \right), &
    \kappa_2 := c \left(1-\frac{\gamma\tr(\Lambda)}{3b}\right).
\end{array}
\]
\end{theorem}

The proofs are deferred to \cref{sec:batch_asymptotics}.

\subsection{Optimal Base learning-rate $
\gammastarb$}

The asymptotically optimal base learning-rate with a batch size of $\bs$, denoted $\gammastarb$, is derived similarly to the case with a batch size of $1$, but with the variance term $\sigma^2$ replaced by $\sigma^2/\bs$. Letting $\widetilde\gamma^*_T$ denote the optimal learning-rate obtained in \cref{sec:eta_star} ($\gammastarb[1]=\gammastar_T$), we have:
\begin{equation}\label{eq:gamma_b_star}
    \gammastarb = \bs^{p}\cdot \gammastar_T,
\end{equation}
where the exponent $p$ is given by:
\[
p=\begin{cases} 
    \frac{\alpha}{2\alpha+\beta-1} & \text{for the constant schedule,} \\
    \frac{\alpha}{\alpha+\beta} & \text{for WSD with } c>0 \text{ and } \alpha < 2, \\
    \frac{2\alpha}{3\alpha+2\beta-2} & \text{for WSD with } c>0 \text{ and } \alpha > 2.
\end{cases}
\]
Observe that $p+\vartheta = 1$ where $\gammastar_T = \Theta(T^{-\vartheta})$. Therefore $\gammastarb = \Theta(b \cdot (bT)^{-\vartheta})$: with a fixed budget, the optimal learning-rate is linear in $\bs$.

\begin{theorem}\label{th:eta_star_batch} 
Let $\gammastarb$ be the batch-scaled learning-rate as defined in \cref{eq:gamma_b_star}.
\begin{equation} 
 \inf_\gamma \risk_T(\gamma) \sim \risk_T(\gammastarb) ,
  \qquad T \to \infty.
\end{equation}
\end{theorem}

Therefore, the optimal risk depends on $\bs$ and $T$ only through the total sample budget $\bs T$: at fixed $\bs T$, batching incurs no statistical penalty, so the batch size can be increased to shorten training at no cost in risk.
This requires re-tuning the base learning-rate. The equivalence holds as long as the resulting $\gamma$ remains admissible under \cref{as:gamma_batch}, which caps it by a constant (which increases with $\bs$ but only up
to $2/L$) and hence defines a critical batch size beyond which the trade-off degrades. This threshold is heuristic: \cref{th:eta_star_batch} holds for a fixed $b$, and near the threshold $\widetilde\gamma^*_{T,b}$ is no longer small.

\subsubsection{Critical Batch Size}

We define the maximal batch size $\bmax(T)$ as the largest batch size $\bs \ge 1$ for which the optimal base learning-rate $\gammastarb$ satisfies the stability condition of \cref{as:gamma_batch}: 
$\bmax(T) = \sup\big\{\bs\geq 1;\quad \bs^p \gammastar \leq \frac{2 \bs}{\tr(\Lambda)+L(\bs-1)} \big\}.$ 
Since $\lim_{\bs\to\infty}\frac{2 \bs}{\tr(\Lambda)+L(\bs-1)} = \frac{2}{L}$, the stability threshold approaches a constant for large batch sizes. Therefore, when $\bmax(T) \gg 1$, it can be computed numerically or approximated using $\bmax(T)\sim \left( \frac{2}{L \gammastar_T} \right)^{1/p}.$

\paragraph{Empirical observations.} At fixed sample
budget, the risk depends on $\bs$ and $T$ only through $\bs T$, so one may trade steps for parallelism at no
statistical cost. This cannot hold for arbitrarily large $\bs$, since $\gammastarb$ grows with $\bs$ while
\cref{as:gamma_batch} caps the admissible learning-rate. We therefore locate the batch size at which
this trade-off breaks down, from two complementary angles: at a fixed sample budget $\bs T$, by tracking
the excess risk; and at a fixed target risk $r$, by tracking the number of steps needed to reach it.
In both settings, scaling is near-perfect up to a threshold close to $\bmax$. It is not exact:
$\gammastarb$ minimizes the asymptotic risk $\widetilde{\mathcal R}_T$, and \cref{th:eta_star_batch}
only gives equivalence as $T\to\infty$, so at finite $T$ (all the more as $\bs$ grows and $T=N/\bs$
shrinks) a small residual gap remains, visible as the slow drift above $1$ before the threshold.

\begin{figure}[H]
    \centering
    \begin{subfigure}[b]{0.48\textwidth}
        \centering
        \includegraphics[width=\linewidth]{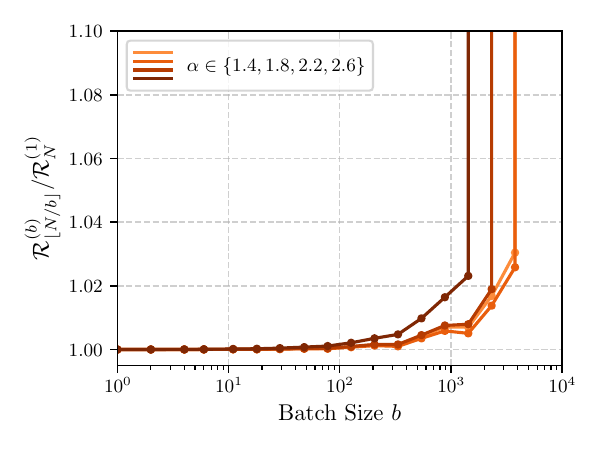}
        \caption{Fixed sample budget $N=\bs T$.}
        \label{fig:batch_risk}
    \end{subfigure}
    \hfill
    \begin{subfigure}[b]{0.48\textwidth}
        \centering
        \includegraphics[width=\linewidth]{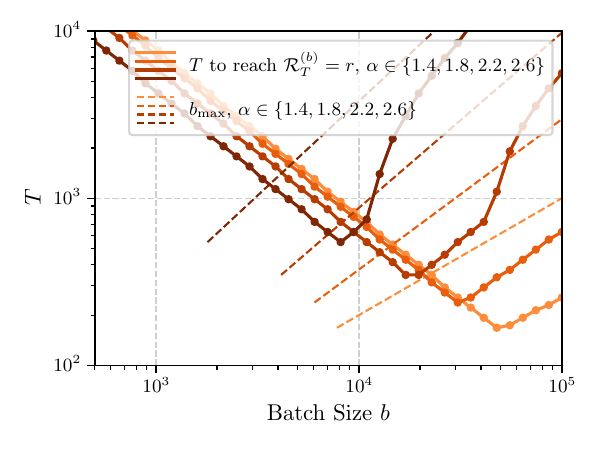}
        \caption{Fixed target risk $r$.}
        \label{fig:steps_fixed_risk}
    \end{subfigure}
    \caption{Effect of the batch size $\bs$ under a linear decay schedule with the asymptotically optimal base learning-rate $\gammastarb$, parameters $L=\Delta=\sigma=1$, $\beta=2$,
    $d=100$.\\
    \labelcref{fig:batch_risk}: Optimal excess risk as a function of $\bs$ at fixed total sample budget
    $N:=\bs T = 200000$, normalized by its value at $\bs=1$.\\
    \labelcref{fig:steps_fixed_risk}: Number of steps $T$ needed to reach a fixed risk
    $\mathcal{R}^{(\bs)}_T = r := 10^{-5}$, as a function of $\bs$. The dashed lines show the horizon
    $T$ at which $\bs=\bmax(T)$. At a fixed target risk $r$ (\cref{fig:steps_fixed_risk}), the required number of steps initially decreases inversely proportionally to $\bs$. This scaling stops for a batch size $\bs$ a few times the theoretical boundary $\bs=\bmax(T)$ (dashed lines), beyond which the trend reverses and $T$ grows with $\bs$: the learning-rate prescribed by the $\bs^p$ rule is then no longer admissible (\cref{eq:gamma_b_star}).}
    \label{fig:batch_crit}
\end{figure}

At a fixed sample budget $N:=\bs T$ (\cref{fig:batch_risk}), the excess risk remains flat for small batch sizes, and the risk only begins to degrade when the batch size approaches a theoretical stability limit. The relevant threshold is no longer $\bmax(T)$, which is a function of a fixed horizon, but a critical batch size $\bcrit(N)$, function of a budget $\bs T$, defined as the maximal $\bs$ such that $\bs \leq \bmax(N/\bs)$. We can compute it numerically \cref{tab:bcrit}. 
\begin{table}[H] 
    \centering
    \begin{tabular}{cc}
        \hline
        $\alpha$ & $\bcrit(N)$ \\
        \hline
        $1.4$ & $2653$ \\
        $1.8$ & $2245$ \\
        $2.2$& $1633$ \\
        $2.6$ & $1021$ \\
        \hline
    \end{tabular}
    \vspace{0.2cm}
    \caption{Critical batch size $\bcrit$ for a sample budget of $N = 2\times 10^5$ and various exponents $\alpha$, computed by testing all values until the limit $\bmax(N/b)$ is exceeded.}
    \label{tab:bcrit}
\end{table} 
We then observe that beyond $\bcrit$, the ratio does not stay close to $1$ but explodes, as \cref{as:gamma_batch} is no longer satisfied. We can provide an approximation for $\bcrit(N)$. Indeed, we can write $\gammastar_T = \gammastar_1 T^{-\vartheta}$, so $\bmax \sim \left( {2}/(L \gammastar_1) \right)^{1/p} T^{\vartheta/p}$ as $T\to \infty$. Furthermore, knowing $p+\vartheta =1$,
\[
\bs \leq \left( \frac{2}{L \gammastar_1} \right)^{1/p} \left(\frac{N}{\bs}\right)^{\vartheta/p} \iff \bs^{1+\vartheta/p} \leq \left( \frac{2}{L \gammastar_1} \right)^{1/p} {N}^{\vartheta/p} \iff \bs \leq\frac{2N^\vartheta}{L \gammastar_1}.
\]
Thus, provided that $N/\bs \gg 1$, $\bcrit(N)$ is approximately $\frac{2N^\vartheta}{L \gammastar_1}$. Note that, for the regime $1<\alpha<2$ and $\beta<\alpha+1$, we recover the same exponent as \textcite{zhang2025how}: they write $B^*\eqsim D^{1-a/\min\{b,2a+1\}}$ with $a=\alpha$ and $b=\beta+\alpha$, which corresponds in our notation to $N^{1-\alpha/\min\{\alpha+\beta,2\alpha+1\}} = N^{1-\frac{\alpha}{\alpha+\beta}}=N^{\frac{\beta}{\alpha+\beta}}$. This is equivalent to $N^\vartheta$.

\section{Optimal Learning-Rate}
\label{app:gammastar}
We keep the definition $\frrisk_T(\gamma) =  C_B \gamma^{-
\omega}T^{-\omega} + {C_V \gamma^{v}T^{-r}}$ of \cref{eq:frrisk}.

\begin{lemma} \label{lemma:eta_star_to_0}
For a constant or WSD schedule:
\begin{enumerate}
    \item Every sequence of minimizers of $\gamma\mapsto \eqrisk_T(\gamma)$ over the interval $(0, \gammax)$ tends to $0$ as $T\to \infty$,
    \item $\displaystyle \gammastar_T := \argmin_{\gamma \in (0, \gammax)} \frrisk_T(\gamma) = \Theta(T^{-\vartheta})$ with $\vartheta:= \frac{\omega - r}{v+\omega} \in (0,1)$,
    \item $\displaystyle \eqrisk_T(\gammastar_T) \;\sim\; \inf_{\gamma \in (0, \gammax)} \eqrisk_T(\gamma),
      \qquad T \to \infty$.
\end{enumerate}
\end{lemma} 
The proof is deferred to \cref{proof:eta_star_to_0}.

While finding this tractable minimizer is a necessary first step, transferring it to the exact risk is mathematically non-trivial: because the sequence $\widetilde\gamma^*_T$ shrinks to zero as $T \to \infty$, pointwise asymptotic equivalence is insufficient, and convergence must be established for a class of sequences $(\gamma_T)$. This class isolates the valid asymptotic regime by enforcing two competing conditions: the base learning-rate must vanish ($\gamma_T \to 0$) to eliminate expansion remainders, while the effective time horizon must simultaneously diverge ($\gamma T_T \to \infty$) to guarantee that discrete Riemann sums properly converge to their continuous integral limits.
Thus, if we write $\mathcal R_T(\gamma) = (1+\varepsilon_T(\gamma))\widetildeR_T(\gamma)$, using the fact that $\widetilde\gamma^*_T = \Theta(T^{-\vartheta})$ with $\vartheta \in (0,1)$, we show that $\lim_{T \to \infty}\varepsilon_T(\widetilde\gamma^*_T)=  0$.

\begin{theorem} \label{th:gammastar_appendix}
With $\widetilde\gamma^*_T$ defined as above, we obtain that the exact risk for the base learning-rate $\widetilde\gamma^*_T$ is asymptotically optimal:
\begin{equation}\label{eq:gammastar}
\inf_{\gamma \in (0, \gammax)} \risk_T(\gamma) \sim \risk_T(\gammastar_T),
  \qquad T \to \infty .
\end{equation}
\end{theorem}
The complete proof is deferred to \cref{sec:gammastar}. We use an infimum rather than a minimum since the existence of an exact minimizer of $\gamma \mapsto \risk_T(\gamma)$ is not established. The proof relies on two ingredients. First, \cref{lemma:uniform_epsilon} shows that the equivalence $\risk_T(\gamma_T) \sim \eqrisk_T(\gamma_T)$ holds uniformly if $\gamma_T T \to \infty$ and $\gamma_T \to 0$. Second, \cref{lemma:bounds} shows that any sequence $(\gamma_T)$ that is at least as good as $(\gammastar_T)$, i.e. $\risk_T(\gamma_T) \leq \risk_T(\gammastar_T)(1+o(1))$, automatically satisfies $\gamma_T \to 0$ and $\gamma_T T \to \infty$. No assumption on the location of the exact minimizers is thus needed.

\section{Comparison with Gaussian Features}
\label{app:gauss}
In the case of Gaussian features $\featvec \sim \mathcal N(0, H)$, we get from Isserlis' theorem: $\E[\featvec_t \featvec_t^\top \vect_t \vect_t^\top \featvec_t \featvec_t^\top] = 2 \Lambda M_t \Lambda + \tr(\Lambda M_t)\Lambda$. Since $\E[\featvec \featvec^\top] = H$ is still satisfied, we have:

\begin{lemma}[Gaussian Features Recurrence]\label{lemma:recurrence_gaussian} For $M_t = Q^\top\E[(\vect_t-\vect_*)(\vect_t-\vect_*)^\top]Q$, and $m_t = \left(M_{i,i}\right)_{i\geq 1} \in \R^d$ the diagonal of $M_t$,
\begin{equation}\label{eq:recurrence_gaussian}
m_{t+1} = \left[(I - \gamma\eta_t\Lambda)^2 + \gamma^2\eta_t^2 \Lambda^2 + \gamma^2\eta_t^2\lambda\lambda^\top \right] m_t + \gamma^2\eta_t^2 \sigma^2 \lambda, \qquad t \in\{0,...,T-1\}.
\end{equation}
\end{lemma}

Therefore we can still compute the exact risk $\frac{1}{2}\lambda^\top m_t$ (although an asymptotic equivalent would be difficult due to the non diagonal term $\lambda\lambda^\top$), and the bias-variance parts $\gaussT{B}$ and $\gaussT{V}$.

To compare this risk with the one found in \cref{th:risk}, we plot $B_T/\gaussT{B}$ and $V_T/\gaussT{V}$ in \cref{fig:slock_vs_gauss}. \cref{fig:bias_ratio} shows that the agreement of the two models degrades as $T\to \infty$. However, we will see \cref{lemma:eta_star_to_0} that we need $\gamma\tr(\Lambda)\to 0$ as $T\to \infty$ to have an optimal base learning-rate. We can therefore plot again such ratios with a vanishing $\gamma_T$. Let us choose $\gammastar_T = \Theta(T^{-\vartheta})$ where $\vartheta \in (0,1)$, defined \cref{tab:gamma_star_combined}.

\begin{figure}[H]
    \centering
    \begin{subfigure}[b]{0.48\textwidth}
        \centering
        \includegraphics[width=\textwidth]{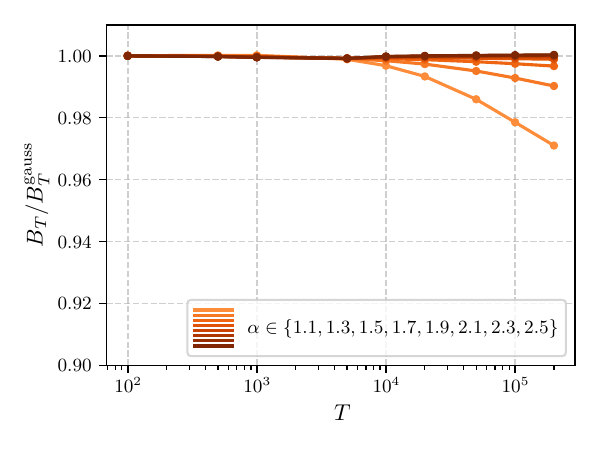}
        \caption{Bias ratio, $\gamma = 10^{-2}$}
        \label{fig:bias_ratio}
    \end{subfigure}\hfill
    \begin{subfigure}[b]{0.48\textwidth}
        \centering
        \includegraphics[width=\textwidth]{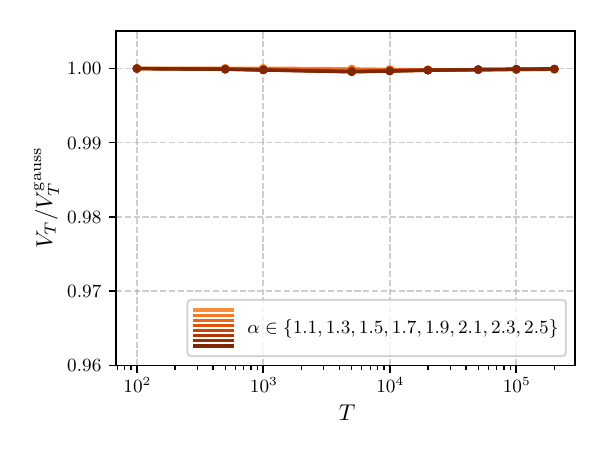}
        \caption{Variance ratio, $\gamma = 10^{-2}$}
        \label{fig:variance_ratio}
    \end{subfigure}
   \caption{Ratio of the bias (\cref{fig:bias_ratio}) and variance (\cref{fig:variance_ratio}) evaluated under Slock's model relative to the standard Normal distribution, with a linear decay schedule. 
   Computed with parameters $d=100, \gamma=10^{-2}, \beta=2, L=\Delta=1, \sigma=1$. As the accumulation of coupled terms increases with $T$, the bias ratio drifts from 1.
   }
    \label{fig:slock_vs_gauss}
\end{figure}

\begin{figure}[H]
    \captionsetup[subfigure]{skip=2pt, belowskip=-4pt}

    \begin{subfigure}[b]{0.48\textwidth}
        \centering
        \includegraphics[width=\textwidth]{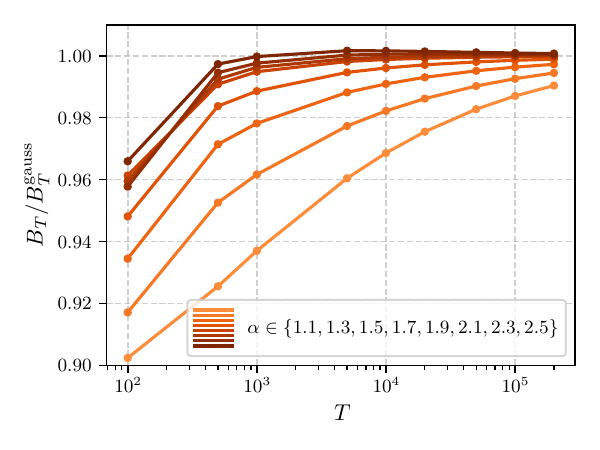}
        \caption{Bias ratio, $\gamma = \gammastar_T$} 
        \label{fig:bias_ratio_gammastar} 
    \end{subfigure}\hfill
    \begin{subfigure}[b]{0.48\textwidth}
        \centering
        \includegraphics[width=\textwidth]{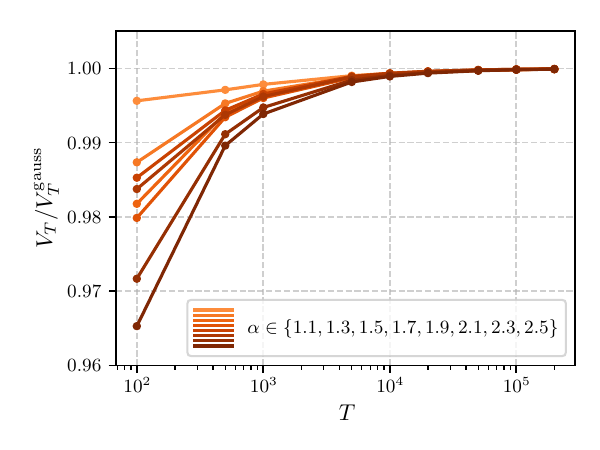}
        \caption{Variance ratio, $\gamma = \gammastar_T$}
        \label{fig:variance_ratio_gammastar}
    \end{subfigure}
   \caption{
   Ratio of the bias (\cref{fig:bias_ratio_gammastar}) and variance (\cref{fig:variance_ratio_gammastar}) evaluated under Slock's model relative to the standard Normal distribution, with a linear decay schedule. 
   Computed with parameters $d=100, \gamma=\gammastar_T, \beta=2, L=\Delta=1, \sigma=1$.}
    \label{fig:slock_vs_gauss_tuned}
\end{figure}

In \cref{fig:slock_vs_gauss_tuned}, the ratios tend to 1, revealing that by using a vanishing base learning-rate, Slock's model exhibits asymptotic risk behavior similar to that of the standard Gaussian model. Indeed, the coupled terms $(\gammastar_T)^2\eta_t^2\lambda\lambda^\top$ become negligible before $I-2\gammastar_T\eta_t\Lambda$ as $\gammastar_T \to 0$, similarly to Slock's model. This confirms the pertinence of our choice of model. Note that as $\alpha$ approaches 1, the convergence slows down significantly. This is consistent with the slow convergence of the series $\sum_i i^{-\alpha}$ near this limit, which implies that the eigenvalue spectrum exhibits a heavier tail, requiring more steps $T$ to reach the asymptotic regime.

\section{Supplementary Figures}
\label{app:figures}

This appendix gathers the numerical experiments referred to in \cref{sec:equivalents,sec:eta_star}. \cref{fig:asymptotic_constant,fig:wsd_equivalent} compare the asymptotic equivalents of the bias and the variance with their exact values, for the constant schedule and the WSD schedule ($c = 0.2$) respectively. \cref{fig:comparison_cst_wsd} compares the risk trajectories of both schedules when the base learning-rate tuned for the constant schedule is reused for WSD. All experiments were run on a single CPU.

\begin{figure}[tbp]
    \centering
    \begin{subfigure}[b]{0.48\textwidth}
        \centering
        \includegraphics[width=\textwidth, clip]{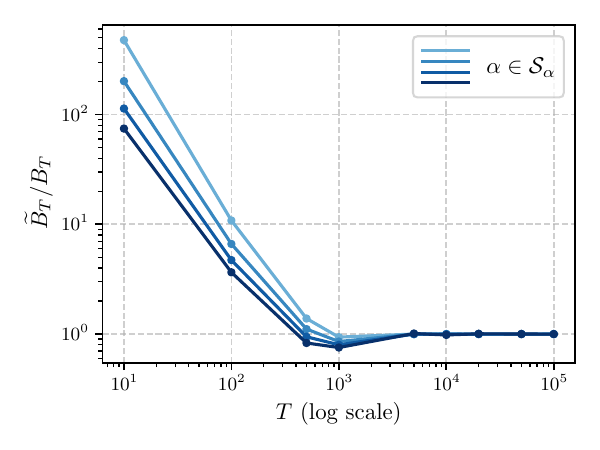}
        \caption{Bias ratio}
        \label{fig:sgd_constant}
    \end{subfigure}\hfill
    \begin{subfigure}[b]{0.48\textwidth}
        \centering
        \includegraphics[width=\textwidth, clip]{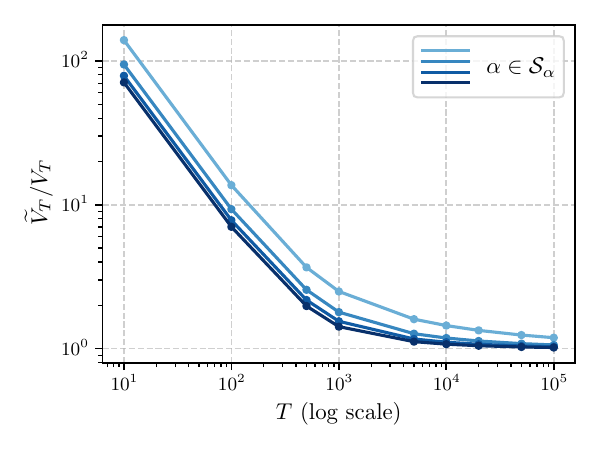}
        \caption{Variance ratio}
        \label{fig:sgd_linear}
    \end{subfigure}
    \caption{For the constant schedule, ratio of the asymptotic equivalents to the exact values for the bias (a) and variance (b) components of the risk, for $\gamma = 10^{-3},L=\Delta=1, \beta = 2, \sigma=0.1$, $d=1000$ and $\mathcal S_\alpha = \{1.4, 1.8, 2.2, 2.6\}$. As the total number of steps $T$ increases, the ratios systematically converge to 1 across different spectral exponents $\alpha$, empirically validating the theoretical approximations.}
    \label{fig:asymptotic_constant}
\end{figure}

\begin{figure}[tbp]
    \centering
    \begin{subfigure}[b]{0.48\textwidth}
        \centering
        \includegraphics[width=\textwidth, clip]{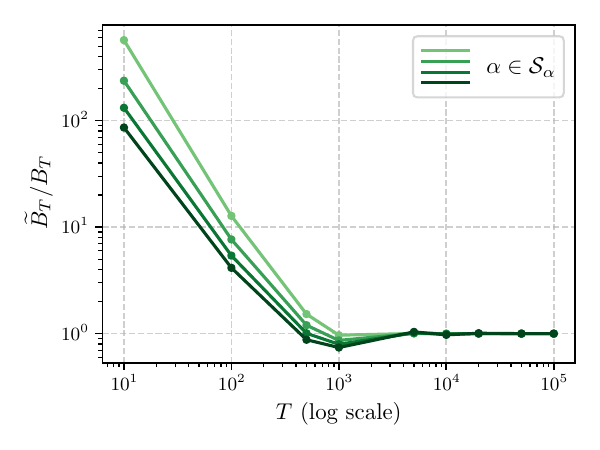}
        \caption{Bias ratio.}
        \label{fig:sgd_constant}
    \end{subfigure}\hfill
    \begin{subfigure}[b]{0.48\textwidth}
        \centering
        \includegraphics[width=\textwidth,clip]{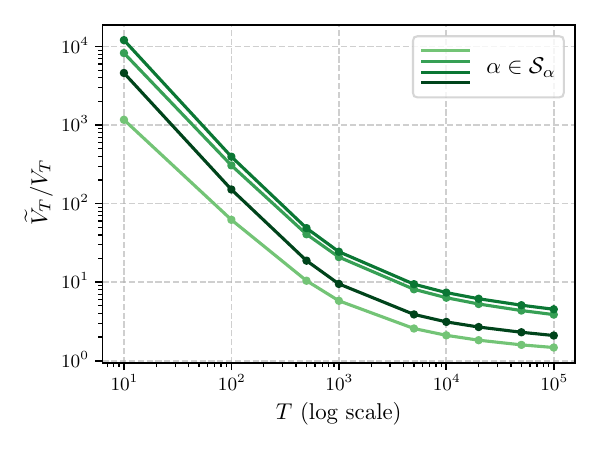}
        \caption{Variance ratio.}
        \label{fig:sgd_wsd}
    \end{subfigure}

    \caption{For the WSD schedule ($c=0.2$), ratio of the asymptotic equivalents to the exact values for the bias (a) and variance (b) components of the risk, for $\gamma = 10^{-3},L=\Delta=1, \beta = 2, \sigma=0.1$, $d=1000$ and $\mathcal S_\alpha = \{1.4, 1.8, 2.2, 2.6\}$.}
    \label{fig:wsd_equivalent}
\end{figure}

While \cref{tab:optimal_scaling} demonstrates the asymptotic superiority of WSD over a constant schedule when both are independently optimized, \cref{fig:comparison_cst_wsd} explores a learning-rate transfer setting. Specifically, we evaluate the performance of WSD when reusing the base learning-rate $\gamma = \gammastar_T$ derived solely for the constant schedule. To this end, we fix the horizon at $T=10^5$ and track the intermediate risk trajectory $\risk_{t\mid T}$ throughout the training process, rather than focusing solely on the final iterate risk $\risk_T = \risk_{T\mid T}$.

\begin{figure}[h]
    \centering
    \begin{subfigure}[b]{0.48\textwidth}
        \centering
        \includegraphics[width=\textwidth,  clip]{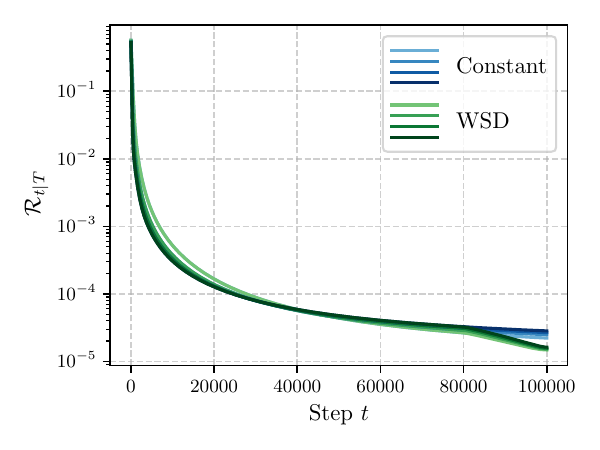}
        \caption{Constant and WSD}
        \label{fig:comparison_cst_wsd_a}
    \end{subfigure}\hfill
    \begin{subfigure}[b]{0.48\textwidth}
        \centering
        \includegraphics[width=\textwidth,  clip]{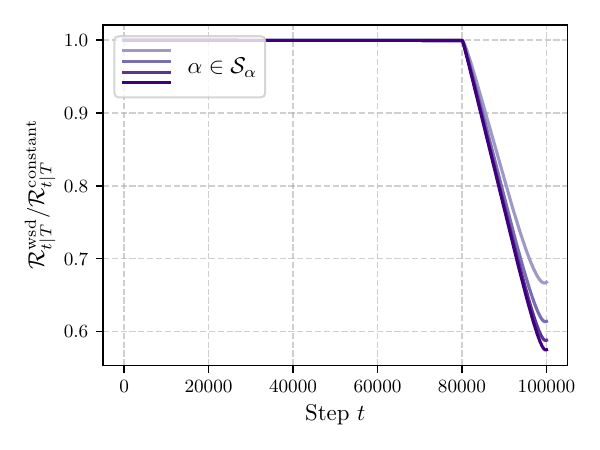}
        \caption{Ratio of the two risks}
        \label{fig:comparison_cst_wsd_b}
    \end{subfigure}
    \caption{Comparison of risk trajectories for constant and WSD schedules ($c=0.2$) over a total horizon $T = 10^5$. Hyperparameters are fixed to $\alpha \in \mathcal S_\alpha =\{1.4, 1.8, 2.2, 2.6\}$, $\beta=2$, $\sigma=0.1, d=100$ and $L=\Delta=1$. The parameter $\gamma = \gammastar_T$ is optimized for the constant schedule, and this exact same value is applied to both methods. As expected, the risks remain identical up to $t=0.8T = 80,000$, after which the WSD cooldown phase effectively reduces the final risk.}
    \label{fig:comparison_cst_wsd}
\end{figure}

\section{Proofs}
\label{app:proofs}

\subsection{Preliminary Proofs}\label{sec:first_proofs}

\label{app:first_proofs}

\begin{lemma}Consider the function $f$ defined \cref{sec:setup}. Then $f$ is differentiable.
\end{lemma}
\begin{proof}
Define the deterministic function $f_{u,v}$, for every $(u,v) \in \mathbb{R}^d \times \mathbb{R}$:
\[\begin{aligned}
    f_{u, v} \colon \mathbb{R}^d &\to \mathbb{R}, \\
    \vect &\mapsto \frac{1}{2}(u^\top \vect - v)^2.
\end{aligned}\]
Therefore $f(\vect) = \mathbb{E}[f_{\featvec,\labelscl}(\vect)]$. To apply the dominated convergence theorem for differentiation, we verify the following conditions:

\begin{itemize}
    \item For all $\vect \in \mathbb{R}^d$, the expectation $\mathbb{E}[f_{\featvec,\labelscl}(\vect)]$ is well-defined since $\featvec$ and $\labelscl$ admit finite second-order moments.
    
    \item For all $(u,v) \in \mathbb{R}^d \times \mathbb{R}$, $f_{u, v}$ is continuously differentiable and its gradient is $\nabla f_{u,v}(\vect) = u(u^\top \vect-v)$.
    
    \item Let $R>0$. For all $\vect \in \mathbb{R}^d$ such that $\|\vect\| \le R$, the gradient norm is bounded by the Cauchy-Schwarz inequality:
    \[\|\nabla f_{\featvec,\labelscl}(\vect)\| = \|\featvec(\featvec^\top \vect - \labelscl)\| \le \|\featvec\|(\|\featvec\| \|\vect\| + |\labelscl|) \le R\|\featvec\|^2 + \|\featvec\||\labelscl|.\]
    The random variable $R\|\featvec\|^2 + \|\featvec\||\labelscl|$ is integrable because $\mathbb{E}[\|\featvec\|^2]$ is finite (since $H$ exists) and $\mathbb{E}[\|\featvec\||\labelscl|]$ is finite (by the Cauchy-Schwarz inequality, as both $\featvec$ and $\labelscl$ have finite second moments).
\end{itemize}

Thus, the local domination condition is met, which allows us to interchange the gradient and the expectation:
\[\nabla f(\vect) = \mathbb{E}[\nabla f_{\featvec,\labelscl}(\vect)] = \mathbb{E}[\featvec(\featvec^\top \vect - \labelscl)].\]

Substituting $\labelscl = \featvec^\top \vect_* + \varepsilon$, we get:
\[\nabla f(\vect) = \mathbb{E}[\featvec(\featvec^\top \vect - (\featvec^\top \vect_* + \varepsilon))] = \mathbb{E}[\featvec \featvec^\top(\vect - \vect_*)] - \mathbb{E}[\featvec\varepsilon].\]

\end{proof}

\begin{lemma} \label{lemma:stability}
Let $L,\gamma >0$ and $\alpha>1$. Assume $S_i>Li^{-\alpha}$ for $i\geq1$. For $\eta \in \R$ and $i\geq 1$, let $P(\eta,i) := 1 - 2\gamma \eta L i^{-\alpha}+ \gamma^2 \eta^2L S_i i ^{-\alpha}$. The following are equivalent:
\begin{enumerate}
    \item[$(i)$] $\forall i \geq 1, \forall \eta \in (0,1], \quad 0\leq P(\eta,i) <1$,
    \item[$(ii)$] $ \forall i\geq1, \quad \gamma < \frac{2}{S_i}.$
\end{enumerate}

\begin{proof}
We fix $i\geq 1$.
Remark that the expression of $P$ is a second order polynomial in $\eta$.
It has a minimum reached at $\eta^0:=\frac{1}{\gamma S_i}$ and whose value is $P(\eta^0,i) = 1 - \frac{2 Li^{-\alpha}}{S_i} + \frac{Li^{-\alpha}}{S_i} = 1- \frac{Li^{-\alpha}}{S_i}$. Therefore since $S_i>Li^{-\alpha}$ and $i^{-\alpha}\leq 1$, $P(\eta,i)>0$ is always satisfied for $\eta\in \R$ and $i\geq 1$. We can therefore write $(i)$ as $\forall i \geq 1, \forall \eta \in (0,1], ~ P(\eta,i) < 1$.

Furthermore, $1-P(\eta,i) = 2\gamma \eta L i^{-\alpha}  - \gamma^2 \eta^2L S_i i ^{-\alpha} = \eta(2\gamma L i^{-\alpha}  - \gamma^2 \eta L S_i i ^{-\alpha} )$ so $1-P(\eta,i) = 0  \iff \eta \in \{0, \frac{2}{\gamma S_i}\}$. Since $1 - P(\eta, i)$ has a maximum in $\eta$,
\begin{equation}\label{eq:interval_eta}
\forall \eta \in \R, \quad 1-P(\eta,i) > 0 \iff \eta \in \left(0, \frac{2}{\gamma S_i}\right).
\end{equation}

\begin{itemize}
\item Suppose $(i)$. Choosing $\eta=1$, $1-P(1,i)>0$ so by \cref{eq:interval_eta}, $\forall i\geq1, 1<\frac{2}{\gamma S_i}$, which is $(ii)$.
\item Suppose $(ii)$. Then for every $i\geq 1$, $\frac{2}{\gamma S_i}>1$. Hence, for all $\eta \in (0,1]$, \cref{eq:interval_eta} gives $1 - P(\eta,i)>0$. $(i)$ is satisfied.
\end{itemize}

\end{proof}

\end{lemma}

\subsection{Computation of the Exact Risk (\cref{th:risk})}
\label{app:base}

\begin{proof}
Let $M_t = Q^\top\E[(\vect_t-\vect_*)(\vect_t-\vect_*)^\top]Q$, and $m_t = \left(M_{i,i}\right)_{i\geq 1} \in \R^d$ the diagonal of $M_t$.
First, the risk can be written with $M_t$:
\begin{align*}
    \risk_{t \mid T} &= \frac{1}{2}\E[(\vect_t-\vect_*)^\top H(\vect_t-\vect_*)],
    \\&=
    \frac{1}{2}\E[\tr\left( (\vect_t-\vect_*)(\vect_t-\vect_*)^\top H \right)],
    \\&=
    \frac{1}{2}\tr(M_t\Lambda)
    = 
    \frac{1}{2}m_t ^\top \lambda.
\end{align*}

Thus $\mathcal R_{0 \mid T} = \frac{1}{2}\lambda^\top m_0$. By unwinding the recurrence for $m_t$ from \cref{lemma:recurrence}, we directly obtain:
    \[
    \risk_{t \mid T} = \frac{1}{2}\lambda^\top \left( \prod_{k=0}^{t-1}[I - 2\gamma\eta_k\Lambda + \gamma^2\eta_k^2\tr(\Lambda)\Lambda]m_0 + \sum_{k=0}^{t-1}\gamma^2\eta_k^2\sigma^2 \prod_{j=k+1}^{t-1}[I - 2\gamma\eta_j\Lambda + \gamma^2\eta_j^2\tr(\Lambda)\Lambda]\lambda  \right).
    \]

\end{proof}

\begin{lemma}\label{lemma:recurrence} For $M_t = Q^\top\E[(\vect_t-\vect_*)(\vect_t-\vect_*)^\top]Q$, and $m_t = \left(M_{i,i}\right)_{i\geq 1} \in \R^d$ the diagonal of $M_t$,
\begin{equation}\label{eq:recurrence}
m_{t+1} = \left[I + (-2\gamma\eta_t   + \gamma^2\eta_t^2 \tr (\Lambda)) \Lambda  \right]m_t + \gamma^2\eta_t^2\sigma^2\lambda, \qquad t \in\{0,...,T-1\}.
\end{equation}
\end{lemma}

\begin{proof}
    Without loss of generality, suppose $H$ is diagonal, i.e. $H=\Lambda$, and $\vect_* = 0$.
    Let us compute a recurrence relation with $M_t$.
    \begin{equation}\label{eq:Mt+1}
    \begin{aligned}
    M_{t+1} &= \E[((I-\gamma \eta_t \featvec_t \featvec_t^\top)\vect_t + \gamma \eta_t \varepsilon_t \featvec_t)((I-\gamma \eta_t \featvec_t \featvec_t^\top)\vect_t + \gamma \eta_t \varepsilon_t \featvec_t)^\top], \\
    &=  \E[(I-\gamma \eta_t \featvec_t \featvec_t^\top)\vect_t \vect_t^\top (I-\gamma \eta_t \featvec_t \featvec_t^\top)]+\gamma^2\eta_t^2\sigma^2 \Lambda,\\
    &=  \E[ (\vect_t-\gamma \eta_t \featvec_t \featvec_t^\top \vect_t)(\vect_t-\gamma \eta_t \featvec_t \featvec_t^\top \vect_t)^\top]  + \gamma^2\eta_t^2\sigma^2\Lambda, \qquad (\text{$\varepsilon_t$ and $\featvec_t$ are independent and $\E[\varepsilon_t] =0$}),\\
    &= M_t - \gamma \eta_t \Lambda M_t - \gamma \eta_t M_t \Lambda + \gamma^2\eta_t^2  \E[\featvec_t \featvec_t^\top \vect_t \vect_t^\top \featvec_t \featvec_t^\top] + \gamma^2\eta_t^2\sigma^2\Lambda.
    \end{aligned}
    \end{equation}
    
    Moreover, the term $\E[\featvec_t \featvec_t^\top \vect_t \vect_t^\top \featvec_t \featvec_t^\top]$ can be computed using the distribution of $\featvec_t$.
    \begin{align*}
    \E[\featvec_t \featvec_t^\top \vect_t \vect_t^\top \featvec_t \featvec_t^\top| \vect_t] 
    &=
    \sum_{i\geq1} \frac{\lambda_i}{\tr (\Lambda)} \times (\tr (\Lambda))^2  e_i \underbrace{e_i^\top  \vect_t \vect_t^\top  e_i}_{=\vect_{t,i}^2} e_i^\top,
    \\&=
    \tr (\Lambda)\sum_{i\geq1} \lambda_i \vect_{t,i}^2  e_i  e_i^\top
    =
    \tr (\Lambda)
    \begin{pmatrix}
    \lambda_1 \vect_{t,1}^2 & 0                  & \cdots \\
    0                  & \lambda_2 \vect_{t,2}^2 & \cdots \\
    \vdots             & \vdots             & \ddots
    \end{pmatrix},
    \\&=
    \tr (\Lambda) \Lambda \diag(\vect_t)^2.
    \end{align*}
    
    Knowing that $\E[\featvec_t \featvec_t^\top \vect_t \vect_t^\top \featvec_t \featvec_t^\top] = \E[\E[\featvec_t \featvec_t^\top \vect_t \vect_t^\top \featvec_t \featvec_t^\top | \vect_t]]$,
    \begin{align*}
    \E[\featvec_t \featvec_t^\top \vect_t \vect_t^\top \featvec_t \featvec_t^\top] &= 
    \tr (\Lambda) \Lambda \E[\diag(\vect_t)^2] ,
    \\&=
    \tr (\Lambda) \Lambda \E\begin{pmatrix}
    \ \vect_{t,1}^2 & 0                  & \cdots \\
    0                  & \vect_{t,2}^2 & \cdots \\
    \vdots             & \vdots             & \ddots
    \end{pmatrix}
    =
    \tr (\Lambda) \Lambda \begin{pmatrix}
    \ M_{t,11} & 0                  & \cdots \\
    0                  & M_{t,22} & \cdots \\
    \vdots             & \vdots             & \ddots
    \end{pmatrix},
    \\&= \tr(\Lambda)\Lambda \diag(m_t).
    \end{align*}
    Thus, substituting the result into \cref{eq:Mt+1},
    \begin{align*}
    M_{t+1} &=
    M_t - \gamma\eta_t \Lambda M_t - \gamma\eta_t M_t \Lambda + \gamma^2\eta_t^2 \tr (\Lambda) \Lambda \diag(m_t) + \gamma^2\eta_t^2\sigma^2\Lambda.
    \end{align*}

    Hence,
    $\begin{aligned}[t]
    m_{t+1} &= m_t - 2 \gamma \eta_t \Lambda m_t +  \gamma^2\eta_t^2  \tr (\Lambda) \Lambda m_t + \gamma^2\eta_t^2\sigma^2\lambda,
    \\&=
    \left[I + (-2\gamma \eta_t   + \gamma^2\eta_t^2 \tr (\Lambda)) \Lambda  \right]m_t + \gamma^2\eta_t^2\sigma^2\lambda.
    \end{aligned}$
\end{proof}

\begin{corollary}\label{th:risk_alpha_beta}
    Under \cref{as:alphabeta}, we have
    \begin{equation}\left\{
    \begin{aligned}
        B_{t \mid T} &= \frac{1}{2} \sum_{i=1}^d \Delta Li^{-\alpha-\beta}  \prod_{k=0}^{t-1} \left[ 1-2\gamma\eta_k Li^{-\alpha}+\gamma^2\eta_k^2 \tr(\Lambda)L i^{-\alpha} \right],
        \\
        V_{t \mid T} &= \frac{1}{2} \sum_{i=1}^d  \sum_{k=0}^{t-1} L^2i^{-2\alpha}\gamma^2\eta_k^2 \sigma^2 \prod_{j=k+1}^{t-1} \left[ 1-2\gamma\eta_j Li^{-\alpha}+\gamma^2\eta_j^2 \tr(\Lambda)L i^{-\alpha} \right].
    \end{aligned}\right.
    \end{equation}
\end{corollary}

\subsection{Asymptotic Equivalents of the Risk}
\label{app:asymptotic}

We will use the Tauberian theorem:
\begin{theorem}[\textcite{korevaar2004tauberian}, Theorem 7.4, Part I. See also \textcite{bach2025ztransform}]\label{th:tauberian}
         Let $(u_k)$ be a real-valued sequence such that the series $U(z) = \sum_{k=0}^\infty u_kz^k$ is defined on the interval $[0,1)$. Let $\mu \in \mathbb N$ and $p>0$. Then, if there exists a constant $K\in \R$ such that $\forall k\geq 1, k^{2+\mu-p}(u_k - u_{k-1})\leq K$, and $\lim_{z\to 1^-} (1-z)^p U^{(\mu)}(z)$ exists, then $\lim_{k \to +\infty} k^{1+\mu-p} u_k$ exists and
         \begin{equation}\label{th:korevaar}
            \lim_{z \to 1^-} (1 - z)^p U^{(\mu)}(z) = \Gamma(p) \cdot \lim_{k \to +\infty} k^{1+\mu-p} u_k.
         \end{equation}
        \end{theorem}

\subsubsection{Constant Schedule (\cref{th:constant})}
\label{app:constant}

\begin{proof}

        Let $\bar \gamma  = \gamma  - \frac{\gamma ^2 \tr(\Lambda)}{2}$. By \cref{as:eta}, $0<\bar\gamma  < \frac{1}{2L}$.

        \paragraph{Bias.}
        The exact bias for a horizon $T$ is:
        \[
        B_T = \frac{1}{2}\sum_i^{\infty} \frac{L\Delta}{i^{\alpha+\beta}} [1 - 2\bar\gamma \frac{L}{i^{\alpha}}]^T.
        \]
        
        The $z$-transform is
        \[
        \mathcal B(z) = \sum_{T=1}^\infty \frac{1}{2}\sum_i^{\infty} \frac{L\Delta}{i^{\alpha+\beta}} [1 - 2\bar\gamma \frac{L}{i^{\alpha}}]^T z^T
        =
        \frac{L\Delta}{2}\sum_{i=1}^\infty  \frac{1}{i^{\alpha+\beta}} \frac{1}{1-z[1 - 2\bar\gamma \frac{L}{i^{\alpha}}]}.
       \]
        
        As $z \to 1$, $\frac{1}{i^{\alpha+\beta}} \frac{1}{1-z[1 - 2\bar\gamma \frac{L}{i^{\alpha}}]} \sim \frac{1}{i^{\alpha+\beta}} \frac{1}{ 2\bar\gamma \frac{L}{i^{\alpha}}} \sim \frac{1}{2L\bar\gamma  i^\beta}$. If $\beta >1, \sum_i \frac{1}{2L\bar\gamma  i^\beta}$ converges. We force the $z$-transform to diverge in differentiating $\mathcal B$.
        \begin{align*}
        \mathcal B^{(\mu)}(z) &=
        \frac{L\Delta}{2} \sum_{i=1}^\infty i^{-(\alpha+\beta)} \frac{\mu! (1 - 2\bar{\gamma } L i^{-\alpha})^\mu}{\left[ 1 - z(1 - 2\bar{\gamma } L i^{-\alpha}) \right]^{\mu+1}}.
        \end{align*}
        
        So with the right $\mu \in \mathbb N_{\geq0}$, we can have $\mu\alpha - \beta >-1$ (e.g. $\mu =\lfloor \frac{\beta-1}{\alpha}\rfloor+1$) so that $\mathcal B^{(\mu)}(z)$ diverges when $z\to 1$, to be able to apply Tauberian theorems. With $s = 1-z >0$,
        \begin{align*}
        i^{-(\alpha+\beta)} \frac{\mu! (1 - 2\bar{\gamma } L i^{-\alpha})^\mu}{\left[ 1 - z(1 - 2\bar{\gamma } L i^{-\alpha}) \right]^{\mu+1}} &\sim 
        \frac{\mu!}{[i^{\alpha}s + (1-s)2\bar \gamma  L]^{\mu+1}} i^{\mu\alpha-\beta}.
        \end{align*}
        
        By substituting $i = x_i s^{-1/\alpha}$ and forcing the appearance of the step size $\Delta x = s^{1/\alpha}$, we can reformulate the sum to reveal a Riemann sum structure:
        \begin{align*} 
        \mathcal B^{(\mu)}(1-s) &\sim
        \frac{L\Delta}{2} \sum_{i=1}^\infty \frac{\mu!}{[ (x_i s^{-1/\alpha})^\alpha s + (1-s)2\bar \gamma  L]^{\mu+1}} (x_i s^{-1/\alpha})^{\mu\alpha-\beta},
        \\&=
        \frac{L\Delta}{2} \sum_{i=1}^\infty \frac{\mu!}{[x_i^\alpha + (1-s)2\bar \gamma  L]^{\mu+1}} x_i^{\mu\alpha-\beta} s^{-\mu + \frac{\beta}{\alpha}},
        \\&=
        \frac{L\Delta}{2} s^{\frac{\beta-1}{\alpha} - \mu} \sum_{i=1}^\infty \frac{\mu!}{[x_i^\alpha + (1-s)2\bar \gamma  L]^{\mu+1}} x_i^{\mu\alpha-\beta} \underbrace{s^{1/\alpha}}_{=\Delta x}.
        \end{align*}
        
        As $s \to 0$, we have $(1-s)2\bar\gamma  L \to 2\bar\gamma  L$. The discrete sum is exactly a Riemann sum with a vanishing step size $\Delta x \to 0$. It converges to the continuous integral over $[0, \infty)$:
        \begin{align*}
        \mathcal B^{(\mu)}(1-s) &\sim
        \frac{L\Delta}{2} s^{\frac{\beta-1}{\alpha} - \mu} \int_0^\infty \frac{\mu!}{[x^\alpha + 2\bar \gamma  L]^{\mu+1}} x^{\mu\alpha-\beta} dx.
        \end{align*}
        
        We directly apply the substitution $v = \frac{x^\alpha}{2\bar\gamma  L}$, which gives $dx = \frac{1}{\alpha} (2\bar\gamma  L)^{1/\alpha} v^{\frac{1-\alpha}{\alpha}} dv$:
        \begin{align*}
        \mathcal B^{(\mu)}(1-s) &\sim
        \frac{L\Delta}{2\alpha} \mu! \, s^{\frac{\beta-1}{\alpha}-\mu} \int_0^\infty \frac{(2\bar\gamma  L v)^{\mu - \frac{\beta}{\alpha}}}{[2\bar\gamma  L(v + 1)]^{\mu+1}} (2\bar\gamma  L)^{1/\alpha} v^{\frac{1-\alpha}{\alpha}} dv,
        \\&=
        \frac{L\Delta}{2\alpha} \mu! \, s^{\frac{\beta-1}{\alpha}-\mu} (2\bar\gamma  L)^{\mu - \frac{\beta}{\alpha} - \mu - 1 + \frac{1}{\alpha}} \int_0^\infty \frac{v^{-\frac{\beta-1}{\alpha}-1+\mu}}{[v + 1]^{\mu+1}} dv,
        \\&=
        \frac{L\Delta}{2\alpha} \mu! \, s^{\frac{\beta-1}{\alpha}-\mu} (2\bar{\gamma }L)^{-\frac{\beta-1}{\alpha}-1} \frac{\Gamma(\frac{\beta-1}{\alpha}+1)\Gamma(\mu-\frac{\beta-1}{\alpha})}{\Gamma(\mu+1)},
        \\\mathcal B^{(\mu)}(z)&\sim
        \frac{L\Delta}{2\alpha} (2\bar{\gamma }L)^{-\frac{\beta-1}{\alpha}-1} \Gamma\left(\frac{\beta-1}{\alpha}+1\right)\Gamma\left(\mu-\frac{\beta-1}{\alpha}\right) (1-z)^{-\left(\mu-\frac{\beta-1}{\alpha}\right)}.
        \end{align*}

        We have shown that 
        $\lim_{z \to 1^-} (1-z)^{\mu-\frac{\beta-1}{\alpha}} \mathcal B^{(\mu)}(z) = \frac{L\Delta}{2\alpha} (2\bar{\gamma }L)^{-\frac{\beta-1}{\alpha}-1} \Gamma\left(\frac{\beta-1}{\alpha}+1\right)\Gamma\left(\mu-\frac{\beta-1}{\alpha}\right)$. 
        Moreover, $T^{2+\frac{\beta-1}{\alpha}} (B_T - B_{T-1} ) \leq 0$ since $(B_T)$ is decreasing. Indeed, each term $[1-2\bar\gamma\frac{L}{i^\alpha}]^T$ is decreasing in $T$ since $\forall i\geq 1, [1-2\bar\gamma\frac{L}{i^\alpha}] \in (0,1)$ from \cref{as:eta}. Applying \cref{th:tauberian},   
        \[ 
        \lim_{z \to 1^-} (1-z)^{\mu-\frac{\beta-1}{\alpha}} \mathcal B^{(\mu)}(z) = \Gamma({\mu-\tfrac{\beta-1}{\alpha}} ) \cdot \lim_{T \to +\infty} T^{1+\frac{\beta-1}{\alpha}} B_T.
        \]
        As $T\to\infty$, this is equivalent to:
        \begin{align*}
        B_T &\sim \frac{1}{\Gamma\left(\mu-\frac{\beta-1}{\alpha}\right)} \left[ \frac{L\Delta}{2\alpha} (2\bar{\gamma }L)^{-\frac{\beta-1}{\alpha}-1} \Gamma\left(\frac{\beta-1}{\alpha}+1\right)\Gamma\left(\mu-\frac{\beta-1}{\alpha}\right) \right] T^{-\left(\frac{\beta-1}{\alpha}+1\right)},
        \\
         &\sim \frac{L\Delta}{2\alpha} (2\bar{\gamma }L)^{-\frac{\beta-1}{\alpha}-1} \Gamma\left(\frac{\beta-1}{\alpha}+1\right) T^{-\left(\frac{\beta-1}{\alpha}+1\right)},
        \\
         &\sim \frac{L\Delta}{2\alpha} \frac{\Gamma\left(\frac{\beta-1}{\alpha}+1\right)}{(2\bar{\gamma }L T)^{\frac{\beta-1}{\alpha}+1}}.
        \end{align*}

    \paragraph{Variance.}
    \[
        V_T = \frac{\gamma ^2 \sigma^2}{4\bar\gamma } \sum_{i=1}^{\infty} L i^{-\alpha} \left[ 1 - (1 - 2\bar\gamma  L i^{-\alpha})^T \right].
    \]
    Therefore, since $|1-2\bar\gamma  Li^{-\alpha}|<1$ and $\left[ 1 - (1 - 2\bar\gamma  L i^{-\alpha})^T \right]$ is dominated,
    \[
        \lim_{T \to \infty} V_T = \frac{\gamma ^2 \sigma^2 L}{4\bar\gamma } \zeta(\alpha) \qquad \text{with }\bar\gamma  = \gamma  -\frac{ \gamma ^2\tr(\Lambda)}{2}.
    \]

\end{proof}

\subsubsection{Asymptotic Risk with Averaging}
\label{sec:average}

For the constant schedule ($\eta_t=1$), let $\bar\theta_T:=\frac1T\sum_{t=1}^T\theta_t$ and
\[
\bar{\mathcal R}_T:=\mathbb E[f(\bar\theta_T)]-f_*
=\frac12\mathrm{Tr}\Big(\Lambda\,Q^\top\mathbb E\big[(\bar\theta_T-\theta_*)(\bar\theta_T-\theta_*)^\top\big]Q\Big).
\]
For $t\ge \ell$, $\mathbb E[\theta_t-\theta_*\mid\mathcal F_\ell]=(I-\gamma H)^{t-\ell}(\theta_\ell-\theta_*)$. Thus
$Q^\top\mathbb E[(\theta_t-\theta_*)(\theta_\ell-\theta_*)^\top]Q=(I-\gamma\Lambda)^{t-\ell}M_\ell$,
whose diagonal only involves $m_\ell$ (defined in \cref{eq:recurrence_moment}). The pairs with $t<\ell$ give the transposed matrices, with the same diagonal. Therefore
\[
\bar{\mathcal R}_T=\frac{1}{2T^2}\sum_{i=1}^\infty\lambda_i\Big[\sum_{\ell=1}^T m_{\ell,i}
+2\sum_{1\le \ell<t\le T}(1-\gamma\lambda_i)^{t-\ell}m_{\ell,i}\Big]=\bar B_T+\bar V_T,
\]
where, unrolling \cref{eq:recurrence_moment} with $\eta_t=1$ and using $1-2\gamma\lambda_i+\gamma^2\mathrm{Tr}(\Lambda)\lambda_i=1-2\bar\gamma\lambda_i$ (\cref{app:notations}),
\[
m_{\ell,i}=(1-2\bar\gamma\lambda_i)^{\ell}\,m_{0,i}+\gamma^2\sigma^2\lambda_i\sum_{j=0}^{\ell-1}(1-2\bar\gamma\lambda_i)^{j},
\]
the first term giving the bias and the second the variance. Factoring $m_{\ell,i}$ out of the two sums in the bracket and substituting $\lambda_i m_{0,i}=L\Delta\, i^{-\alpha-\beta}$ and $\lambda_i^2=L^2 i^{-2\alpha}$ (\cref{as:alphabeta}), we obtain
\begin{equation*}
\begin{cases}
\bar B_T=\dfrac{1}{2T^2}\displaystyle\sum_{i=1}^\infty\frac{L\Delta}{i^{\alpha+\beta}}
\sum_{\ell=1}^T\Big(1-\frac{2\bar\gamma L}{i^{\alpha}}\Big)^\ell
\Big[1+2\sum_{t=\ell+1}^{T}\Big(1-\frac{\gamma L}{i^{\alpha}}\Big)^{t-\ell}\Big],\\[5mm]
\bar V_T=\dfrac{1}{2T^2}\displaystyle\sum_{i=1}^\infty\frac{L^2\gamma^2\sigma^2}{i^{2\alpha}}
\sum_{\ell=1}^T\sum_{j=0}^{\ell-1}\Big(1-\frac{2\bar\gamma L}{i^{\alpha}}\Big)^j
\Big[1+2\sum_{t=\ell+1}^{T}\Big(1-\frac{\gamma L}{i^{\alpha}}\Big)^{t-\ell}\Big].
\end{cases}
\end{equation*}

\paragraph{Bias.}

The bracket gathers the pairs $(\ell,t)$ by $\min(\ell,t)$, so that
\[
\bar B_T=\frac{1}{2T^2}\sum_{i=1}^\infty\frac{L\Delta}{i^{\alpha+\beta}}
\sum_{\ell,t=1}^T\Big(1-\frac{2\bar\gamma L}{i^{\alpha}}\Big)^{\min(\ell,t)}
\Big(1-\frac{\gamma L}{i^{\alpha}}\Big)^{|t-\ell|}.
\]
Each pair $(\ell,t)$ contributes for every $T\ge\max(\ell,t)$. 
Writing $m = \min(\ell,t)$ and $d = |t-\ell|$, each value $(m,d)=(m,0)$ comes from the single pair $(m,m)$, while each value $(m,d)$ with $d \geq 1$ comes from the two pairs $(m, m+d)$ and $(m+d, m)$, which explains the factor $2$. Moreover, summing over $T \geq \max(\ell,t) = m+d$ removes the constraint $m+d \leq T$, so that the sum runs over all $m \geq 1$, $d \geq 0$ and factorizes, since $z^{\max(\ell,t)} = z^{m} z^{d}$.

Summing the geometric series in $m$ and $d$,
\[
\sum_{T=1}^\infty T^2\bar B_Tz^T=\frac{z}{2(1-z)}\sum_{i=1}^\infty\frac{L\Delta}{i^{\alpha+\beta}}
\frac{\big(1-\frac{2\bar\gamma L}{i^{\alpha}}\big)\big(1+z(1-\frac{\gamma L}{i^{\alpha}})\big)}
{\big(1-z(1-\frac{2\bar\gamma L}{i^{\alpha}})\big)\big(1-z(1-\frac{\gamma L}{i^{\alpha}})\big)}.
\]
With $z=1-s$, the two denominators read $s+(1-s)\frac{2\bar\gamma L}{i^{\alpha}}$ and
$s+(1-s)\frac{\gamma L}{i^{\alpha}}$, and the numerator is $2+O(s+i^{-\alpha})$. Setting
$x_i=is^{1/\alpha}$ (Riemann sum, dominated convergence) and $w=\gamma Lx^{-\alpha}$,
\[
\sum_{T=1}^\infty T^2\bar B_T(1-s)^T\sim\frac{L\Delta}{\alpha}(\gamma L)^{-\omega}s^{-(3-\omega)}
\int_0^\infty\frac{w^{\omega-1}dw}{(1+w)(1+\frac{2\bar\gamma}{\gamma}w)}
=\frac{L\Delta}{\alpha}L^{-\omega}\frac{\pi}{\sin(\pi\omega)}
\frac{(2\bar\gamma)^{1-\omega}-\gamma^{1-\omega}}{2\bar\gamma-\gamma}\,s^{-(3-\omega)},
\]
the integral being finite iff $0<\omega<2$. \cref{th:tauberian} with $\mu=0$ and
$p=3-\omega$ gives, assuming its one-sided condition holds,
\[
\bar B_T\sim\frac{L\Delta}{\alpha}\frac{\Gamma(\omega)}{(2-\omega)(1-\omega)}
\frac{(2\bar\gamma)^{1-\omega}-\gamma^{1-\omega}}{2\bar\gamma-\gamma}(LT)^{-\omega},
\qquad 0<\omega<2.
\]
For $\omega>2$, $T^2\bar B_T$ is a sum of nonnegative terms over the pairs $\max(\ell,t)\le T$, so by monotone convergence it tends to its value over all pairs, which is finite precisely because $\beta-\alpha>1$:
$\bar B_T\sim\frac{\Delta}{4\gamma\bar\gamma L}\sum_{i=1}^\infty i^{\alpha-\beta}
\big(1-\frac{2\bar\gamma L}{i^{\alpha}}\big)\big(2-\frac{\gamma L}{i^{\alpha}}\big)\,T^{-2}$.

\paragraph{Variance.}

Similarly, since $\sum_{j=0}^{m-1}\big(1-\frac{2\bar\gamma L}{i^{\alpha}}\big)^j
=\frac{i^{\alpha}}{2\bar\gamma L}\big[1-\big(1-\frac{2\bar\gamma L}{i^{\alpha}}\big)^m\big]$,
\[
\sum_{T=1}^\infty T^2\bar V_Tz^T=\frac{z}{2(1-z)^2}\sum_{i=1}^\infty\frac{L^2\gamma^2\sigma^2}{i^{2\alpha}}
\frac{1+z(1-\frac{\gamma L}{i^{\alpha}})}
{\big(1-z(1-\frac{2\bar\gamma L}{i^{\alpha}})\big)\big(1-z(1-\frac{\gamma L}{i^{\alpha}})\big)}.
\]
The same substitutions give, for every $\alpha>1$,
\[
\sum_{T=1}^\infty T^2\bar V_T(1-s)^T\sim\frac{\sigma^2}{\alpha}(\gamma L)^{\frac1\alpha}
s^{-(2+\frac1\alpha)}\int_0^\infty\frac{w^{1-\frac1\alpha}dw}{(1+w)(1+\frac{2\bar\gamma}{\gamma}w)},
\]
and \cref{th:tauberian} with $\mu=0$ and $p=2+\frac1\alpha$ yields, under the same assumption,
\[
\bar V_T\sim\frac{\alpha\,\Gamma(1-\frac1\alpha)}{\alpha+1}\,\sigma^2L^{\frac1\alpha}\gamma^2
\frac{\gamma^{\frac1\alpha-1}-(2\bar\gamma)^{\frac1\alpha-1}}{2\bar\gamma-\gamma}\,T^{\frac1\alpha-1}.
\]

If we assume $\gamma\to0$, then $\bar\gamma\sim\gamma$ and we recover
$\bar B_T\sim\frac{L\Delta}{\alpha}\frac{\Gamma(\omega)(2^{1-\omega}-1)}{(2-\omega)(1-\omega)}(L\gamma T)^{-\omega}$
if $\omega<2$,
$\bar B_T\sim\frac{\Delta\,\zeta(\beta-\alpha)}{2L\gamma^2}T^{-2}$ if $\omega>2$, and
$\bar V_T\sim\frac{\alpha\Gamma(1-\frac1\alpha)(1-2^{\frac1\alpha-1})}{\alpha+1}\sigma^2(L\gamma)^{\frac1\alpha}T^{\frac1\alpha-1}$.

\subsubsection{Linear Schedule (\cref{th:wsd}, case $c=1$)}
\label{app:linear}
\begin{proof}

We have:
\[\left\{\begin{aligned}
\eta_k &=  1 - \frac{k}{T}, \\
\bar\nu_k &= \frac{2\gamma\eta_k - \gamma^2\eta_k^2\tr(\Lambda)}{2} = \gamma\eta_k - \frac{1}{2}\tr(\Lambda)\gamma^2\eta_k^2, \\
\Pi_{T,i} &= \prod_{k=0}^{T-1} \left[ 1 - 2\bar\nu_k \frac{L}{i^\alpha} \right].
\end{aligned}\right.\]

By \cref{as:eta}, $\bar\gamma = \frac{2\gamma - \gamma^2\tr(\Lambda)}{2} <\frac{1}{2L}$. Then $\forall i\geq1, \forall  k\in\{1,...,T\}, 1 - 2\bar\nu_k \frac{L}{i^\alpha}  \geq0$.

\paragraph{Tauberian condition.} Let $p_0 := 1 - 1/\zeta(\alpha) > 0$ and $\delta := 2 - \gamma\,\mathrm{Tr}(\Lambda) > 0$. For $\eta \in [0,1]$, the factor $P_i(\eta) := 1 - 2\gamma\eta\lambda_i + \gamma^2\eta^2\mathrm{Tr}(\Lambda)\lambda_i$ satisfies $P_i(\eta) \geq p_0$, $|P_i'(\eta)| \leq 2\gamma\lambda_i$ and $1 - P_i(\eta) \geq \delta\gamma\lambda_i\eta$. Since $|\eta^{(T)}_k - \eta^{(T-1)}_k| \leq \frac{k}{cT(T-1)}$ for $k \leq T-2$ (also for WSD, with $\eta^{(T)}_k = \min(1, \frac{T-k}{cT})$), we get
\[
  \big|\log (\Pi_{T,i}) - \log (\Pi_{T-1,i})\big| \leq \frac{2\gamma\lambda_i}{p_0}\Big(\frac{1}{2c} + 1\Big) =: C_1\gamma\lambda_i ,
\]
which is bounded, hence $|\Pi_{T,i} - \Pi_{T-1,i}| \leq C_2\,\gamma\lambda_i\,\Pi_{T-1,i}$. Moreover, $\sum_{k} \eta^{(T-1)}_k \geq (T-1)/2$ gives $\Pi_{T-1,i} \leq e^{-\delta\gamma\lambda_i (T-1)/2}$. Therefore
\[
  |B_T - B_{T-1}| \leq C_3 \sum_{i \geq 1} i^{-2\alpha-\beta}\, e^{-\delta\gamma L T i^{-\alpha}/2} = O\big(T^{-1-\omega}\big),
\]
   by comparison with $\int_0^\infty x^{-2\alpha-\beta} e^{-aTx^{-\alpha}}\,dx = \frac{\Gamma(\omega+1)}{\alpha}(aT)^{-1-\omega}$, where $a := \delta\gamma L/2$. Hence the condition of \cref{th:tauberian} for the bias, $T^{1+\omega}(B_T - B_{T-1}) \leq K$, holds for all $\gamma < 2/\mathrm{Tr}(\Lambda)$. 
For the variance, write $V_T = \frac12\sum_i w_i \sum_{\ell=1}^{T} \eta_T(\ell)^2\, \Pi_T(\ell)$, where $w_i := L^2\gamma^2\sigma^2 i^{-2\alpha}$, $\eta_T(\ell) := \min(1, \frac{\ell}{cT})$ is the step size $\ell$ steps before the end, and $\Pi_T(\ell) := \prod_{\ell' < \ell} P_i(\eta_T(\ell'))$. Aligning the two horizons from the end, $0 \leq \eta_{T-1}(\ell) - \eta_T(\ell) \leq \eta_{T-1}(\ell)/T$, so that, with $E(\ell) := \sum_{\ell' < \ell} \eta_{T-1}(\ell')$, we have $|\log (\Pi_T(\ell)) - \log (\Pi_{T-1}(\ell))| \leq \frac{2\gamma\lambda_i}{p_0 T} E(\ell)$ and $\Pi_{T-1}(\ell) \leq e^{-\delta\gamma\lambda_i E(\ell)}$. Hence, for $T$ large enough,
\[
  |V_T - V_{T-1}| \leq \frac{C}{T} \sum_{i} w_i \sum_{\ell < T} \eta_{T-1}(\ell)^2\, e^{-\delta\gamma\lambda_i E(\ell)/2} + \frac12 \sum_i w_i\, e^{-\delta\gamma\lambda_i (T-1)/2},
\]
the last term coming from the additional step. Since $E(\ell) \geq \frac{\ell(\ell-1)}{2cT}$ for $\ell \leq cT$ and $E(\ell) \geq \ell - \frac{cT}{2} - 2$ otherwise, the first double sum is $O(T^{p-1})$, as for $\widetilde V_T$: modes $i \gtrsim T^{1/\alpha}$ contribute $O(T^{1/\alpha - 1})$, and modes $i \lesssim T^{1/\alpha}$ contribute $O\big(T^{-1/2}\sum_{i \lesssim T^{1/\alpha}} i^{-\alpha/2}\big)$. The second sum is $O(T^{1/\alpha - 2})$. Therefore $T^{2-p}|V_T - V_{T-1}| \leq K$, with $p = 1/\alpha$ if $\alpha < 2$ and $p = 1/2$ if $\alpha > 2$.

\paragraph{Bias.}

Let $j = T - k$. As $k$ goes from $0$ to $T-1$, $j$ goes from $T$ down to $1$, and $1 - \frac{k}{T} = \frac{j}{T}$.

\begin{align*}
1 - 2\bar\nu_k \frac{L}{i^\alpha} &= 1 - \frac{2L}{i^\alpha} \left[ \gamma   \left(\frac{j}{T}\right) - \frac{1}{2}\tr(\Lambda)\gamma  ^2 \left(\frac{j}{T}\right)^2 \right] = 1 - \underbrace{\frac{2\gamma   L}{i^\alpha}}_{\a_i=i^{-\alpha}\a} \left(\frac{j}{T}\right) + \underbrace{\frac{\tr(\Lambda)\gamma  ^2 L}{i^\alpha}}_{\b_i = i^{-\alpha}\b} \left(\frac{j}{T}\right)^2.
\end{align*}
Therefore,
\begin{align*}
\Pi_{T,i} &= \prod_{j=1}^{T} \left[ \frac{\b_i}{T^2} j^2 - \frac{\a_i}{T} j + 1 \right] = \left( \frac{\b_i}{T^2} \right)^T \prod_{j=1}^{T} \left( j^2 - \frac{\a_i T}{\b_i} j + \frac{T^2}{\b_i} \right),
\\&=
({\b_i})^T \prod_{j=1}^{T} \left( \left(\frac{j}{T}\right)^2 - \frac{\a_i }{\b_i} \left(\frac{j}{T}\right) + \frac{1}{\b_i} \right)=
 \exp\left[\sum_{j=1}^{T} \log\left( \b_i\left(\frac{j}{T}\right)^2 - \a_i \left(\frac{j}{T}\right) + 1\right)\right],
\\&\sim \sqrt{\b_i-\a_i+1}\exp\left[T\int_{0}^{1} \log\left( \b_iu^2 - \a_iu + 1\right)du\right].
\end{align*}

By Euler-Maclaurin, $\displaystyle \sum_{j=1}^{T} f\left(\frac{j}{T}\right) = T\int_{0}^{1} f(u)du + \frac{f(1)-f(0)}{2} + r_{i,T}$ 
and $r_{i,T} = O(\frac{i^{-\alpha}}{T}).$

Let $\s_{i,T} = \Pi_{T,i}\exp\left[-T\int_{0}^{1} \log\left( \b_iu^2 - \a_iu + 1\right)du\right] = e^{r_{i,T}}\sqrt{\b_i-\a_i+1} \to 1$ as $i\to \infty$.

$\s_i := \lim_{T\to \infty} \s_{i,T} = \sqrt{\b_i-\a_i+1}$. The convergence is uniform with respect to $i$.
As $z\to 1$,
\begin{align*}
\mathcal B(z) &= \frac{L\Delta}{2}\sum_{T=0}^\infty \sum_{i=1}^\infty \frac{\s_{i,T}}{i^{\alpha+\beta}}e^{T\int_0^1\log(\frac{\b}{i^\alpha}u^2 - \frac{\a}{i^\alpha}u+1)du}z^T,
\\&
\sim \frac{L\Delta}{2} \sum_{i=1}^\infty \frac{\s_i}{i^{\alpha+\beta}} \frac{1}{1-ze^{\int_0^1\log(\frac{\b}{i^\alpha}u^2 - \frac{\a}{i^\alpha}u+1)du}},
\end{align*}
the inversion being possible by Fubini's theorem (all terms are positive).

Let us take the $\mu$-th derivative of $\mathcal B$  so that $\mathcal B^{(\mu)}(z)$ diverges as $z \to 1$.

$\displaystyle \mathcal B(z) = \frac{L\Delta}{2} \sum_{i=1}^\infty \sum_{T=0}^\infty \frac{\s_{i,T}}{i^{\alpha+\beta}} e^{TI_i}z^T$ so $\displaystyle\mathcal B'(z) =  \frac{L\Delta}{2} \sum_{i=1}^\infty \sum_{T=1}^\infty \frac{T\s_{i,T}}{i^{\alpha+\beta}} e^{TI_i}z^{T-1}$. 

We can iterate this $\mu$ times and we get $(z \to 1)$:
\begin{align*}
\mathcal B^{(\mu)}(z) \sim \frac{L\Delta}{2} \mu! \sum_{i=1}^\infty \frac{\s_i}{i^{\alpha+\beta}} \frac{\left(e^{I_i}\right)^\mu}{\left(1 - z e^{I_i}\right)^{\mu+1}}.
\end{align*}

First, $\s_i = 1+ O(i^{-\alpha})$. Second, we know that $\int_0^1\ln(\frac{\b}{i^\alpha}u^2 - \frac{\a}{i^\alpha}u+1)du = \frac{\b}{3i^\alpha} - \frac{\a}{2i^\alpha} + O(i^{-2\alpha})$ so $e^{I_i} = 1 - Ci^{-\alpha}+O(i^{-2\alpha})$ with $C=\frac{\a}{2}-\frac{\b}{3}$.

By defining $x_i = i s^{1/\alpha}$ and forcing the step size $\Delta x = s^{1/\alpha}$, we can reformulate the sum for large $i$ (where $\s_i \to 1$):
\begin{align*}
\mathcal{B}^{(\mu)}(1-s) &\sim 
\frac{L\Delta}{2} \mu! \sum_{i=1}^\infty \frac{ (1+O((x_is^{-1/\alpha})^{-\alpha}))(x_i s^{-1/\alpha})^{-(\alpha+\beta)} }{ [s + (1-s)( C (x_i s^{-1/\alpha})^{-\alpha}+O((x_is^{-1/\alpha})^{-2\alpha}))]^{\mu+1} },
\\&=
\frac{L\Delta}{2} \mu! \sum_{i=1}^\infty \frac{(1+O(x_i^{-\alpha}s)) x_i^{-(\alpha+\beta)} s^{1 + \frac{\beta}{\alpha}} }{ [s + (1-s) C x_i^{-\alpha} s +O(x_i^{-2\alpha}s^2)]^{\mu+1} },
\\&=
\frac{L\Delta}{2} \mu! \sum_{i=1}^\infty \frac{(1+O(x_i^{-\alpha}s))  x_i^{-(\alpha+\beta)} s^{1 + \frac{\beta}{\alpha}} }{ s^{\mu+1} [1 + (1-s) C x_i^{-\alpha}+ O(x_i^{-2\alpha}s) ]^{\mu+1} }.
\end{align*}

To recognize the standard Riemann form, we multiply the numerator and denominator by $x_i^{\alpha(\mu+1)}$:
\begin{align*}
\mathcal{B}^{(\mu)}(1-s) &\sim
\frac{L\Delta}{2} \mu! \, s^{\frac{\beta}{\alpha} - \mu} \sum_{i=1}^\infty \frac{(1+O(x_i^{-\alpha}s))  x_i^{\alpha\mu - \beta} }{ [x_i^\alpha + (1-s) C + O(x_i^{-\alpha}s) ]^{\mu+1} },
\\&=
\frac{L\Delta}{2} \mu! \, s^{\frac{\beta-1}{\alpha} - \mu} \sum_{i=1}^\infty \underbrace{ \frac{ (1+O(x_i^{-\alpha}s)) x_i^{\alpha\mu - \beta} }{ [x_i^\alpha + (1-s) C + O(x_i^{-\alpha}s)]^{\mu+1} }  }_{=:U_i}\underbrace{s^{1/\alpha}}_{= \Delta x}.
\end{align*}

Set $
\begin{aligned}[t] g_s : 
[0,\infty) &\to \R,\\ x &\mapsto U_i 
\end{aligned}$ with $i = \lceil x s^{-1/\alpha}\rceil $. 
Then $g_s(x_i) = U_i$, and since $s^{1/\alpha} = \Delta x$,
\[
 \sum_{i=1}^\infty  \frac{ x_i^{\alpha\mu - \beta} }{ [x_i^\alpha + (1-s) C + O(x_i^{-\alpha}s)]^{\mu+1} } s^{1/\alpha} =  \int_{0}^\infty  g_s(u) du.
\]

For every $x \in [0,\infty)$, $g_s(x) \xrightarrow{s\to0} \frac{x^{\alpha\mu-\beta}}{(x^\alpha +C)^{\mu+1}}$. By the dominated convergence theorem,
\begin{align*}
\mathcal{B}^{(\mu)}(1-s) &\sim
\frac{L\Delta}{2} \mu! \, s^{\frac{\beta-1}{\alpha} - \mu} \int_0^\infty \frac{x^{\alpha\mu - \beta}}{[x^\alpha + C]^{\mu+1}} dx.
\end{align*}

Therefore as $s \to 0$,
\begin{align*}
\mathcal B^{(\mu)} \sim \frac{L\Delta}{2\alpha} C^{-\frac{\beta-1}{\alpha}-1} \Gamma\left(\frac{\beta-1}{\alpha}+1\right)\Gamma\left(\mu - \frac{\beta-1}{\alpha}\right) s^{\frac{\beta-1}{\alpha} - \mu}.
\end{align*}

Moreover, the Tauberian condition established above gives $T^{1+\omega}(B_T-B_{T-1}) \le K$, so \cref{th:tauberian} applies with $p = \mu - \frac{\beta-1}{\alpha}$ (so that $2+\mu-p = 1+\omega$ and $1+\mu-p = \omega$):
\[
\lim_{z\to 1^-} (1-z)^{\mu-\frac{\beta-1}{\alpha}}\,\mathcal B^{(\mu)}(z)
= \Gamma\Big(\mu-\frac{\beta-1}{\alpha}\Big)\lim_{T\to\infty} T^{\omega} B_T ,
\]
hence
\[
B_T \sim \frac{L\Delta}{2\alpha} \frac{\Gamma\left(\frac{\beta-1}{\alpha} + 1\right)}{\left[ \gamma  L T \left( 1 - \frac{\gamma  Tr(\Lambda)}{3} \right) \right]^{\frac{\beta-1}{\alpha} + 1}}.
\]

\paragraph{Variance.}

The variance is defined as 
\[
V_T = \frac{1}{2}\sum_{i=1}^\infty \sum_{k=0}^{T-1}  \frac{L^2}{i^{2\alpha}}\gamma^2\eta_k^2 \sigma^2 \prod_{m=k+1}^{T-1} \left[ 1 - 2\bar{\nu}_m \frac{L}{i^\alpha} \right].
\]

The same method can be used for the product.
Let $j = T-m$.
\[
\prod_{m=k+1}^{T-1} \left[ 1 - 2\bar{\nu}_m \frac{L}{i^\alpha} \right] 
=
\prod_{j=1}^{T-k-1} \left[ 1 - 2\bar{\nu}_{T-j} \frac{L}{i^\alpha} \right] =
\exp\left[\sum_{j=1}^{T-k-1} \log\left( \b_i\left(\frac{j}{T}\right)^2 - \a_i \left(\frac{j}{T}\right) + 1\right)\right].
\]

By Euler-MacLaurin, 
\begin{multline*}\sum_{j=1}^{T-k-1} \log\left( \b_i\left(\frac{j}{T}\right)^2 - \a_i \left(\frac{j}{T}\right) + 1\right) = T \int_0^{1-\frac{k+1}{T}} \log\left( \b_i u^2 - \a_i u + 1 \right) du  \\+ \frac{\log(\b_i v_k^2 - \a_i v_k + 1)}{2} +o(1).\end{multline*}

Hence,
\[
\prod_{m=k+1}^{T-1} \left[ 1 - 2\bar{\nu}_m \frac{L}{i^\alpha} \right]= \underbrace{(1+o(1))\sqrt{\b_i v_k^2 - \a_i v_k + 1}}_{\s_{i,k} \xrightarrow{i} 1}  \exp\left[ T \int_0^{1-\frac{k+1}{T}} \log\left( \b_i u^2 - \a_i u + 1 \right) du \right],
\]
with $v_k = 1 - \frac{k+1}{T}$.
Let $\ell = T - k$ and $x = \frac{\ell-1}{T}$. The sum over $\ell$ is a Riemann's sum, so
\begin{align*}
V_T &= \sum_{i=1}^\infty \sum_{\ell=1}^{T} \frac{L^2}{i^{2\alpha}}\gamma ^2 \sigma^2 \left(\frac{\ell}{T}\right)^2 \s_{i,T-\ell} \exp\left[ T \int_0^{\frac{\ell-1}{T}} \log\left( \b_i u^2 - \a_i u + 1 \right) du \right], \\
&= \sum_{i=1}^\infty T \int_0^1 \frac{L^2}{i^{2\alpha}}\gamma ^2 \sigma^2 x^2 \s_{i,T(1-x)}\exp\left[ T \int_0^{x} \log\left( \b_i u^2 - \a_i u + 1 \right) du \right] dx +O(1/T).
\end{align*}

\paragraph{Case $1 < \alpha < 2$:}

Knowing that $\s_{i,k} = 1+O(i^{-\alpha})$ uniformly, the $z$-transform is:
\begin{align*}
\mathcal V(z) &=  \frac{1}{2} \sum_{T=0}^\infty \left( \sum_{i=1}^\infty \int_0^1 \frac{L^2}{i^{2\alpha}}\gamma ^2 \sigma^2 x^2  (1+O(
i^{-\alpha}))\exp\left[ T \int_0^{x-\frac{1}{T}} \log\left( \b_i u^2 - \a_i u + 1 \right) du \right] dx + O(1/T) \right)z^T, 
\\&=
\frac{1}{2}\sum_{i=0}^\infty \int_0^1 \frac{L^2}{i^{2\alpha}}\gamma ^2 \sigma^2 x^2 (1+O(
i^{-\alpha}))  \frac{ze^{I_i(x)}}{(1-ze^{I_i(x)})^2} dx + \underbrace{\frac{1}{2}\sum_{T=0}^\infty O(1/T)z^T}_{=:R(z)}.
\end{align*}

Indeed, $\sum_{T=0}^\infty T y^T = \sum_{T=0}^\infty y\frac{d}{dy}y^{T} = y\frac{d}{dy} \frac{1}{1-y} = \frac{y}{(1-y)^2} $. For $i \to \infty$,
\begin{align*}
I_i(x) &= \int_0^{x}\log\left( \b_i u^2 - \a_i u + 1 \right) du 
= \int_0^{x} \b_i u^2 - \a_i u + O(i^{-2\alpha})du ,
\\&= i^{-\alpha}\left(\frac{\b}{3}x^3 - \frac{\a}{2}x^2\right) + O(i^{-2\alpha})
=
-i^{-\alpha}C(x) +O(i^{-2\alpha}),
\end{align*}
Where $C(x) = -\frac{\b}{3}x^3 + \frac{\a}{2}x^2$.

\begin{align*}
\mathcal V(z)-R(z) &=\frac{1}{2} \sum_{i=1}^\infty \int_0^1 \frac{L^2}{i^{2\alpha}}\gamma ^2 \sigma^2 x^2 (1+O(
i^{-\alpha}))\frac{ze^{I_i(x)}}{(1-ze^{I_i(x)})^2} dx ,
\\&=
\frac{1}{2}\sum_{i=1}^\infty \int_0^1 \frac{L^2}{i^{2\alpha}}\gamma ^2 \sigma^2 x^2  (1+O(
i^{-\alpha}))\frac{(1-s)e^{I_i(x)}}{(1-(1-s)(1-i^{-\alpha}C(x) +O(i^{-2\alpha})))^2} dx,
\\&=
\frac{1}{2}\sum_{i=1}^\infty \int_0^1 \frac{L^2}{i^{2\alpha}}\gamma ^2 \sigma^2 x^2 \frac{1+O(i^{-\alpha})}{(s + (1-s)i^{-\alpha}C(x)+O(i^{-2\alpha}))^2} dx,
\\&=
\frac{1}{2}\int_0^1 L^2 \gamma ^2 \sigma^2   \sum_{i=1}^\infty  \frac{x^2}{i^{2\alpha}} \frac{1+O(i^{-\alpha})}{(s + (1-s)i^{-\alpha}C(x)+O(i^{-2\alpha}))^2} dx.
\end{align*}

Here the sum diverges when $z \to 1$, so there is no need to differentiate $\mathcal V-R$. We then have to turn the sum into an integral.
Set $x_i = is^{1/\alpha}$. Therefore $\Delta x_i = s^{1/\alpha}$.

\begin{align*}
\mathcal V(z) -R(z)& =
\frac{1}{2}\sum_{i=1}^\infty \int_0^1 L^2 \gamma ^2 \sigma^2    \frac{u^2}{(x_is^{-1/\alpha})^{2\alpha}} \frac{1+O((x_is^{-1/\alpha})^{-\alpha})}{(s + (1-s) (x_is^{-1/\alpha})^{-\alpha}C(u)+O((x_is^{-1/\alpha})^{-2\alpha}))^2} du,
\\&=
\frac{L^2 \gamma ^2 \sigma^2}{2}\sum_{i=1}^\infty \int_0^1     \frac{u^2}{x_i^{2\alpha}s^{-2}} \frac{1+O(x_i^{-\alpha}s)}{(s + (1-s) x_i^{-\alpha}sC(u)+O(x_i^{-2\alpha}s^{2}))^2} du,
\\&=
\frac{L^2 \gamma ^2 \sigma^2}{2}s^{-1/\alpha}\sum_{i=1}^\infty \underbrace{\int_0^1     \frac{u^2}{x_i^{2\alpha}} \frac{1+O(x_i^{-\alpha}s)}{(1 + (1-s) x_i^{-\alpha}C(u)+O(x_i^{-2\alpha}s))^2} s^{1/\alpha}du}_{=:U_i(s)}.
\end{align*}

Set $g_s : 
\begin{aligned}[t]
[0,\infty) &\to \R,\\ x &\mapsto U_i(s)
\end{aligned}$ with $i = \lceil x s^{-1/\alpha}\rceil $. Then $\int_0^\infty g_s(x)dx = \sum_{i=1}^\infty U_i(s)$.

Moreover, $\displaystyle \forall i\geq1, g_s(x) \xrightarrow{s\to0}\int_0^1 \frac{u^2}{x^{2\alpha}}\frac{1}{(1+x^{-\alpha}C(u))^2}du.$

By the Dominated convergence theorem,
\[\sum_{i=1}U_i(s) \xrightarrow{s\to 0} \int_0^\infty\int_0^1 \frac{u^2}{x^{2\alpha}}\frac{1}{(1+x^{-\alpha}C(u))^2}dudx.
\]

\bigskip

Let $w = C(u)x^{-\alpha}$, then $dx = -\frac{1}{\alpha} C(u)^{1/\alpha} w^{-1/\alpha - 1} dw$.
\begin{align*}
\int_{0}^\infty \frac{x^{-2\alpha}}{(1 + C(u)x^{-\alpha})^2} dx
&= \int_{0}^\infty \frac{C(u)^{-2} w^2}{(1+w)^2} \left( \frac{1}{\alpha} C(u)^{1/\alpha} w^{-1/\alpha - 1} \right) dw,
\\&= \frac{1}{\alpha} C(u)^{\frac{1-2\alpha}{\alpha}} \int_0^\infty \frac{w^{1-1/\alpha}}{(1+w)^2} dw,
\\&= \frac{1}{\alpha} C(u)^{\frac{1-2\alpha}{\alpha}} \Gamma\left(2 - \frac{1}{\alpha}\right) \Gamma\left(\frac{1}{\alpha}\right).
\end{align*}

Therefore, the behavior of the z-transform is:
\begin{align*}
\mathcal V(1-s) - R(1-s) &\sim \frac{L^2 \gamma ^2 \sigma^2 }{2\alpha}s^{-1/\alpha}   \Gamma\left(2 - \frac{1}{\alpha}\right) \Gamma\left(\frac{1}{\alpha}\right) \int_0^1 u^2 C(u)^{\frac{1-2\alpha}{\alpha}} du .
\end{align*}

By the Tauberian theorem, if $\mathcal V(1-s)-R(1-s) \sim K s^{-\omega}$ with $\omega = \frac{1}{\alpha} > 0$, then the $T$-th coefficient behaves as $V_T-O(1/T) \sim \frac{K}{\Gamma(\omega)} T^{\omega - 1}$.
\begin{align*}
V_T -O(1/T)&\sim \frac{1}{\Gamma\left(\frac{1}{\alpha}\right)} \left[ \frac{L^2 \gamma ^2 \sigma^2 }{2\alpha} \Gamma\left(2 - \frac{1}{\alpha}\right) \Gamma\left(\frac{1}{\alpha}\right) \int_0^1 u^2 C(u)^{\frac{1-2\alpha}{\alpha}} du \right] T^{\frac{1-\alpha}{\alpha}} ,
\\
V_T&\sim \frac{L^2 \gamma ^2 \sigma^2}{2\alpha} \Gamma\left(2 - \frac{1}{\alpha}\right) \left( \int_0^1 u^2 C(u)^{\frac{1-2\alpha}{\alpha}} du \right) T^{\frac{1-\alpha}{\alpha}}.
\end{align*}
Because $O(1/T)$ is negligible before $T^{\frac{1-\alpha}{\alpha}}$.

We can now substitute $C(u)$ into the temporal integral and factor out the leading terms:
\begin{align*}
\int_0^1 u^2 C(u)^{\frac{1-2\alpha}{\alpha}} du &= \int_0^1 u^2 \left( \frac{\a}{2}u^2 - \frac{\b}{3}u^3 \right)^{\frac{1-2\alpha}{\alpha}} du ,
\\&= \left(\frac{\a}{2}\right)^{\frac{1-2\alpha}{\alpha}} \int_0^1 u^2 \cdot u^{2\left(\frac{1-2\alpha}{\alpha}\right)} \left( 1 - \frac{2\b}{3\a}u \right)^{\frac{1-2\alpha}{\alpha}} du ,
\\&= \left(\frac{\a}{2}\right)^{\frac{1-2\alpha}{\alpha}} \int_0^1 u^{\frac{2-2\alpha}{\alpha}} \left( 1 - \frac{2\b}{3\a}u \right)^{\frac{1-2\alpha}{\alpha}} du.
\end{align*}

Let $\mathcal{I}_{\alpha, \gamma } = \int_0^1 u^{\frac{2-2\alpha}{\alpha}} \left( 1 - \frac{\gamma  \tr(\Lambda)}{3} u \right)^{\frac{1-2\alpha}{\alpha}} du$ .

The final asymptotic equivalent for the variance under a linear schedule is therefore :
\begin{align*}V_T &\sim  \frac{L^2 \gamma ^2 \sigma^2}{2\alpha} \Gamma\left(2 - \frac{1}{\alpha}\right) \left(\gamma  L\right)^{\frac{1-2\alpha}{\alpha}} \mathcal{I}_{\alpha, \gamma } T^{\frac{1-\alpha}{\alpha}}, \qquad \text{if }1<\alpha<2.
\end{align*}

\paragraph{Case $\alpha > 2$.}
The spectral tail decays too rapidly; the integral $\mathcal{I}_{\alpha, \gamma }$ diverges.

This time, we cannot concentrate on large $i$. However, with a fixed $i$, knowing $C(u) =\gamma  L u^2 +O(u^3)$ as $u \to 0$, we can cut the integral along $u$.
\begin{align*}
\mathcal V(z)-R(z) &= \frac{1}{2} \sum_{i=1}^\infty \int_0^1 \frac{L^2}{i^{2\alpha}}\gamma ^2 \sigma^2 u^2   S_{i,T(1-u)}\frac{ze^{I_i(u)}}{(1-ze^{I_i(u)})^2} du,
\\
&=
\frac{1}{2} \sum_{i=1}^\infty \int_0^1 \frac{L^2}{i^{2\alpha}}\gamma ^2 \sigma^2 u^2 (1+O(i^{-\alpha}u))\frac{ze^{I_i(u)}}{(1-ze^{I_i(u)})^2} du,
\\&=
\frac{1}{2} \sum_{i=1}^\infty \int_0^1 \frac{L^2}{i^{2\alpha}}\gamma ^2 \sigma^2 u^2 (1+O(i^{-\alpha}u))\frac{z(1-C(u)i^{-\alpha}+O(i^{-\alpha}u^2))}{(1-z(1-C(u)i^{-\alpha}+O(i^{-\alpha}u^2)))^2} du,
\\
\mathcal V(1-s) - R(1-s)&=
\frac{1}{2} \sum_{i=1}^\infty \int_0^1 \frac{L^2}{i^{2\alpha}}\gamma ^2 \sigma^2 u^2 (1+O(i^{-\alpha}u))\frac{(1-s)(1-C(u)i^{-\alpha}+O(i^{-\alpha}u^2))}{(1-(1-s)(1-C(u)i^{-\alpha}+O(i^{-\alpha}u^2)))^2} du.
\end{align*}

Let $u = s^{1/2} w$. $du=s^{1/2}dw$,
\begin{align*}
\mathcal V(1-s) - R(1-s) 
&= \frac{L^2 \gamma ^2 \sigma^2}{2} \sum_{i=1}^\infty \frac{1}{i^{2\alpha}} \int_0^{s^{-1/2}} \frac{sw^2 \left(1 + {O}(s + i^{-\alpha}s^{1/2}w)\right)}{\left(s + (1-s)i^{-\alpha}C(s^{1/2}w) + {O}(i^{-\alpha}sw^2)\right)^2} s^{1/2}dw,
\\&=
\frac{L^2 \gamma ^2 \sigma^2}{2} s^{-1/2}\sum_{i=1}^\infty
\underbrace{
\frac{1}{i^{2\alpha}} \int_0^{s^{-1/2}} \frac{w^2 \left(1 + {O}(s + i^{-\alpha}s^{1/2}w)\right)}{\left(1 + (1-s)i^{-\alpha}C(s^{1/2}w)/s + {O}(i^{-\alpha}w^2)\right)^2} dw
}_{U_i(s)}.
\end{align*}

First, $\displaystyle
\frac{C(s^{1/2}w)}{s}  = \frac{-\frac{\b}{3}(s^{1/2}w)^3+\frac{\a}{2}(s^{1/2}w)^2}{s} \xrightarrow{s \to 0} \frac{\a}{2}w^2.
$ For every $i \geq 1$, by the dominated convergence theorem,
\begin{align*}
U_i(s) &= \frac{1}{i^{2\alpha}}\int_0^\infty \frac{w^2 \left(1 + {O}(s + i^{-\alpha}s^{1/2}w)\right)}{\left(1 + (1-s)i^{-\alpha}C(s^{1/2}w)/s + {O}(i^{-\alpha}w^2)\right)^2} \mathbf 1_{[0, s^{-1/2}]}(w)dw,
\\&
\xrightarrow{s\to0} \frac{1}{i^{2\alpha}}\int_0^\infty\frac{w^2}{(1+\frac{\a}{2}w^2i^{-\alpha})^2}dw.
\end{align*}

Therefore, using the DCT once again,
\[
\sum_{i=1}^\infty U_i(s) \xrightarrow{s\to0} \sum_{i=1}^\infty \frac{1}{i^{2\alpha}} \int_0^\infty\frac{w^2}{(1+\frac{\a}{2}w^2i^{-\alpha})^2}dw.
\]

Let $u = \sqrt{\frac{\a}{2}i^{-\alpha}}w$.

The integral becomes:
\[
\int_0^{\infty} \frac{\frac{2}{i^{-\alpha}\a}\, w^2}{(1+w^2)^2} \sqrt{\frac{2}{i^{-\alpha}\a}}\,\mathrm{d}w
=  (i^{-\alpha}\gamma  L)^{-3/2} \int_0^{\infty} \frac{w^2}{(1+w^2)^2}\,\mathrm{d}w = (i^{-\alpha}\gamma  L)^{-3/2} \frac{\pi}{4}.
\]

Therefore
\begin{align*}
\mathcal V(1-s)  - R(1-s) &\sim \frac{L^2 \gamma ^2 \sigma^2}{2} s^{-1/2}\sum_{i=1}^\infty (L\gamma )^{-3/2} i^{-\alpha/2} \sim \frac{1}{2}(L\gamma )^{1/2} \sigma^2 s^{-1/2} \zeta(\alpha/2)\frac{\pi}{4}.
\end{align*}

Applying the Tauberian theorem with $\omega = 1/2$ yields the $T^{-1/2}$ rate (the $O(1/T)$ term is negligible):
\begin{align*}
V_T &\sim \frac{1}{\Gamma(1/2)} \left[ \frac{\pi}{8} L^{1/2} \gamma ^{1/2} \sigma^2 \zeta\left(\frac{\alpha}{2}\right) \right] T^{-1/2} = \frac{\sqrt{\pi}}{8} L^{1/2} \gamma ^{1/2} \sigma^2 \zeta\left(\frac{\alpha}{2}\right) T^{-1/2}.
\end{align*}

\end{proof}

\subsubsection{WSD Schedule (\cref{th:wsd})}
\label{app:wsd}
\begin{proof}

\begin{align*}
\Pi_{T,i}
&= \prod_{j=0}^{T-1}
   \left[1 - 2\frac{L}{i^\alpha}
   \!\left(\gamma\eta_j - \frac{\gamma^2\eta_j^2\operatorname{tr}(\Lambda)}{2}\right)\right] ,
   \\
&= \left(1 - \a_i + \b_i\right)^{(1-c)T+1}
   \prod_{j=(1-c)T+1}^{T-1}
   \left(1 - \a_i\frac{T-j}{cT} + \b_i\!\left(\frac{T-j}{cT}\right)^{\!2}\right).
\end{align*}

Set $\ell = j-(1-c)T$ (from $1$ to $cT-1$), so $\dfrac{T-j}{cT} = 1 - \dfrac{\ell}{cT}$.
Then set $p = cT - \ell$ (from $cT-1$ to $1$), so $1 - \dfrac{\ell}{cT} = \dfrac{p}{cT}$.
\[
\sum_{j=(1-c)T+1}^{T-1}
\log\!\left(1 - \a_i\frac{T-j}{cT}+\b_i\!\left(\frac{T-j}{cT}\right)^{\!2}\right)
= \sum_{p=1}^{cT-1}
  \log\!\left(1 - \a_i\frac{p}{cT}+\b_i\!\left(\frac{p}{cT}\right)^{\!2}\right).
\]
By Euler–-Maclaurin ($\sum_{p=1}^{cT-1}$ has boundary values $f(0)=0$, $f(1)=\log(1-\a_i+\b_i)$):
\[
\sum_{p=1}^{cT-1}
\log\!\left(1-\a_i\frac{p}{cT}+\b_i\!\left(\frac{p}{cT}\right)^{\!2}\right)
= cT\!\int_0^1\!\log\!\left(1-\a_ix+\b_ix^2\right)dx
   {- \frac{\log(1-\a_i+\b_i)}{2}+O\left(\frac{1}{i^\alpha T}\right)}.
\]
Then we can write:
\begin{align*}
\Pi_{T,i}
&= \s_{i,T}\exp\!\left[
  (1-c)T\log(1-\a_i+\b_i)
  + cT\!\int_0^1\!\log\!\left(1-\a_ix+\b_ix^2\right)dx
\right],
\end{align*}
with $\s_{i,T} = \sqrt{1-\a_i+\b_i}(1+O(\tfrac{1}{i^\alpha T}))$. Let $\s_i := \lim_{T\to\infty }\s_{i,T} = \sqrt{1-\a_i+\b_i}$.
\paragraph{Bias.}
The method is identical to the linear schedule, with
\[
I_i = (1-c)\log(1-\a_i+\b_i)
    + c\!\int_0^1\!\log\!\left(1-\a_ix+\b_ix^2\right)dx.
\]
For $i$ large, $I_i \sim -i^{-\alpha}\gamma  L\kappa(c)$
with $\kappa(c) = 2-c-\gamma \tr(\Lambda)(1-\frac{2c}{3})$,
so the Tauberian argument gives
\[
B_T \sim \frac{L\Delta}{2\alpha}
\frac{\Gamma\!\left(\frac{\beta-1}{\alpha}+1\right)}
{\left[\gamma  LT\!\left(2-c-\gamma \tr(\Lambda)\!\left(1-\frac{2c}{3}\right)\right)
\right]^{\frac{\beta-1}{\alpha}+1}}.
\]

\paragraph{Variance.}
\begin{align*}
V_T
&= \frac{1}{2}\sum_{i=1}^\infty \sum_{k=0}^{T-1}
   \frac{L^2}{i^{2\alpha}}\gamma^2\eta_k^2\sigma^2
   \prod_{m=k+1}^{T-1}\!\left[1-2\bar\nu_m\frac{L}{i^\alpha}\right]
= V_T^A + V_T^B,
\end{align*}
where $V_T^A$ (resp.\ $V_T^B$) sums over $k \le (1-c)T$ (constant phase)
and $k > (1-c)T$ (cooldown phase).

\subparagraph{Constant part: $V_T^A$.}
For the cooldown product from $m=(1-c)T+1$ to $T-1$,
set $\ell=T-m$ (from $cT-1$ to $1$); then $\eta_m = \ell/(cT)$
and the factor is $1-\a_i\ell/(cT)+\b_i(\ell/(cT))^2$. Let $\s_i^A := (1-\a_i+\b_i)^{-1/2}\xrightarrow{i\to\infty}1$. Hence

\begin{align*}
V_T^A
&= \frac{1}{2}\sum_{i=1}^\infty \sum_{k=0}^{(1-c)T}
   \frac{L^2\gamma ^2\sigma^2}{i^{2\alpha}}
   [1-\a_i+\b_i]^{(1-c)T-k}
   \exp\!\left[\sum_{\ell=1}^{cT-1}
   \log\!\left(1-\a_i\frac{\ell}{cT}+\b_i\!\left(\frac{\ell}{cT}\right)^{\!2}\right)\right],
\\
&\sim
\frac{1}{2}\sum_{i=1}^\infty \sum_{k=0}^{(1-c)T}
   \frac{L^2\gamma ^2\sigma^2}{i^{2\alpha}}
   [1-\a_i+\b_i]^{(1-c)T-k}
   \frac{1}{\sqrt{1-\a_i+\b_i}}
   \exp\!\left[cT\!\int_0^1\!\log\!\left(1-\a_ix+\b_ix^2\right)dx\right],
\\
&\sim
\frac{1}{2}\sum_{i=1}^\infty \frac{L^2\gamma ^2\sigma^2}{i^{2\alpha}}
   \frac{[1-\a_i+\b_i]^{(1-c)T+1}-1}{(1-\a_i+\b_i)-1}\,\s_i^A
   \exp\!\left[cT\!\int_0^1\!\log\!\left(1-\a_ix+\b_ix^2\right)dx\right].
\end{align*}
With $\a_i=\a i^{-\alpha}$, $\b_i=\b i^{-\alpha}$,
$J_i=\int_0^1\log(1-\a_iu+\b_iu^2)du = \q i^{-\alpha}+O(i^{-2\alpha})$ for $i\gg 1$,
\[
V_T^A \sim
\frac{L^2\gamma ^2\sigma^2}{2(\b-\a)}
\sum_{i=1}^\infty i^{-\alpha} \s_i^A
\left[(1-\a_i+\b_i)^{(1-c)T+1}e^{cJ_iT} - e^{cJ_iT}\right].
\]
The $z$-transform is
\begin{align*}
\mathcal{V}^A(z) = \frac{L^2\gamma ^2\sigma^2}{2(\b-\a)}
\sum_{i=1}^\infty i^{-\alpha} \s_i^A
&\left[
  \frac{1-\a_i+\b_i}{1-z(1-\a_i+\b_i)^{1-c}e^{cJ_i}}
  - \frac{1}{1-ze^{cJ_i}}
\right],
\\=
\frac{L^2\gamma ^2\sigma^2}{2(\b-\a)}
\sum_{i=1}^\infty i^{-\alpha} \s_i^A
&\bigg[
  \frac{1+(-\a+\b)i^{-\alpha}}{1-z(1-\a_i+\b_i)^{1-c}(1+c\q i^{-\alpha} +O(i^{-2\alpha}))}
  \\
  &- \frac{1}{1-z(1+c\q i^{-\alpha} +O(i^{-2\alpha}))}
\bigg].
\end{align*}
Using
$\log(1-\a_i+\b_i) = (\b-\a)i^{-\alpha}+O(i^{-2\alpha})$,
$\exp(cJ_i) = 1+c\q i^{-\alpha}+O(i^{-2\alpha})$,
the two poles are $\kappa_1 = (\a-\b)(1-c)-c\q $ and $\kappa_2=-c\q $.

\begin{align*}
\mathcal{V}^A(1-s) 
&=
\frac{L^2\gamma ^2\sigma^2}{2(\b-\a)}
\sum_{i=1}^\infty i^{-\alpha}
\left[
  \frac{1}{s+(1-s)\kappa_1i^{-\alpha} +O(i^{-2\alpha})}
  - \frac{1}{s+(1-s)\kappa_2i^{-\alpha} +O(i^{-2\alpha})}
\right],
\\&=
\frac{L^2\gamma ^2\sigma^2}{2(\b-\a)}
\sum_{i=1}^\infty x_i^{-\alpha} s
\left[
  \frac{1}{s+(1-s)\kappa_1 x_i^{-\alpha} s +O(x_i^{-2\alpha} s^{2})}
  -  \frac{1}{s+(1-s)\kappa_2 x_i^{-\alpha} s +O(x_i^{-2\alpha} s^{2})}
  \right],
\\&=
\frac{L^2\gamma ^2\sigma^2}{2(\b-\a)} s^{-1/\alpha}
\sum_{i=1}^\infty x_i^{-\alpha} 
\bigg[
  \frac{1}{1+(1-s)\kappa_1 x_i^{-\alpha}  +O(x_i^{-2\alpha} s)}
  \\& \qquad -  \frac{1}{1+(1-s)\kappa_2 x_i^{-\alpha}  +O(x_i^{-2\alpha} s)}
  \bigg] s^{1/\alpha}.
\end{align*}

With $x_i =is^{1/\alpha}$, $\Delta x = s^{1/\alpha}$.

By the DCT we can replace the sum by an integral over $[0,\infty)$:
\begin{align*}
\mathcal{V}^A(1-s)
&\sim \frac{L^2\gamma ^2\sigma^2}{2(\b-\a)} s^{-1/\alpha}
  \int_0^\infty x^{-\alpha}
  \!\left[\frac{1}{1+\kappa_1 x^{-\alpha}}
         -\frac{1}{1+\kappa_2 x^{-\alpha}}\right]dx,
\\&\sim
\frac{L^2\gamma ^2\sigma^2}{2\alpha(\b-\a)}
\,\Gamma\!\left(\tfrac{1}{\alpha}\right)\Gamma\!\left(1-\tfrac{1}{\alpha}\right)
\!\left[\kappa_1^{\frac{1-\alpha}{\alpha}}-\kappa_2^{\frac{1-\alpha}{\alpha}}\right]
s^{-1/\alpha}.
\end{align*}
and the Tauberian theorem with an exponent $1/\alpha$ yields
\[
V_T^A \sim
\frac{L^2\gamma ^2\sigma^2}{2\alpha(\b-\a)}
\,\Gamma\!\left(1-\tfrac{1}{\alpha}\right)
\!\left[\kappa_1^{\frac{1-\alpha}{\alpha}}-\kappa_2^{\frac{1-\alpha}{\alpha}}\right]
T^{\frac{1-\alpha}{\alpha}}.
\]
($\b-\a<0$ and $\kappa_1>\kappa_2>0$, so
$\kappa_1^{(1-\alpha)/\alpha}<\kappa_2^{(1-\alpha)/\alpha}$, so $V_T^A>0$).

\subparagraph{Cooldown part: $V_T^B$.}
Set $\ell = T-k$ (from $1$ to $cT-1$ as $k$ goes from $T-1$
to $(1-c)T+1$). With the schedule $\eta_k =  (T-k)/(cT)$,
one has $\eta_k =   \ell/(cT)$.
For the product from $m=k+1$ to $T-1$,
set $j=T-m$ (from $\ell-1$ to $1$); then $\eta_m=  j/(cT)$
and the factor is $1-\a_ij/(cT)+\b_i(j/(cT))^2$. Hence
\begin{align*}
V_T^B
&= \frac{1}{2}\sum_{i=1}^\infty \sum_{\ell=1}^{cT-1}
   \frac{L^2\sigma^2}{i^{2\alpha}}
   \!\left(\gamma \frac{\ell}{cT}\right)^{\!2}
   \exp\!\left[\sum_{j=1}^{\ell-1}
   \log\!\left(1-\a_i\frac{j}{cT}+\b_i\!\left(\frac{j}{cT}\right)^{\!2}\right)\right],
\\
&\sim
\frac{1}{2}\sum_{i=1}^\infty \sum_{\ell=1}^{cT-1}
   \frac{L^2\sigma^2\gamma ^2}{i^{2\alpha}}
   \!\left(\frac{\ell}{cT}\right)^{\!2}
   \frac{1}{\sqrt{1-\a_i\frac{\ell}{cT}+\b_i\!\left(\frac{\ell}{cT}\right)^{\!2}}}
   \exp\!\left[cT\!\int_0^{\frac{\ell}{cT}}
   \log\!\left(1-\a_ix+\b_ix^2\right)dx\right],
\\
&\sim
\frac{1}{2}\sum_{i=1}^\infty cT\!
\int_0^1\frac{L^2\sigma^2\gamma ^2}{i^{2\alpha}}\,u^2
   \exp\!\left[cT\!\int_0^u\log\!\left(1-\a_ix+\b_ix^2\right)dx\right]du
\;\sim\; V_{cT}^{\mathrm{linear}}.
\end{align*}
By the computations of \cref{app:linear} with $T \leftarrow cT$:

\paragraph{Case $1 < \alpha < 2$.}
Both $V_T^A$ and $V_T^B$ are $\Theta(T^{(1-\alpha)/\alpha})$:
\begin{align*}
V_T &\sim
\frac{L^2\gamma ^2\sigma^2}{2\alpha(\b-\a)}
\Gamma\!\left(\frac{\alpha-1}{\alpha}\right)
\!\left[\kappa_1^{\frac{1-\alpha}{\alpha}}-\kappa_2^{\frac{1-\alpha}{\alpha}}\right]
T^{\frac{1-\alpha}{\alpha}}
+
\frac{L^2\gamma ^2\sigma^2}{2\alpha}
\Gamma\!\left(2-\frac{1}{\alpha}\right)
(\gamma  L)^{\frac{1-2\alpha}{\alpha}}
\mathcal{I}_{\alpha,\gamma }\,(cT)^{\frac{1-\alpha}{\alpha}}.
\end{align*}

\paragraph{Case $\alpha > 2$.}
$V_T^A = O(T^{(1-\alpha)/\alpha}) = o(T^{-1/2})$,
so $V_T^A \ll V_T^B$ and
\[
V_T \sim V_{cT}^{\mathrm{linear}}
\sim \frac{\sqrt{\pi}}{8}\,L^{1/2}\gamma ^{1/2}\sigma^2\,
\zeta\!\left(\frac{\alpha}{2}\right)(cT)^{-1/2}.
\]

\end{proof}

\subsection{Optimal Base Learning-Rate}
\label{sec:optimal_gamma}
We want to calculate the optimal learning rate $\widetilde \gamma ^\star$ that minimizes the total asymptotic risk. We consider the functions $\gamma  \mapsto B_T(\gamma ), \gamma  \mapsto \widetilde B_T(\gamma )$, $\gamma  \mapsto V_T(\gamma )$ and $\gamma  \mapsto \widetilde V_T(\gamma )$.

\subsubsection{Proof of \cref{lemma:eta_star_to_0}}
\label{proof:eta_star_to_0}

\begin{proof}\leavevmode

\paragraph{Item 1.}
Write $\widetilde V_T(\gamma )= z(\gamma )\,T^{-r}$, where $z$ is continuous and strictly
positive on the admissible interval and $r\geq0$ does not depend on $\gamma $.
We always have $r<\omega$ (indeed $\omega-r = \omega>0$ for constant, 
$\omega-r=\frac{\beta}{\alpha}>0$ for WSD $\alpha<2$, and
$\omega-r=\frac12+\frac{\beta-1}{\alpha}>0$ for WSD $\alpha>2$).

Fix $\epsilon>0$. On the interval
$[\epsilon,\gammax)$, the function $z$ attains a minimum $z_\epsilon>0$ since $z(\gamma)\xrightarrow{\gamma\to\gammax}\infty$ for constant and $\lim_{\gamma\to\gammax}z(\gamma) >0$ for WSD, so
\[
    \forall \gamma  \in [\epsilon, \gammax),\quad  \widetilde{\mathcal R}_T(\gamma )\ \geq \widetilde V_T(\gamma )
    \ \ge\ z_\epsilon T^{-r}.
\]
Choose $x_T=T^{-\w}$ with
$\w\in\big(0,\,1-\tfrac{r}{\omega}\big)$ (a nonempty interval since $r<\omega$). Then
$x_T\to0$, hence $x_T\tr(\Lambda)\to0$, and since $z(\gamma )\xrightarrow{\gamma \to0}0$, $z(x_T) =o(1)$.
The simplified equivalents apply:
\[
    \widetilde{\mathcal R}_T(x_T)
    = O\big(x_T^{-\omega}T^{-\omega}\big)+z(x_T)T^{-r}
    = O\big(T^{-\omega(1-\w)}\big)+o\big(T^{-r}\big)
    = o\big(T^{-r}\big),
\]
both exponents being $>r$ by the choice of $\w$.

As a result, for $T$ large enough,
$\displaystyle \min_{\gamma }\widetilde{\mathcal R}_T(\gamma )\le \widetilde{\mathcal R}_T(x_T)
< \tfrac{z_\epsilon}{2}T^{-r}\le \inf_{\gamma \ge\epsilon}\widetilde{\mathcal R}_T(\gamma )$,
so every minimizer lies in $(0,\epsilon)$. Since $\epsilon>0$ is arbitrary, a minimizer of the asymptotic risk satisfies
$\gamma\to0$ as $T \to \infty$.

\paragraph{Item 2.}
We have $\gamma\tr(\Lambda)\to0$, so the correction
factors in the equivalents of \cref{th:constant,th:wsd} satisfy, near any minimizer $\gamma$ of $\eqrisk$, which tends to $0$ by item 1,
$1-\gamma \frac{\tr(\Lambda)}{3}=1+o(1)$ and $2-c-\gamma \tr(\Lambda)(1-\tfrac{2c}{3})=2-c+o(1)$.

The leading-order risk is obtained by freezing these factors at their $\gamma \to0$ limits:
\[
    \frrisk_T(\gamma )=C_B\,\gamma ^{-\omega}\,T^{-\omega}+C_V\,\gamma ^{v}\,T^{-r},
    \qquad C_B,C_V>0,
\]
where the exponents are
\[
\begin{cases}
   v=1 \text{ and } r=0 & \text{for constant schedule},\\
    v=\frac{1}{\alpha} \text{ and } r=1 - \frac{1}{\alpha} & \text{for WSD schedule with } 1 < \alpha <2,\\
     v=r=\frac{1}{2}  & \text{for WSD schedule with } \alpha>2.
\end{cases}
\]

Let us now minimize $\gamma \mapsto \frrisk_T(\gamma)$.
It is a smooth function with a unique interior critical point, which gives
\[
    \frac{d}{d\gamma }\frrisk_T(\gamma )=0
    \iff \gamma ^{v+\omega}=\frac{\omega\,C_B}{v\,C_V}\,T^{r-\omega}
    \iff \gammastar_T=\Big(\frac{\omega\,C_B}{v\,C_V}\Big)^{\frac1{v+\omega}}T^{\frac{r-\omega}{v+\omega}}.
\]

So $\gammastar_T = \Theta(T^{-\vartheta})$ where $\vartheta = -\frac{r-\omega}{v+\omega} = \frac{\omega-r}{v+\omega}$.

\paragraph{Item 3.}
Comparing
\cref{th:constant,th:wsd} with $\frrisk_T$ term by term,
\[
  \eqrisk_T(\gamma) = \varphi_B(\gamma)\,C_B\,\gamma^{-\omega}T^{-\omega}
                    + \varphi_V(\gamma)\,C_V\,\gamma^{v}T^{-r},
\]
where $\varphi_B,\varphi_V$ collect the correction factors
($1-\gamma \frac{\tr(\Lambda)}{3}, \kappa(c,\gamma)/(2-c)$, $\kappa_1$, $\kappa_2$, $\mathcal I_{\alpha,\gamma})$, are independent of $T$ and tend to $1$ as $\gamma \to 0$. Therefore, by positivity of all terms for sufficiently large $T$,
\[
  \min\{\varphi_B(\gamma),\varphi_V(\gamma)\}  \frrisk_T(\gamma)
  \leq \eqrisk_T(\gamma)
  \leq \max\{\varphi_B(\gamma),\varphi_V(\gamma)\}  \frrisk_T(\gamma).
\]

Moreover $\lim_{\gamma\to 0}\min\{\varphi_B(\gamma),\varphi_V(\gamma)\}=1$ and $\lim_{\gamma \to0}\max\{\varphi_B(\gamma),\varphi_V(\gamma)\}=1$.
Write $\m_T:=\inf_\gamma\eqrisk_T(\gamma)$ and
$\frm_T:=\inf_\gamma\frrisk_T(\gamma)=\frrisk_T(\gammastar_T)$.
Let $\gamma_T$ be a minimizer of $\eqrisk_T$. Then $\gamma_T\to0$ by 1., while by
2., $\gammastar_T\to0$. Hence as $T \to \infty$,
\begin{align}
\label{eq:min_eqrisk}
  &\m_T \leq \eqrisk_T(\gammastar_T)
      \leq \max\{\varphi_B(\gammastar
      _T),\varphi_V(\gammastar_T)\} \cdot\frrisk_T(\gammastar_T)
      = \frm_T\,(1+o(1)),
      \\
  &\frm_T  \leq \frrisk_T(\gamma_T)
      \leq \frac{\eqrisk_T(\gamma_T)}{\min\{\varphi_B(\gamma_T),\varphi_V(\gamma_T)\}}  
      = \m_T\,(1+o(1)).
\end{align}
Therefore $\m_T\sim \frm_T$, and combining with \cref{eq:min_eqrisk},
$\eqrisk_T(\gammastar_T)\sim \frm_T\sim \m_T$, which is \cref{lemma:eta_star_to_0}, part 3.

\end{proof}

\subsubsection{Proof of \cref{th:gammastar}}
\label{sec:gammastar}

\begin{lemma}[Product Error Bound]\label{lemma:error_r}
Let $(\gamma_T)_{T \ge 1}$ and $(\tau_T)_{T \ge 1}$ be two sequences of positive real numbers such that $\lim_{T\to \infty} \gamma_T = 0$ and $\lim_{T\to \infty} \tau_T =  \infty$.

Consider a family of analytical terms $(\Pi_{T,i})_{T,i}$ that can be decomposed as $\Pi_{T,i} = S_{T,i} \exp\left(-\tau_T i^{-\alpha} + X_{T,i}\right)$.
Assume there exist constants $C_0, C_1 > 0$ such that, for $T$ sufficiently large, the normalization factor and the exponent remainder satisfy uniformly for all $i \ge 1$:
\[
|S_{T,i}- 1| \le C_0 \gamma_T i^{-\alpha} \quad \text{and} \quad |X_{T,i}| \le C_1 \gamma_T \tau_T i^{-\alpha} 
\]

If we express this term as $\Pi_{T,i} = (1+r_{T,i})e^{-\tau_T i^{-\alpha}}$, then there exists a constant $C > 0$ such that, for $T$ large enough to ensure $C_1 \gamma_T \le 1/2$, the residual $r_{T,i}$ is uniformly bounded by:
\begin{equation}\label{eq:uniform_r}
   |r_{T,i}| \le C \gamma_T (1+\tau_T i^{-\alpha}) e^{\frac{1}{2}\tau_T i^{-\alpha}}
\end{equation}
\end{lemma}

\begin{proof}
We use the elementary bound $|e^x-1|\le|x|e^{|x|}$, valid for every $x\in\R$.
Splitting
\[
   r_{T,i}=(S_{T,i}-1)\,e^{X_{T,i}}+\big(e^{X_{T,i}}-1\big)
\]
and applying the hypotheses,
\[
   |r_{T,i}|
   \le\Big(C_0\,\gamma_T\,i^{-\alpha}+C_1\,\gamma_T\,\tau_T\, i^{-\alpha}\Big)\,
      e^{C_1\gamma_T\tau_T i^{-\alpha}} .
\]
Since $\gamma_T \xrightarrow{T\to\infty}0$, we may further restrict to $T$ large
enough that $C_1\gamma_T\leq\tfrac12$; then $C_1\gamma_T\tau_T i^{-\alpha}\le\tfrac12\tau_T i^{-\alpha}$
for every $i\ge1$, and with $i^{-\alpha}\le1$ and $C:=C_0+C_1$ we obtain, uniformly in
$i\ge1$,
\begin{equation*}
   |r_{T,i}|\le C\,\gamma_T\,(1+\tau_T i^{-\alpha})\,e^{\frac12\tau_T i^{-\alpha}}.
\end{equation*}
\end{proof}

Let us define $\epsilon_T(\gamma) := \risk_T(\gamma)/\eqrisk_T(\gamma)-1$ so that $\mathcal R_T(\gamma) = (1+\epsilon_T(\gamma)) \widetilde {\mathcal R}_T$. \cref{th:wsd} is then equivalent to $\forall \gamma \in (0, \gammax),\; \epsilon_T(\gamma) \xrightarrow{T\to\infty} 0$. Pointwise convergence at a fixed $\gamma$ is not enough, since $\gammastar_T$ is not fixed. We need convergence along particular sequences $(\gamma_T)$ so that we can have $\mathcal R_T (\gammastar_T)\sim \widetildeR_T(\gammastar_T)$.

\begin{lemma}[Equivalence of Risks]
\label{lemma:uniform_epsilon}
Let $(\gamma_T)_T$ be a sequence in $(0, \gammax)$. Assume $\gamma_T \xrightarrow{T \to \infty}0$ and $T\gamma_T \xrightarrow{T \to \infty}\infty$. Considering $\epsilon_T$ defined above,
\[
\epsilon_T(\gamma_T) \xrightarrow{T \to \infty}0, \quad \text{which is equivalent to} \quad \risk_T(\gamma_T) \sim \eqrisk_T(\gamma_T).
\]
\end{lemma}

\begin{proof}
Since $B_T,V_T,\widetilde B_T,\widetilde V_T>0$, the ratio
$\mathcal R_T/\widetilde{\mathcal R}_T$ lies between $B_T/\widetilde B_T$ and
$V_T/\widetilde V_T$. It therefore suffices to bound each of these two ratios.
Let $\tau:=\gamma_T L(2-c)T$. 
\[
   \gamma_T\xrightarrow[T\to\infty]{}0
   \qquad\text{and}\qquad
   \tau=(2-c)L\,\gamma_T\,T\xrightarrow[T\to\infty]{}\infty.
\]

From the proof of \cref{th:wsd}, the per-coordinate contraction factor is
$\Pi_{T,i}=\s_{i,T}\,e^{TI_i}$, where
\[
   \s_{i,T}=1+O(\gamma_T i^{-\alpha}),
   \qquad
   I_i=-\,i^{-\alpha}\gamma_T L\,\kappa(c)+O(\gamma_T^2 i^{-2\alpha}),
   \qquad
   \kappa(c)=2-c-\gamma_T\tr(\Lambda)\big(1-\tfrac{2c}{3}\big).
\]
Both $\a_i=2\gamma_T Li^{-\alpha}$ and $\b_i=\gamma_T^2\tr(\Lambda)Li^{-\alpha}$ are
$O(\gamma_T i^{-\alpha})$ uniformly in $i\ge1$; since
$\gamma_T\to0$, so that for $T$ large enough these quantities are below the radius of convergence of the
expansions of $\log(1-\a_i+\b_i)$ and of $\int_0^1\log(1-\a_i u+\b_i u^2)\,du$, and the
two $O(\cdot)$ above are uniform in $i\ge1$. Fix $C_0,C_1>0$ such that, for such $T$,
\begin{equation}\label{eq:uniform_OO}
   |\s_{i,T}-1|\le C_0\,\gamma_T\, i^{-\alpha},
   \qquad
   |X_i|\le C_1\,\gamma_T\,\tau\, i^{-\alpha},
   \qquad X_i:=TI_i+\tau i^{-\alpha},
\end{equation}
The second bound follows from
$X_i=-i^{-\alpha}\gamma_T LT\,\big(\kappa(c)-(2-c)\big)+O(\gamma_T^2Ti^{-2\alpha})$ together with
$\kappa(c)-(2-c)=-\gamma_T\tr(\Lambda)(1-\tfrac{2c}{3})=O(\gamma_T)$ and
$\gamma_T^2Ti^{-2\alpha}=O(\gamma_T\tau i^{-\alpha})$.

We compare $\Pi_{T,i}$ to $e^{-\tau i^{-\alpha}}$, i.e.\ we replace $\kappa(c)$ by its
$\gamma_T\to0$ limit $2-c$, and define $r_{T,i}$ by
\[
   \Pi_{T,i}=(1+r_{T,i})\,e^{-\tau i^{-\alpha}},
   \qquad\text{so that}\qquad
   1+r_{T,i}=\s_{i,T}\,e^{X_i}.
\]
Thus \cref{lemma:error_r} gives:
\begin{equation*}
   |r_{T,i}|\le C\,\gamma_T\,(1+\tau i^{-\alpha})\,e^{\frac12\tau i^{-\alpha}}.
\end{equation*}
The three ingredients of $r_{T,i}$ are thus: the normalisation $\s_{i,T}-1$, the
Euler--Maclaurin remainder hidden in $I_i$, and the freezing
$\kappa(c)\rightarrow 2-c$.

\paragraph{Bias.}
$B_T(\gamma_T)=\frac{L\Delta}{2}\sum_{i=1}^\infty i^{-(\alpha+\beta)}\Pi_{T,i} = \frac{L\Delta}{2}\sum_{i=1}^\infty i^{-(\alpha+\beta)}(1+r_{T,i})e^{-\tau i^{-\alpha}}$.
Using \cref{eq:uniform_r} of \cref{lemma:error_r}, 
\[
B_T(\gamma_T) = \frac{L\Delta}{2} \left(S_1(\tau) +S_2(\tau)\right)\quad\text{where } S_1(\tau):=\sum_{i=1}^\infty i^{-(\alpha+\beta)}e^{-\tau i^{-\alpha}} \text{ and } S_2(\tau):=\sum_{i=1}^\infty r_{T,i} i^{-(\alpha+\beta)}e^{-\tau i^{-\alpha}}.
\] 

Applying Euler--Maclaurin to $x\mapsto x^{-(\alpha+\beta)}e^{-\tau x^{-\alpha}}$
(with the substitution $u=\tau x^{-\alpha}$, $\mathrm dx=\tfrac1\alpha\tau^{1/\alpha}u^{-1/\alpha-1}\mathrm du$),
\[
   S_1(\tau)=\int_0^\infty x^{-(\alpha+\beta)}e^{-\tau x^{-\alpha}} dx\,(1+o(1))
   =\frac{\tau^{-\omega}}{\alpha}\int_0^\infty u^{\omega-1}e^{-u}
   du\,(1+o(1))
   =\frac{\Gamma(\omega)}{\alpha}\,\tau^{-\omega}\,\bigl(1+\delta(\tau)\bigr),
\]
with $\omega=\tfrac{\beta-1}{\alpha}+1$ and $\delta(\tau)\xrightarrow{\tau\to\infty}0$
the (sum $\to$ integral) discrepancy. This is exactly the $\gamma$-free
content of the Tauberian step of \cref{th:wsd}, read here as a plain Laplace
asymptotic. In particular $\frac{\alpha\,\tau^{\omega}}{\Gamma(\omega)}S_1(\tau)=1+\delta(\tau)$. Note that extending the lower bound of the continuous integral from $1$ to $0$ introduces a residual error $\int_0^1 x^{-(\alpha+\beta)}e^{-\tau x^{-\alpha}}\,\mathrm dx$. Because of the exponential term, this residual decays exponentially as $O(e^{-\tau})$ and is therefore safely absorbed into the $1+o(1)$ asymptotic equivalence. 

The correction to $S_1(\tau)$ is $S_2(\tau)=\sum_{i=1}^{\infty}i^{-(\alpha+\beta)}e^{-\tau i^{-\alpha}}r_{T,i}$.
Using \cref{eq:uniform_r}, the factor $\gamma_T$ pulls out and it remains to compare
\[
   \sum_{i=1}^\infty  i^{-(\alpha+\beta)}(1+\tau i^{-\alpha})\,e^{-\frac12\tau i^{-\alpha}} \quad \text{to } S_1(\tau).
\]
Applying Euler--MacLaurin again, the whole sum is still $O(\tau^{-\omega})$.
Therefore
\[
   \frac{\alpha\,\tau^{\omega}}{\Gamma(\omega)}S_2(\tau) = \frac{\alpha\,\tau^{\omega}}{\Gamma(\omega)}\sum_{i=1}^\infty \frac{e^{-\tau i^{-\alpha}}}{i^{\alpha+\beta}}\,r_{T,i}=O(\gamma_T).
\]
Combining $S_1(\tau)$ and $S_2(\tau)$,
$\frac{\alpha\tau^\omega}{\Gamma(\omega)}\sum_i i^{-(\alpha+\beta)}\Pi_{T,i}=1+\delta(\tau)+O(\gamma_T)$,
so $B_T(\gamma_T)=\frac{L\Delta}{2\alpha}\Gamma(\omega)\tau^{-\omega}(1+\delta(\tau)+O(\gamma_T))$.
On the other hand \cref{th:wsd} gives
$\widetilde B_T(\gamma_T)=\frac{L\Delta}{2\alpha}\Gamma(\omega)\,[\gamma_T L\kappa(c)T]^{-\omega}
=\frac{L\Delta}{2\alpha}\Gamma(\omega)\,\tau^{-\omega}\,(1+O(\gamma_T))$,
the last equality because $[\kappa(c)/(2-c)]^{-\omega}=1+O(\gamma_T)$. Dividing,
\[
   \frac{B_T(\gamma_T)}{\widetilde B_T(\gamma_T)}=1+O(\gamma_T)+\delta(\tau).
\]

\paragraph{Variance.}
$V_T = \frac{1}{2}\sum_{i=1}^\infty \sum_{k=0}^{T-1}
   \frac{L^2}{i^{2\alpha}}\gamma_T^2\eta_k^2\sigma^2
   \prod_{m=k+1}^{T-1}\!\left[1-2\bar\nu_m\frac{L}{i^\alpha}\right] = V^A_T+V^B_T,$
where 
\begin{itemize}

        \item Regarding the constant part $V_T^A$,
        \begin{align*}
        V_T^A
        &= \frac{1}{2}\sum_{i=1}^\infty \sum_{k=0}^{(1-c)T}
           \frac{L^2\gamma_T ^2\sigma^2}{i^{2\alpha}}
           [1-\a_i+\b_i]^{(1-c)T-k}
           \exp\!\left[\sum_{\ell=1}^{cT-1}
           \log\!\left(1-\a_i\frac{\ell}{cT}+\b_i\!\left(\frac{\ell}{cT}\right)^{\!2}\right)\right],
        \\&= \frac{1}{2}\sum_{i=1}^\infty \sum_{k=0}^{(1-c)T}
           \frac{L^2\gamma_T ^2\sigma^2}{i^{2\alpha}}
           (\c_i)^{(1-c)T-k}(1+r^A_{T,i})\,e^{-\kappa_2 Ti^{-\alpha}}, \quad (\text{$\c_i$ consists in two geometric sums)}
           \\&= 
           \frac{L\gamma_T ^2\sigma^2}{4\bar\gamma_T}\sum_{i=1}^\infty  i^{-\alpha}
           (1+r^A_{T,i}) \left[e^{-\kappa_2 Ti^{-\alpha}} - e^{-\kappa_1 Ti^{-\alpha}} \right] = S_1 + S_2,
        \end{align*}
        where $r^A_{T,i}$ satisfies $|r^A_{T,i}| \leq C\gamma_T (1+\kappa_2 T i^{-\alpha})e^{\frac12 \kappa_2 T i^{-\alpha}} $ (\cref{lemma:error_r}) with $\tau$ replaced by $\kappa_2 T \leq \tau$; and $S_2$ the term containing $r_{T,i}^A$.

        By Euler-Maclaurin,
        $\sum_{i=1}^{\infty} i^{-\alpha} e^{-\kappa T i^{-\alpha}} = \frac{\Gamma\left(1 - \frac{1}{\alpha}\right)}{\alpha} (\kappa T)^{\frac{1-\alpha}{\alpha}} (1 + \delta'(T))$.
        We have $S_1= \widetilde V_T^A (1+\delta'(T))$ and $S_2 = O(\gamma_T)S_1$, so $\frac{V_T^A}{\widetilde V_T^A} = 1+O(\gamma_T) +\delta'(T)$.

        \item For $\alpha<2$, regarding the cooldown part $V_T^B$, denoting $x=\frac{\ell}{cT}$,
        \begin{align*}
        V_T^B
        &= \frac{1}{2}\sum_{i=1}^\infty \sum_{\ell=1}^{cT-1}
        \frac{L^2\sigma^2}{i^{2\alpha}}
           \!\left(\gamma_T \frac{\ell}{cT}\right)^{\!2}
           \exp\!\left[\sum_{j=1}^{\ell-1}
           \log\!\left(1-\a_i\frac{j}{cT}+\b_i\!\left(\frac{j}{cT}\right)^{\!2}\right)\right],
        \\&=
        \frac{L^2\sigma^2\gamma_T^2}{2}\sum_{i=1}^\infty i^{-2\alpha}\sum_{\ell=1}^{cT-1}
        x^2
        (1+r^B_{T,i,\ell})\exp\bigg[{-cTi^{-\alpha}\underbrace{\gamma_TLx^2\left(1-\frac{\gamma_T\tr(\Lambda)x}{3}\right)}_{=:\psi_T(x)}}\bigg],
        \end{align*}
        where $r^B_{T,i,\ell}$ obeys the bound of \cref{lemma:error_r} with $\tau$ replaced by $cT\psi_T(x)$. The bound stays uniform as $cT\psi_T(x) \leq \kappa_2T\leq \tau$. \cref{as:eta} gives $\gamma_T\tr(\Lambda) <2$ so $\frac13\gamma_TLx^2 \leq \psi_T(x)\leq \gamma_TLx^2$ for $x \in [0,1]$. Moreover,
        \[
        \sum_{\ell=1}^{cT-1}
        x^2e^{-cTi^{-\alpha}\psi_T(x)}=cT\int_0^1 x^2 e^{-cTi^{-\alpha}\psi_T(x)}dx \cdot (1+O(1/T))).
        \]
        The Riemann sum over $i$ is made possible with a step size of $(\gamma_T L cT)^{-1/\alpha}$ that tends to 0 since $\gamma_TT \to \infty.$ The integral converges if $\alpha<2$, assumed here.
        By the same decomposition $(S_1,S_2)$ with $S_2 = O(\gamma_T)S_1$, we get $\frac{V^B_T}{\widetilde V^B_T}=1+O(\gamma_T) + \delta_\alpha''(T)$.

        \item For $\alpha>2$, reading the proof of \cref{th:wsd} ($c=1$) with $\gamma_T$ explicit and
        keeping the finite upper limit $s^{-1/2}$, the substitution $u=s^{1/2}w$ gives the
        per-mode contribution
        \[
           v^B_i(T)=\tfrac{\sqrt\pi}{8}\sigma^2 L^{1/2}\gamma_T^{1/2}T^{-1/2}\,i^{-\alpha/2}\,F(\theta_i)\,(1+O(\gamma_T)+O(1/T)),
           \qquad \theta_i:=\gamma_T L\,i^{-\alpha}T,
        \]
        where $F(\theta):=\tfrac{4}{\sqrt\pi}\int_0^{\sqrt\theta}y^2e^{-y^2}\mathrm dy$ increases from
        $F(0)=0$ to $\lim_{\theta \to \infty}F(\theta)=1$, and the $o(1)$ (Euler--Maclaurin) is uniform. Since
        $\widetilde V^B_T=\tfrac{\sqrt\pi}{8}\sigma^2 L^{1/2}\gamma_T^{1/2}\zeta(\tfrac\alpha2)T^{-1/2}$
        replaces every $F(\theta_i)$ by $1$,
        \[
           \frac{V^B_T}{\widetilde V^B_T}
           =1-\underbrace{\frac{1}{\zeta(\alpha/2)}\sum_{i\ge1}i^{-\alpha/2}\bigl(1-F(\theta_i)\bigr)+O(\gamma_T)+O(1/T)}_{=:-\delta''(\gamma_T,T)}.
        \]
        \[
           |\delta''(\gamma_T,T)|
           =\frac{1}{\zeta(\alpha/2)}\sum_{i\ge1}i^{-\alpha/2}\bigl(1-F(\theta_i)\bigr) +o(1).
        \]
        Since $\gamma_TT\to\infty$, for each fixed $i$,
        $1-F(\theta_i)\to0$. Dominated by the summable $i^{-\alpha/2}$ ($\alpha/2>1$),
        DCT gives $|\delta''(\gamma_T,T)|\to0$. Denote $|\delta''(\gamma_T,T)|$ by $\delta_\alpha''(T).$
    \end{itemize}

Hence $V_T/\widetilde V_T=1+O(\gamma_T)+\delta'(T)+\delta_\alpha''(T)$ with $\delta'(T)$ and $\delta_\alpha''(T)$ that tend to 0 as $T\to \infty$.

\paragraph{Conclusion.}
Collecting both ratios,
\[
   \epsilon_T(\gamma_T)=
   O(\gamma_T) + o(1).
\]
Hence, $\epsilon_T(\gamma_T)\to 0$.

\textbf{Constant schedule.} The bias is treated as above with $\tau := 2\bar\gamma_T L T$, since $\Pi_{T,i} = (1 - 2\bar\gamma_T L i^{-\alpha})^T$. For the variance, $V_T/V_\infty = 1 - \zeta(\alpha)^{-1}\sum_{i} i^{-\alpha}(1-2\bar\gamma_T L i^{-\alpha})^T = 1 - O\big((\gamma_T T)^{1/\alpha - 1}\big) \to 1$ since $\gamma_T T \to \infty$.
\end{proof}
We know that $\gammastar_T = \Theta(T^{-\vartheta})$, $\vartheta \in (0,1)$, so $T\gammastar_T =\Theta(T\cdot T^{-\vartheta}) \xrightarrow{T\to \infty} \infty$ and $\gammastar_T = \Theta(T^{-\vartheta} )\xrightarrow{T\to \infty} 0$. \cref{lemma:uniform_epsilon} gives $\epsilon_T(\gammastar_T) \xrightarrow{T\to \infty} 0$.
Therefore we have the following:
\begin{corollary}\label{cor:epsilon}
With $(\gammastar_T)$ defined as above,
\[
\risk_T(\gammastar_T) \sim \eqrisk_T(\gammastar_T), \quad T \to \infty.
\]
\end{corollary}

We need to prove that these assumptions are also valid for a minimizing sequence of base learning rates.

\begin{lemma}\label{lemma:bounds}
Let $(\gamma_T)$ be a sequence of base learning rates such that $\risk_T(\gamma_T) = O(\risk_T(\gammastar_T)), \; T\to \infty$, and $(\gammastar_T)$ the minimizer of $\gamma \mapsto \frrisk_T(\gamma)$. Then $\gamma_T$ satisfies:
\[\gamma_T \to 0 \quad \text{and} \quad \gamma_T T \to \infty.\]
\end{lemma} 

\begin{proof}
We recall the definitions:
\[\left\{\begin{aligned}
\bar\nu_k &:= \frac{2\gamma_T\eta_k - \gamma_T^2\eta_k^2\tr(\Lambda)}{2} = \gamma_T\eta_k - \frac{1}{2}\tr(\Lambda)\gamma_T^2\eta_k^2, \\
\Pi_{T,i} &:= \prod_{k=0}^{T-1} \left[ 1 - 2\bar\nu_k \frac{L}{i^\alpha} \right].
\end{aligned}\right.\]

Knowing that $\forall x \in[0,1/2], \; -2x \leq \log(1-x)\leq -x$, we can bound $\Pi_{T,i}=\exp(\sum_k\log(1-2\bar\nu_k L i^{-\alpha}))$ for $i$ sufficiently large so that $\forall k, \,2\bar\nu_k L i^{-\alpha } \leq 1/2$. We denote $i_{\min}$ the index from which it is satisfied. It is independent of $T$ and $\gamma_T$ since $\forall k, 2L\bar\nu_k \leq \frac{L}{\tr(\Lambda)} = \frac{1}{\zeta(\alpha)}\leq1$.
\begin{equation}\label{eq:bounds_Pi}
\exp\left( \sum_{k=0}^{T-1} -4Li^{-\alpha}\bar\nu_k \right) \leq \Pi_{T,i} \leq \exp\left( \sum_{k=0}^{T-1} -2Li^{-\alpha}\bar\nu_k  \right) .
\end{equation}

\paragraph{Constant schedule.}

\begin{itemize}
\item By \cref{eq:bounds_Pi} We have 
\[B_T \geq  \frac12\sum_{i\geq i_{\min}} L\Delta i^{-\alpha-\beta} \exp(-4Li^{-\alpha} T\bar\gamma_T),
\]with $\bar\gamma_T = \bar\nu_0$. Therefore, to obtain $B_T\to 0$, we must have $T\bar\gamma_T \to \infty$. Since the term $- \frac{1}{2}\tr(\Lambda)\gamma_T^2\eta_k^2$ is negative, this requires $T\gamma_T \to \infty$.
\item For the variance, if we keep only the last term of the sum over $k$, with $\prod_k (1-2\bar\nu_k)$ equal to 1,
\[
\displaystyle V_T \geq \frac12 \sum_{i\geq i_{\min}}  L^2 i^{-2\alpha} \gamma_T^2 \sigma^2. 
\]
Therefore we must have $\gamma_T \to 0$.

\end{itemize}

\paragraph{WSD schedule.}

First, let $\ell = T-k$.

If $\displaystyle \ell \geq cT, \quad \sum_{j=1}^{\ell-1}\eta_{j+k} = \underbrace{\frac{cT}{2}}_{\text{decay part}} + \underbrace{(\ell - cT)}_{\text{constant part}}$. If $\ell < cT$ there is only a decay part so $cT/2$ is an upper bound. Moreover, $\frac{\ell^2}{2cT}-(\ell- \frac{cT}{2})= \frac{cT}{2}(\frac{\ell}{cT}-1)^2 \geq 0$. Therefore:
\begin{equation}\label{eq:slope}
\sum_{j=k+1}^{T-1}\eta_j = \sum_{j=1}^{\ell-1}\eta_{j+k} \leq \frac{\ell^2}{2cT}.
\end{equation}

\begin{itemize}
\item  Let us show that $\gamma_T T \to \infty$.

We have $\bar\nu_k \leq \gamma_T \eta_k$ and $\sum_{k=0}^{T-1}\eta_k \leq T$. Since $\exp$ is increasing, \cref{eq:bounds_Pi} gives
\[B_T \geq  \frac12\sum_{i\geq i_{\min}} L\Delta i^{-\alpha-\beta} \exp\left(-4Li^{-\alpha} \sum_{k=0}^{T-1}\bar\nu_k\right) \geq \frac12\sum_{i\geq i_{\min}} L\Delta i^{-\alpha-\beta} \exp\left(-4Li^{-\alpha} T \gamma_T\right),
\]
So, to have $B_T\to 0$, we find the same conclusion $T\gamma_T \to \infty.$

\item Let us show that $\gamma_T \to 0$. Suppose first $\alpha<2$.
Using \cref{eq:slope},
\begin{align*}
\displaystyle V_T &\geq \frac12 \sum_{i\geq i_{\min}}  L^2 i^{-2\alpha} \gamma_T^2 \sigma^2\sum_{k=0}^{T-1}\eta_k^2 \prod_{j=k+1}^{T-1}[1-2\bar\nu_j],
\\
&\stackrel{\labelcref{eq:bounds_Pi}}{\geq}
\frac12 \sum_{i\geq i_{\min}}  L^2 i^{-2\alpha} \gamma_T^2 \sigma^2\sum_{k=0}^{T-1}\eta_k^2 \exp\left( \sum_{j=k+1}^{T-1} -4Li^{-\alpha}\bar\nu_j\right),
\\
&\stackrel{\labelcref{eq:slope}}{\geq} \frac12 \sum_{i\geq i_{\min}}  L^2 i^{-2\alpha} \gamma_T^2 \sigma^2\sum_{k=0}^{T-1}\eta_k^2 \exp\left( -2Li^{-\alpha}\gamma_T\frac{(T-k)^2}{cT} \right).
\end{align*}

We index the sum on the decay part from $\ell=\lceil cT/2 \rceil$ to $\ell=\lfloor cT \rfloor$, and on dimensions $i\geq i_{\min}$ greater than $\gamma LcT$. Let $i_{\min}'(\gamma LcT) := \max\{i_{\min}, (2\gamma L c T)^{1/\alpha}\}$. 

For $i \geq i_{\min}'(\gamma LcT),\; 2Li^{-\alpha} \gamma \frac{\ell^2}{cT} \leq \left(\frac{\ell}{cT}\right)^2$ and $\frac{\ell}{cT}\leq 1$ , so 
\begin{equation}\label{eq:exponent}
\displaystyle \exp\left( -2Li^{-\alpha}\gamma\frac{\ell^2}{cT} \right) \geq \exp \left(- \bigg(\frac{\ell}{cT}\bigg)^2\right) \geq 1/e.
\end{equation}
Therefore, as all terms are positive:
\begin{align*}
V_T 
&\geq \frac12 \sum_{i\geq i_{\min}}  L^2 i^{-2\alpha} \gamma_T^2 \sigma^2\sum_{\ell=1}^{T}\eta_{T-\ell}^2 \exp\left( -2Li^{-\alpha}\gamma_T\frac{\ell^2}{cT} \right),
\\
&\stackrel{\labelcref{eq:exponent}}{\geq} \frac12 \sum_{i\geq i'_{\min}(\gamma_T LcT)}  L^2 i^{-2\alpha} \gamma_T^2 \sigma^2\sum_{\ell=\lceil cT/2 \rceil}^{\lfloor cT \rfloor}\frac{\eta_{T-\ell}^2 }{e} \quad\text{and}\displaystyle \quad \eta_{T-\ell}=\frac{\ell}{cT}\geq\frac12 \text{ on the block}, \\
&\geq \frac{\sigma^2\gamma_T^2L^2}{2e}\left(\sum_{\ell=\lceil cT/2\rceil}^{\lfloor cT\rfloor}\frac14\right)\sum_{i\geq i'_{\min}(\gamma_T LcT)}i^{-2\alpha},
\\
&\geq \frac{\sigma^2\gamma_T^2L^2\,cT}{32\,e}\int_{i'_{\min}(\gamma_T LcT)}^{\infty}x^{-2\alpha}dx
\ =\ \frac{\sigma^2\gamma_T^2L^2\,cT}{32\,e\,(2\alpha-1)}\,(i'_{\min}(\gamma_T LcT))^{1-2\alpha} ,
\\&
\geq \frac{\sigma^2\gamma_T^2L^2\,cT}{32\,e\,(2\alpha-1)}\,(2\gamma_T LcT)^{\frac{1-2\alpha}{\alpha}} \quad \text{ for }\gamma_T T \text{ large enough so that } i_{\min}'(\gamma_TLcT)=(2\gamma_TLcT)^{1/\alpha},\\
&= \frac{2^{\frac{1-2\alpha}{\alpha}}\sigma^2L^{1/\alpha}}{32\,e\,(2\alpha-1)}\gamma_T^{1/\alpha} (cT)^{\frac{1-\alpha}{\alpha}}.
\end{align*}

\item Suppose $\alpha>2$. 

We now keep only the dimension $i=i_{\min}$, and index the sum on the decay part over the
block $\ell \in \left[\tfrac12 A^{-1/2},\, A^{-1/2}\right]$, where
$A := \frac{2Li_{\min}^{-\alpha}\gamma_T}{cT}$ is such that the exponent equals $A\ell^2$.
This block is contained in $\{1,\dots,\lfloor cT\rfloor\}$ and contains at least
$\tfrac14 A^{-1/2}$ integers, since $\gamma_T T \to \infty$ gives $A^{-1/2}\leq cT$ and
$\gamma_T L \leq \frac{2}{\zeta(\alpha)} \leq 2$ gives $A^{-1/2}\geq 4$ for $T$ large enough.

For $\ell$ in this block,
\begin{equation}\label{eq:exponent_big_alpha}
\exp\left( -2Li_{\min}^{-\alpha}\gamma_T\frac{\ell^2}{cT} \right) = \exp\left(-A\ell^2\right) \geq 1/e
\qquad\text{and}\qquad
\eta_{T-\ell}^2 = \left(\frac{\ell}{cT}\right)^2 \geq \frac{1}{4A(cT)^2}.
\end{equation}
Therefore, as all terms are positive,
\begin{align*}
V_T
&\stackrel{\labelcref{eq:exponent_big_alpha}}{\geq}
\frac{L^2 i_{\min}^{-2\alpha}\gamma_T^2\sigma^2}{2}\frac{1}{e}\frac{A^{-1/2}}{4}\frac{1}{4A(cT)^2}
\ =\ \frac{L^2 i_{\min}^{-2\alpha}\sigma^2}{32\,e}\,\frac{\gamma_T^2\,A^{-3/2}}{(cT)^2},
\\
&=\ \frac{i_{\min}^{-\alpha/2}L^{1/2}\sigma^2}{2^{13/2}\,e}\ \gamma_T^{1/2}\,(cT)^{-1/2},
\end{align*}
using $A^{-3/2}=\left(2Li_{\min}^{-\alpha}\right)^{-3/2}\gamma_T^{-3/2}(cT)^{3/2}$, together with
$2-\tfrac32=\tfrac12$ and $-2+\tfrac32=-\tfrac12$.

\item \emph{Conclusion for $\alpha <2$ and $\alpha>2$.}
Remind that by \cref{cor:epsilon}, $\risk_T(\gammastar_T) \sim \eqrisk_T(\gammastar_T)$.

This gives $\risk_T(\gammastar_T) = O(T^{-q})$ with 
$q=\begin{cases}
\frac{\alpha+\beta-1}{\alpha+\beta} & \text{if }\alpha<2,
\\
\frac{2\alpha+2\beta-2}{3\alpha+2\beta-2} & \text{if } \alpha>2.
\end{cases}$

Combining with the bounds above yields $\gamma_T = O(T^{-\alpha q + \alpha - 1})$ if $\alpha <2$ and $\gamma_T = O(T^{1-2q })$ if $\alpha >2$, which gives $\gamma_T = O(T^{-\vartheta})$ where $\vartheta$ is the exponent defined \cref{lemma:eta_star_to_0} verifying $\gammastar_T = \Theta(T^{-\vartheta})$. In particular, $\gamma_T \to 0$ and $\gamma_TT\to \infty$.

\end{itemize}

\end{proof}

\label{proof:ccl_risks}
\begin{proof}[Proof of \cref{th:gammastar}.]
Let us combine the preceding lemmas. Let $(\gamma_T)$ a sequence which satisfies 
\begin{equation}\label{eq:ccl_risks}
\risk_T(\gamma_T)\leq\risk_T(\gammastar_T)(1 + o(1)).
\end{equation}
\begin{itemize}
\item
\cref{eq:ccl_risks} implies $\risk_T(\gamma_T) = O(\risk_T(\gammastar_T))$. By \cref{lemma:bounds}, $\gamma_T\to 0$ and $\gamma_T T \to \infty$.
\item
By \cref{lemma:uniform_epsilon}, we have therefore $\risk_T(\gamma_T) \sim \eqrisk_T(\gamma_T)$. 
\item
By \cref{cor:epsilon}, $\risk_T(\gammastar_T) \sim \eqrisk_T(\gammastar_T)$.
\item
By \cref{lemma:eta_star_to_0}, item 3., $\eqrisk_T(\gammastar_T)\sim \min_\gamma \eqrisk_T(\gamma)\leq \eqrisk_T(\gamma_T)$, so $\eqrisk_T(\gamma_T) \geq \eqrisk_T(\gammastar_T)(1+o(1)).$
\end{itemize}

So $\risk_T(\gamma_T) \sim \eqrisk_T(\gamma_T) \geq \eqrisk_T(\gammastar_T)(1+o(1)) \sim \risk_T(\gammastar_T)(1+o(1)) .$ 

Knowing \cref{eq:ccl_risks}, we get $\risk_T(\gamma_T) \sim \risk_T(\gammastar_T)$.

\end{proof}

\subsection{Optimal Cooldown Length}
\label{app:optimal_c}
We define for $c\in (0,1]$:
\begin{equation}\label{eq:h}
    h(c) := \frac{\mathcal{W}(\alpha, c)}{(2-c)^{1/\alpha}}=\frac{1}{4\alpha}\Gamma\left(1-\frac{1}{\alpha}\right) \left[ \frac{\alpha}{2-\alpha} c^{\frac{1-\alpha}{\alpha}} (2-c)^{-1/\alpha} - (2-c)^{-1} \right].
\end{equation}

\begin{lemma}\label{lemma:c_star}
Suppose $1<\alpha<2$.
The function $h$ defined in \cref{eq:h} is strictly decreasing over $(0,1]$.
\end{lemma}

\begin{proof}

Taking the derivative of this expression with respect to $c$ yields:
\[
    h'(c) = \frac{1}{4\alpha}\Gamma\left(1-\frac{1}{\alpha}\right) (2-c)^{-2} \underbrace{\left[ \frac{2-2\alpha+\alpha c}{2-\alpha} c^{\frac{1-2\alpha}{\alpha}} (2-c)^{\frac{\alpha-1}{\alpha}} - 1 \right]}_{=: w(c)-1}.
\]
For $c \in (0,1)$,
\begin{itemize}
\item If $c \leq \frac{2\alpha-2}{\alpha}$:\quad $\frac{2-2\alpha+\alpha c}{2-\alpha}\leq 0$ so $w(c)-1 <0$.
\item If $c > \frac{2\alpha-2}{\alpha}$:\quad $w(c)>0$ and
\[
\frac{w'(c)}{w(c)} = \frac{\alpha^2 c(2-c) + (2-2\alpha+\alpha c)(2-4\alpha+\alpha c)} {\alpha c(2-c)(2-2\alpha+\alpha c)} >0.
\]
Therefore $w$ increases so $w(c) < \lim_{c\to1}w(c) =1$. Thus, $h'(c) < 0$.
\end{itemize}

Consequently, $c<1 \Rightarrow h'(c) < 0$ while $h'(1)=0$, which implies that $h$ is strictly decreasing on the interval $(0, 1].$ 

\end{proof}

\subsection{Computation of the Risk with a Minibatch}
\label{app:batch}
With a batch size $\bs$, the mean stochastic gradient is (suppose w.l.o.g. $\vect_*=0$):
\[
g_t = \frac{1}{\bs}\sum_{i=1}^\bs
\left[  {\featvec}_{i} {\featvec}_i^\top \vect_t - \varepsilon_i {\featvec}_i   \right].
\]

So $\vect_{t+1} = \frac{1}{\bs}\sum_{i=1}^\bs(I - \gamma\eta_t{\featvec}_i {\featvec}_i^\top)\vect_t +\gamma\eta_t \varepsilon_i {\featvec}_i$.

\begin{align*}
M_{t+1} &= \frac{1}{\bs^2}\E\left[\sum_{i=1}^\bs((I-\gamma \eta_t {\featvec}_i {\featvec}_i^\top)\vect_t + \gamma \eta_t \varepsilon_i {\featvec}_i)\sum_{j=1}^\bs((I-\gamma \eta_t {\featvec}_j {\featvec}_j^\top)\vect_t + \gamma \eta_t \varepsilon_j {\featvec}_j)^\top\right] ,
\\&=
\frac{1}{\bs^2}\sum_{i,j} \E\left[((I-\gamma \eta_t {\featvec}_i {\featvec}_i^\top)\vect_t + \gamma \eta_t \varepsilon_i {\featvec}_i)((I-\gamma \eta_t {\featvec}_j {\featvec}_j^\top)\vect_t + \gamma \eta_t \varepsilon_j {\featvec}_j)^\top\right] ,
\\&=
\frac{1}{\bs^2}\sum_{i,j} \E\left[((I-\gamma \eta_t {\featvec}_i {\featvec}_i^\top)\vect_t \vect_t^\top (I-\gamma \eta_t {\featvec}_j {\featvec}_j^\top)^\top \right] +\gamma^2\eta_t^2 \E[\varepsilon_i\varepsilon_j {\featvec}_i {\featvec}_j^\top],
\\&=
\frac{1}{\bs^2}\sum_{i,j} \E[\vect_t \vect_t^\top]-\gamma \eta_t\E[{\featvec}_i{\featvec}_i^\top \vect_t \vect_t^\top]-\gamma \eta_t\E[ \vect_t \vect_t^\top{\featvec}_j{\featvec}_j^\top]+ \gamma^2\eta_t^2\E\left[ {\featvec}_i {\featvec}_i^\top \vect_t \vect_t^\top {\featvec}_j {\featvec}_j^\top\right] +\gamma^2\eta_t^2 \E[\varepsilon_i\varepsilon_j {\featvec}_i {\featvec}_j^\top],
\\&=
\frac{1}{\bs^2}\sum_{i,j} \E[\vect_t \vect_t^\top]-\gamma \eta_t\E[{\featvec}_i{\featvec}_i^\top \vect_t \vect_t^\top]-\gamma \eta_t\E[ \vect_t \vect_t^\top{\featvec}_j{\featvec}_j^\top]+ \gamma^2\eta_t^2\E\left[ {\featvec}_i {\featvec}_i^\top \vect_t \vect_t^\top {\featvec}_j {\featvec}_j^\top\right] +\gamma^2\eta_t^2 \sigma^2 \Lambda \delta_{i,j},
\\&=
M_t - \gamma \eta_t\Lambda M_t - \gamma \eta_t M_t \Lambda +\frac{1}{\bs}\gamma^2\eta_t^2\sigma^2\Lambda  +\frac{1}{\bs^2}\sum_{i,j}  \gamma^2\eta_t^2\E\left[ {\featvec}_i {\featvec}_i^\top \vect_t \vect_t^\top {\featvec}_j {\featvec}_j^\top\right] .
\end{align*}

However,
\[
\E\left[ {\featvec}_i {\featvec}_i^\top \vect_t \vect_t^\top {\featvec}_j {\featvec}_j^\top\right] =
\begin{cases}
    \Lambda M_t \Lambda & \text{if }i \neq j,
    \\
    \tr (\Lambda) \Lambda \E[\diag(\vect_t)^2] & \text{if }i=j.
\end{cases}
\]

Taking the diagonal vector only, we get $\begin{cases}
    \Lambda^2 m_t & \text{if }i \neq j,
    \\
    (\tr \Lambda) \Lambda m_t & \text{if }i=j.
\end{cases}$

Therefore
\[
m_{t+1} = \left(I - 2\gamma \eta_t\Lambda+ \gamma^2\eta_t^2\frac{\bs-1}{\bs}\Lambda^2 + \gamma^2\eta_t^2\frac{1}{\bs}\tr (\Lambda)\Lambda \right)m_t + \gamma^2\eta_t^2 \frac{\sigma^2}{\bs} \lambda.
\]

\subsection{Minibatch Asymptotic Analysis}

\label{sec:batch_asymptotics}

\paragraph{Constant schedule. Bias:}

With $\eta_j=1$, the effective rate $\bar\gamma$ becomes $\bar\gamma := \gamma\Big(1-\tfrac{\gamma\tr(\Lambda)}{2\bs}\Big)$. Moreover, each term of the sum can be written:
\[
    \rho_i := 1 - 2\gamma Li^{-\alpha} + \gamma^2\tfrac{\tr(\Lambda)}{\bs}Li^{-\alpha}
             + \gamma^2\tfrac{\bs-1}{\bs}L^2i^{-2\alpha}
           = 1 - 2\bar\gamma\,Li^{-\alpha} + \gamma^2\tfrac{\bs-1}{\bs}L^2i^{-2\alpha}.
    \qquad
\]
Hence $B_t=\tfrac12\sum_i \tfrac{L\Delta}{i^{\alpha+\beta}}\rho_i^{\,t}$ has $z$-transform
$\mathcal B(z)=\tfrac{L\Delta}{2}\sum_i i^{-(\alpha+\beta)}\,[1-z\rho_i]^{-1}$, and we follow the
batch-free proof. After $\mu$ differentiations, with $s=1-z$ and the substitution
$x_i=is^{1/\alpha}$ (so $i^{-\alpha}=x_i^{-\alpha}s$ and $i^{-2\alpha}=x_i^{-2\alpha}s^2$),
the denominator reads
\[
    1-(1-s)\rho_i
    = s + (1-s)\big(2\bar\gamma L\,i^{-\alpha} + \gamma^2\tfrac{\bs-1}{\bs}L^2 i^{-2\alpha}\big)
    = s\Big[\,1 + (1-s)\,2\bar\gamma L\,x_i^{-\alpha}
             + \underbrace{(1-s)\,\gamma^2\tfrac{\bs-1}{\bs}L^2\,x_i^{-2\alpha}\,s}_{\;\to\,0\ \text{as }s\to0}\Big].
\]
The quadratic term carries an extra factor $s$ and therefore vanishes pointwise; by
dominated convergence the Riemann sum converges to the same integral as in the batch-free
case, with a new $\bar\gamma$. The Tauberian step then gives
\[
    B_T \sim \frac{L\Delta}{2\alpha}\,
    \frac{\Gamma\!\big(\tfrac{\beta-1}{\alpha}+1\big)}
         {\big(2\bar\gamma\,LT\big)^{\frac{\beta-1}{\alpha}+1}}
       = \frac{L\Delta}{2\alpha}\,
    \frac{\Gamma\!\big(\tfrac{\beta-1}{\alpha}+1\big)}
         {\big(2\gamma\big(1-\tfrac{\gamma\tr\Lambda}{2\bs}\big)LT\big)^{\frac{\beta-1}{\alpha}+1}},
\]
which is \cref{th:constant} with the single replacement
$\tr(\Lambda)\to\tr(\Lambda)/\bs$.

\paragraph{Constant schedule. Variance:}

We can directly compute the new variance. In this case, the $\zeta$ function cannot be used because the final sum is of the form $\sum_{i=1}^\infty \frac{1}{Ai^{\alpha}+B}$ where $B\neq 0$. 
\begin{align*}
V_T &= \frac{\sigma^2}{2\bs} \sum_{i=1}^\infty  \sum_{k=0}^{T-1} L^2i^{-2\alpha}\gamma^2   \left[ 1-2\gamma Li^{-\alpha}+\gamma^2 \frac{\tr(\Lambda)}{\bs}L i^{-\alpha} + \gamma^2\frac{\bs-1}{\bs}L^2i^{-2\alpha}\right]^{T-1-k},
\\&= 
 \frac{\sigma^2}{2\bs} \sum_{i=1}^\infty  L^2i^{-2\alpha}\gamma^2   \frac{1-\left[ 1-2\gamma Li^{-\alpha}+\gamma^2 \frac{\tr(\Lambda)}{\bs}L i^{-\alpha} + \gamma^2\frac{\bs-1}{\bs}L^2i^{-2\alpha}\right]^T}{1-\left[ 1-2\gamma Li^{-\alpha}+\gamma^2 \frac{\tr(\Lambda)}{\bs}L i^{-\alpha} + \gamma^2\frac{\bs-1}{\bs}L^2i^{-2\alpha}\right]},
\\ &\xrightarrow{T\to\infty} \frac{\gamma L\sigma^2}{4b} \sum_{i=1}^\infty \frac{1}{\left(1-\frac{\gamma \tr(\Lambda)}{2\bs} \right)i^{\alpha}-\frac{L\gamma}{2}\frac{\bs-1}{\bs}}.
\end{align*}

\paragraph{WSD schedule (including linear).}

Let $\a_i = 2\gamma L i^{-\alpha}$, $\b_i = \gamma^2\frac{\tr(\Lambda)}{\bs}Li^{-\alpha}$ and $\c_i  =  \gamma^2\frac{\bs-1}{\bs}L^2i^{-2\alpha} = O(i^{-2\alpha})$.

In the exponential, there is now $I_i=\int\log (1 - \a_iu+\b_iu^2+\c_i u^2)du$.

There are two cases:
\begin{itemize}
    \item If large dimensions dominate (bias for all $\alpha$, and variance for $\alpha<2$), then $I_i = -\frac{\a_i}{2}+\frac{\b_i}{3} + O(i^{-2\alpha})$.
    \item If small $u$ dominate (i.e. not the entire variance accumulation), for variance and $\alpha>2$,
    \[
        I_i = \int_0^x \left(-\a_i u + O(x^2)\right) du = -\frac{\a_i}{2}x^2 + O(x^3).
    \]
    
\end{itemize}

Therefore $\c_i$ never appears in the leading asymptotic equivalents. Compared with the batch-free case, the only two changes are $\tr(\Lambda)\to\tr(\Lambda)/\bs$ (through $\b_i$)
and $\sigma^2\to\sigma^2/\bs$ in the noise prefactor. Substituting these into the proofs of \cref{th:wsd} yields \cref{th:batch}.

\end{document}